\documentclass[a4paper,oneside,12pt,reqno]{article}

\usepackage[T1]{fontenc}
\usepackage[utf8]{inputenc}
\usepackage[english]{babel}
\usepackage{amsfonts}
\usepackage{amsmath}
\usepackage{fancyhdr}
\usepackage{epic}
\usepackage{eepic}
\usepackage{amssymb}
\usepackage[colorlinks = true,urlcolor=cyan,citecolor=blue, linkcolor=blue,hypertexnames=false]{hyperref}
\usepackage{fancybox}
\usepackage{graphicx}
\usepackage{amsthm}
\usepackage{wrapfig}
\usepackage[usenames]{color}
\usepackage{float}
\usepackage{verbatim}
\usepackage{cancel}
\usepackage{enumitem}
\usepackage{multicol}
\usepackage{appendix}

\AtBeginEnvironment{subappendices}{
\chapter*{Appendix}
\addcontentsline{toc}{section}{Appendix}
\counterwithin{figure}{section}
\counterwithin{table}{section}
}

\usepackage{tikz}
\tikzset{snake it/.style={decorate, decoration=snake}}
\usepackage{pgfplots}
\pgfplotsset{compat=newest}
\pgfplotsset{
            layers/my layer set/.define layer set={
            background,
            main,
            foreground
        }{
           },
           set layers=my layer set,
    }

\newcounter{dummy}
\makeatletter
\newcommand\myitem[1][]{\item[(#1)]\refstepcounter{dummy}\def\@currentlabel{#1}}
\makeatother

\newenvironment{Figure}
  {\par\medskip\noindent\minipage{\linewidth}}
  {\endminipage\par\medskip}

\theoremstyle{plain}
\newtheorem{teo}{{Theorem}}[section]
\newtheorem{cor}[teo]{Corollary}
\newtheorem{pro}[teo]{Proposition}
\newtheorem{lem}[teo]{Lemma}

\theoremstyle{definition}
\newtheorem{defi}[teo]{Definition}

\newtheorem{rem}{Remark}

\renewcommand{\d}{\text{d}}
\newcommand{\dx}{\,{\rm d}x}
\newcommand{\dy}{\,{\rm d}y}
\newcommand{\dt}{\,{\rm d}t}
\newcommand{\ds}{\,{\rm d}s}

\newcommand{\E}{\mathcal{E}}
\newcommand{\EE}{\mathbb{E}}

\newcommand{\R}{\mathbb{R}}

\renewcommand{\L}{\mathcal{L}}

\renewcommand{\b}{\mathcalboondox{b}}

\DeclareFontFamily{U}{BOONDOX-calo}{\skewchar\font=45 }
\DeclareFontShape{U}{BOONDOX-calo}{m}{n}{
	<-> s*[1.05] BOONDOX-r-calo}{}
\DeclareFontShape{U}{BOONDOX-calo}{b}{n}{
	<-> s*[1.05] BOONDOX-b-calo}{}
\DeclareMathAlphabet{\mathcalboondox}{U}{BOONDOX-calo}{m}{n}
\SetMathAlphabet{\mathcalboondox}{bold}{U}{BOONDOX-calo}{b}{n}
\DeclareMathAlphabet{\mathbcalboondox}{U}{BOONDOX-calo}{b}{n}

\begin{document}

\title{\vspace{-2cm}\bf Is Stochastic Gradient Descent Effective?\\ A PDE Perspective \\ on Machine Learning processes}

	\author{\Large Davide Barbieri,$^{\!a,b}$ Matteo Bonforte,$^{\!a,c}$
		Peio Ibarrondo$^{a,d}$
    }
	\date{}

	\maketitle

	\begin{abstract}

 In this paper we analyze the behaviour of the stochastic gradient descent (SGD), a widely used method in supervised learning for optimizing neural network weights via a minimization of non-convex loss functions. Since the pioneering work of Li, Tai and E (2017), the underlying structure of such processes can be understood via parabolic PDEs of Fokker-Planck type, which are at the core of our analysis. Even if Fokker-Planck equations have a long history and an extensive literature, almost nothing is known when the potential is non-convex or when the diffusion matrix is (very) degenerate, and this is the main difficulty that we face in our analysis.
This affects the long-time behaviour of solutions, a crucial point in understanding deep characteristics of the associated learning process.

 Our main contribution is identifying two different regimes in the learning process. In the initial phase of SGD, the loss function drives the weights to concentrate around the nearest local minimum, which may not necessarily be optimal. We refer to this phase as the \emph{drift regime} and we provide quantitative estimates that shed light on this concentration phenomenon.
Next, we introduce the \emph{diffusion regime}, which typically happens after some time, where stochastic fluctuations help the learning process to diffuse and so escape suboptimal local minima.
We analyze the ``Mean Exit Time'' (MET), i.e. the time needed to escape a local minimum, and prove precise upper and lower bounds of the MET. Finally, we address the asymptotic convergence of SGD, tackling the complexities of non-convex cost functions together with the degeneracies in the diffusion matrix, that do not allow to use the standard approaches, and require new techniques. For this purpose, we exploit two different methods: duality and entropy methods.

This work provides new results about the dynamics and effectiveness of SGD, offering a deep connection between stochastic optimization and PDE theory. It also provides some answers and insights to basic questions in the Machine Learning processes: How long does SGD take to escape from a local minimum? Do neural network parameters converge using SGD? How do parameters evolve in the first stage of training with SGD?

	\end{abstract}
	
 \small
	
	\noindent {\sc Keywords: } stochastic gradient descent, supervised learning, mass concentration, mean exit time, asymptotic behaviour, entropy method.

		\noindent{\sc MSC2020 Classification}. Pri: 35Q68, 68T07, 35K65, 35B40.\quad Sec: 60J60, 35J70, 35K15
  %35B40 Asymptotic behaviour of solutions to PDEs
  %35H10 Hypoelliptic equations
  %35J70 Degenerate elliptic equations
  %35K15 Initial value problems for second order parabolic equations
  %35K65 Degenerate parabolic equations
  %35K37 Reaction-Diffusion equation
  %35Q68 PDE in connection with computer science
  %60J60 Diffusion processes
  %68T07 Artificial neural networks and deep learning

	\begin{itemize}[leftmargin=0pt]\itemsep0pt \parskip0pt \parsep0pt
\footnotesize
        \item[(a)]Departamento de Matem\'{a}ticas, Universidad Aut\'{o}noma de Madrid,\\
		ICMAT - Instituto de Ciencias Matem\'{a}ticas, CSIC-UAM-UC3M-UCM, \\
		Campus de Cantoblanco, 28049 Madrid, Spain.

        \item[(b)]
        E-mail:\texttt{~davide.barbieri@uam.es }
        Web-page: \footnotesize\texttt{https://sites.google.com/view/davidebarbieri/}

		\item[(c)]
		E-mail:\texttt{~matteo.bonforte@uam.es }
		Web-page: \footnotesize\texttt{https://verso.mat.uam.es/\textasciitilde{}matteo.bonforte/}
		
		\item[(d)]
		E-mail:\texttt{~ibarrondo.peio@gmail.com};

	\end{itemize}

\small
	
	\newpage
	
	{\itemsep2pt \parskip1pt \parsep1pt \tableofcontents }
	
	\normalsize

    \vfill
    \section*{Acknowledgments} D.B., P.I. were partially supported by the research Projects PID2022-142202NB-I00 and PID2019-105599GB-I0 all funded by the Spanish Ministry of Science and Innovation MCIN/AEI/10.13039/501100011033. M.B., P.I. were partially supported by the research Projects MTM2017-85757-P, PID2020-113596GB-I00 and PID2023-150166NB-I00 all fund\-ed by the Spanish Ministry of Science and Innovation MCIN/AEI/10.13039/501100011033. D.B., M.B., P.I.  acknowledge the financial support from the Spanish Ministry of Science and Innovation, through the ``Severo Ochoa Programme for Centres of Excellence in R\&D'' CEX2019-000904-S and CEX2023-001347-S, and by the European Union’s Horizon 2020 research and innovation programme under the Marie Sk\l odowska-Curie grant agreement no. 777822. P.I. was  partially funded by the FPU-grant  FPU19/04791  from the Spanish Ministry of Universities.
	
	\smallskip\noindent {\sl\small\copyright~2025 by the authors. This paper may be reproduced, in its entirety, for non-commercial purposes.}

\newpage

\section{Introduction}

In recent decades, Machine Learning has gained significant attention across various scientific communities, thanks to its practical utility and wide range of applications.
One of the main challenges in Machine Learning is choosing the weights or parameters of a neural network to minimize a non-convex loss function based on a given data set. This is a non-convex optimization problem formulated as follows:
\begin{equation}
\arg\min\limits_{\theta\in\R^d}L(\theta):=\arg\min\limits_{\theta\in\R^d}\frac{1}{N}\sum_{i=1}^{N}L_i(\theta)
\end{equation}
where $L,L_i:\R^d\rightarrow\R_+$ for $i=1,\dots,N$ and $\theta\in\mathbb{R}^d$ are the parameters of the model. The function $L = \frac{1}{N}\sum_{i=1}^{N}L_i$ represents the total loss function of the given data set with $N$ samples, while $L_i$ represents the loss associated to the $i^{\rm th}$ training sample. A  classical approach to this problem is to perform a gradient descent, where given some initial weights $\theta_0\in\R^d$ and a constant learning rate $\eta>0$ the weights at step $n\ge 0$ are updated by
\[
\theta_{n+1}=\theta_n-\eta\nabla L(\theta_n)\quad\qquad\forall n\ge 0\,.
\]
This learning algorithm is computationally too expensive to implement since at each step we have to compute $N$ gradients of loss functions $L_i$. Moreover, this method may lead the parameters to a local minimum of $L$ instead of reaching the global minimum. In order to avoid these problems, the use of the stochastic gradient descent (SGD) has proven to be effective. Defining $\lbrace \gamma_n\rbrace_{n\ge 1}$ as i.i.d. uniform random variables with values on $\lbrace 1,\dots,N\rbrace$, the basic SGD algorithm reads
\[
\theta_{n+1}=\theta_n-\eta\nabla L_{\gamma_n}(\theta_n)\quad\qquad\forall n\ge 0\,.
\]
Implementing this method yields very good results in practice, but the mathematical theory behind it remains poorly understood, see \cite{BCN18} for an overview. The goal of this manuscript is to shed light on this learning process and to present new research directions in this field.

In 2017, Li, Tai and E opened a new framework for analyzing the SGD learning process in \cite{LiTaiWeinan2017}, which they examined in more detail in \cite{LiTaiWeinan2019}. They proved that the SGD iteration is a first order approximation, in the Euler-Maruyama discretization sense, of an associated stochastic differential equation (SDE). See Section \ref{sec:methodology} for a discussion of this model and related work.  {This manuscript is devoted to the study of the Fokker-Planck equation for such a continuous SDE model, that is the drift-diffusion PDE for the transition probability.}

For this purpose, we will examine the interplay between the drift term governed by the loss function and the degenerate diffusion provided by the randomness in SGD. In the initial stages of the learning process, the dynamics of the parameters primarily follows the drift, tending to concentrate around the local minima of the loss function $L$. Nevertheless, under mild condition on the degeneracy of the diffusion coefficients, we estimate the training duration required for the neural network to escape from non-optimal local minima. In this manuscript, we present a quantitative analysis of the two regimes of motion in the Fokker-Planck evolution: the drift regime in Section \ref{sec Drift regime} and the diffusion regime in Section \ref{sec Diffusion Regime}.

Additionally, in Section \ref{sec Asymptotic}, we address the problem of the existence of steady states for the parameters distributions and the asymptotic convergence towards them. As a consequence of the high degeneracy of this problem, the differential operator describing the evolution of the transition probability resembles an intermediate case between a transport equation and a diffusion equation, see Section \ref{sec Entropy method}. By means of simple yet representative examples we analyze the extreme cases to get an idea of the possible scenarios. We provide two types of results. On one hand, we introduce a new variant of the SGD learning process, and exploit the recent results by Porretta \cite{Por} about  convergence to steady states in  non-degenerate Fokker-Planck equations. This will also allow us to prove existence of steady measures for the general case,   which to the best of our knowledge was not known.   On the other hand, we review some results from Arnold and Erb \cite{Arnold2014} based on the entropy method   (\'a la Bakry-Émery),   and see when they can apply to our situation.   We also provide a new entropy computation, which differs from the classical one for non-degenerate Fokker-Planck equations, and represents a (small) novelty in this direction.   These results hold only for constant diffusion matrices, hence representing a local scenario around a minimum of the cost function.   Finally, in Section \ref{sec Open Questions}, we provide a number of open questions and directions for future research.

Beside the interest in the Machine Learning applications, which has motivated this work, our analysis could be also of interest in the study of stochastic differential equations and their long time behaviour.

\subsection{Main results}

For more generality, in this paper we will consider the mini-batch Stochastic Gradient Descent. Instead of selecting randomly one sample in each step, we choose randomly a batch of samples $B_n$ of size $\b\ge 1$. Namely, the SGD with constant learning rate $\eta>0$ and batch size $|B_n|=\b\ge1$ reads as follows
\begin{equation}\label{SGD Main}\tag{SGD}
\theta_{n+1}=\theta_n-\frac{\eta}{\b}\sum_{b_i\in B_n}\nabla L_{b_i}(\theta_n)\qquad\forall n\ge0\,,
\end{equation}
with $B_n=\lbrace b_i\rbrace_{i=1}^{\b}\in\lbrace 1,2,\dots,N\rbrace^\b$ being a batch of indexes chosen with uniform distribution in each step $n$. In \cite{LiTaiWeinan2019} and \cite{FengLiLiu}, it is deduced that we can approximate \eqref{SGD Main} with the continuous stochastic process defined by
\begin{align}\label{SDE intro}
\d X_t=-\nabla L(X_t)\dt +\sqrt{\frac{\eta}{\b}Q(X_t)}\,\d W_t\,,
\end{align}
with $Q(x) = \tfrac{1}{N}\sum_{i=1}^{N}\nabla L_i(x)\otimes\nabla L_i(x)-\nabla L(x)\otimes\nabla L(x)$ being a nonnegative matrix for any $x\in\mathbb{R}^d$.
Observe that $Q$ is generically not invertible. Indeed, since we have
\begin{equation*}
 Q(x)=T(x)^*\, T(x)\qquad\mbox{with}\qquad T(x) =
  \begin{pmatrix}
    \frac{1}{N}\nabla L_1(x)-\nabla L(x) \\
    \vdots \\
    \frac{1}{N}\nabla L_N(x)-\nabla L(x)
  \end{pmatrix} \in \R^{N \times d}\,,
\end{equation*}
then the rows of $T(x)$ sum zero at every $x\in\mathbb{R}^d$. Here we have used $T(x)^*$ to denote the transpose of $T(x)$. Specifically, in the common case of overparametrization \cite{over-parametrized}, we have that $Q(x)\in\mathbb{R}^{d\times d}$ satisfies the rank condition rank$(Q(x))\le N-1< d$. As a consequence, the noise in \eqref{SDE intro} is degenerate.

 In addition, the ratio between the learning rate and the mini-batch is usually considered very small, which diminishes the noise. We call this ratio the \emph{effective learning rate} and  we denote it by
\begin{equation*}
   \varepsilon^2:=\frac{\eta}{2\b}\,.
\end{equation*}

One of the advantages of the continuous approximation \eqref{SDE intro} is that we can use the theory of PDEs to analyze the behaviour of the transition probability and its evolution in time. More specifically, the transition probability $\rho(t,x)$ associated to the process in \eqref{SDE intro} with an initial distribution $\rho_0(x)$ satisfies the following Fokker-Planck type equation:
\begin{equation}\label{FP intro}
\begin{cases}
  \;\;\partial_t \rho&=\nabla\cdot\bigg(\varepsilon^2\nabla\cdot\left(Q(x)\rho\right)+\rho\nabla L(x)\bigg)\qquad\mbox{in}\qquad(0,\infty)\times\mathbb{R}^d\,,\\
\rho(0,x)&=\rho_0(x)\hspace{55mm}\mbox{in }\qquad\mathbb{R}^d\,,
\end{cases}
\end{equation}
where $\nabla\cdot A$, with $A$ being a matrix, denotes the divergence taken columnwise. For the deduction of this equation from Ito's formula see \cite[Section 4.5]{Schuss}. The class of solutions that we consider lie on the probability space $\mathcal{P}_k(\mathbb{R}^d)$ of nonnegative normalized Radon measures with $k\ge 0$ finite moments.
\begin{defi}[Weak solutions.]\label{weak solution}
  We say that $\rho\in C\left([0,T),\mathcal{P}_k(\mathbb{R}^d)\right)$ is a weak solution of \eqref{FP intro} starting from $\rho_0\in\mathcal{P}_k(\mathbb{R}^d)$ if $\forall\varphi\in C^{1,2}\left([0,T]\times\mathbb{R}^d\right)$ with polynomial growth of order $k$ in the spatial variable, it holds that
  \begin{equation}\label{eq:def_weaksol}
  \begin{split}
     &\int_{\mathbb{R}^d}\varphi(T,x)\rho(T,x)\dx-\int_{\mathbb{R}^d}\varphi(0,x)\rho_0(x)\dx-\int_{0}^{T}\int_{\mathbb{R}^d}\partial_t\varphi(t,x)\rho(t,x)\dx \dt \\ =&\int_{0}^{T}\int_{\mathbb{R}^d}\bigg[\varepsilon^2\mbox{tr}\left(Q(x)D^2\varphi(t,x)\right)-\nabla L(x)\cdot\nabla\varphi(t,x)\bigg]\rho(t,x)\dx\dt
  \end{split}
  \end{equation}
\end{defi}
 {
The existence of weak solutions is guaranteed whenever $Q$ and $\nabla L$ belong to $C(\R^d)$, see \cite[Theorem 6.7.3]{BogaKrylov}. For our purposes, the most natural assumption will be slightly stronger, namely we will assume $L \in C^{1,1}(\R^d)$.
\begin{rem}
In this paper we will deal with two parabolic equations: the Fokker-Planck equation \eqref{FP intro}, and the Ornstein-Uhlenbeck equation \eqref{Intro Mean Exit Time}, whose generators of the evolution are one the formal adjoint of the other. This duality is further expressed in the explicit comparison of their elliptic regularizations in \eqref{O-U} and \eqref{F-P}. The Fokker-Planck equation requires the notion of weak (distributional) solution, Definition \ref{weak solution}, since we allow the solutions to be \emph{pure} measures. On the other hand, for the dual equation, we shall use a dual notion to weak solutions, which in this case coincides with the viscosity solution. For more details, see \cite{CrIsLi, Por}. For convenience of the reader, we recall here the definitions of viscosity solutions from \cite{CrIsLi} that we will use.
\end{rem}
\begin{defi}\label{def:viscosity}
Let $\mathcal{S}(N) \subset \R^{N \times N}$ denote symmetric $N \times N$ matrices, and let $F : \R^N \times \R \times \R^N \times \mathcal{S}(N) \to \R$ be a continuous function satisfying
\[
F(x,r,p,X) \leq F(x,s,p,Y) \, , \ \forall \ r \leq s, Y \leq X.
\]
For $\mathcal{O} \subset \R^N$, $\hat{x} \in \mathcal{O}$, $u : \mathcal{O} \to R$, the second-order superjet of $u$ at $\hat{x}$ is the set $J^{2,+}_\mathcal{O}u(\hat{x})$ of elements $(p, X) \in \R^N \times \mathcal{S}(N)$ such that, for all $\mathcal{O} \ni x \to \hat{x}$
\[
u(x) \leq u(\hat{x}) + \langle p, x - \hat{x}\rangle + \frac12 \langle X(x - \hat{x}), x - \hat{x}\rangle + o(|x - \hat{x}|^2).
\]
The corresponding subjet $J^{2,-}_\mathcal{O}u(\hat{x})$ is defined analogously, by the reversed inequality, or, equivalently, by $J^{2,-}_\mathcal{O}u(\hat{x}) = -J^{2,+}_\mathcal{O}(-u)(\hat{x})$.\\
We say that an upper semicontinuous $u: \mathcal{O} \to \R$ is a viscosity solution of $F \leq 0$ (subsolution) on $\mathcal{O}$ if\footnote{Note that, due to the monotonicity of $F$, this applies to the Laplace equation $F(x,r,p,X) = - tr(X)$ means $\Delta u \geq 0$}
\[
F(x, u(x), p, X) \leq 0 \quad \forall \, x \in \mathcal{O} \, , \ \forall \, (p, X) \in J^{2,+}_\mathcal{O}u(\hat{x}).
\]
Analogously, a lower semicontinuous $u: \mathcal{O} \to \R$ is called a viscosity solution of $F \geq 0$ (supersolution) on $\mathcal{O}$ if
\[
F(x, u(x), p, X) \geq 0 \quad \forall \, x \in \mathcal{O} \, , \ \forall \, (p, X) \in J^{2,-}_\mathcal{O}u(\hat{x}).
\]
A \emph{viscosity solution} of the equation
\[
F(x, u(x), Du(x), D^2u(x)) = 0
\]
is a function $u : \mathcal{O} \to \R$ that is both a viscosity subsolution and a viscosity supersolution of $F = 0$ on $\mathcal{O}$.
\end{defi}
}

The main challenge of equation \eqref{FP intro} is the degenerate diffusion provided by the nonnegative matrix $Q$. This equation is somehow an intermediate case between a diffusive Fokker-Planck equation and a pure transport equation. In the first case, when $Q$ is uniformly elliptic, the diffusion regularizes and  asymptotic convergence towards the unique steady state holds  whenever $L$ grows at infinity. However, in the latter case, where the diffusion matrix $Q$ is zero, the mass tends to concentrate around critical points of $L$. Specifically, the solution will tend to a sum of Dirac's deltas weighted according to the mass distribution of the initial datum.

We address this problem by focusing on three main questions outlined below. The first question concerns the initial stage of the learning process described in \eqref{SGD Main}:

\vspace{2mm}

\begin{center}
\textbf{Q1: How do parameters evolve in the first stage of training with SGD?}
\end{center}

\vspace{2mm}

Due to the smallness of the effective learning rate $\varepsilon>0$, we expect that the drift term, given by $\nabla L$, will lead the dynamics of the parameters for small times. As detailed in Section \ref{sec Drift regime}, if the loss function is $\lambda$-convex with $\lambda>0$ in a ball centered at a local minimum $x_0\in\mathbb{R}^d$, then we aim to obtain that
\begin{equation*}
\int_{B_{R(t)}(x_0)}\rho(t,x)\dx\ge\int_{B_{R_0}(x_0)}\rho_0(x)\dx-c(\varepsilon,t) \qquad\forall 0<t<T_\varepsilon\,,
\end{equation*}
where $R(t)$ is a decreasing radius and $0\le c(\varepsilon,t)\xrightarrow[]{\varepsilon\to 0}0$. In order to give a precise statement, we need to introduce the correct speed of decreasing of $R(t)$ and a suitable approximation of the characteristic function of a ball.

\begin{teo}[Local mass concentration]\label{Thm Concentration}
Assume that $L$ is $\lambda$-convex in $B_{(1+\delta)R_0}(0)$ with a minimum at 0 and $\lambda>0$.   Let $\varepsilon\in  (0,1)$ and   let $\rho$ be a weak solution of \eqref{FP intro} with $0\le Q(x)\le \sigma I_{d\times d}$ for every $x\in B_{(1+\delta)R_0}(0)$. Let us consider $\varphi(t,r):[t_0,\infty)\times\mathbb{R}_+\rightarrow\mathbb{R}_+$ a smooth cut-off function (see \eqref{smooth characteristic}) such that
$$
\varphi(t,r)\equiv 1\quad\mbox{if}\quad r\le R_0\, e^{-\frac{\lambda}{2}(t-t_0)} \qquad\mbox{and}\qquad\varphi(t,r)\equiv0\quad\mbox{if}\quad r>(1+\delta)R_0\, e^{-\frac{\lambda}{2}(t-t_0)}\,.
$$
 {
Then, there exists a $C$ depending only on $\delta, \sigma, d$ such that, for all $\alpha \in (0,1)$ it holds that
\begin{equation*}
\int_{\mathbb{R}^d}\varphi(t,|x|)\rho(t,x)\dx\ge\int_{\mathbb{R}^d}\varphi(t_0,|x|)\rho(t_0,x)\dx - \varepsilon^{2(1-\alpha)}.
\end{equation*}
for every $t_0<t<t_0 + T(\varepsilon, C)$ with $T(\varepsilon, C) = \frac{1}{\lambda}\log\left(1 +\frac{R_0^2 \lambda}{C \varepsilon^{2 \alpha}}\right)$.
}
\end{teo}

The result above highlights the strong dependence on the initial parameter distribution during the first stage of the \eqref{SGD Main}. Actually, it indicates that the mass of $\rho_0$ within the basins of attraction of $L$ tends to stay concentrated around the local minima for a while.

Once we know that if the effective learning rate is small enough there is a concentration phenomenon  around local minima, the next natural question is the following:

\vspace{2mm}

\begin{center}
\textbf{Q2: How long does SGD take to escape from a local minimum?}
\end{center}

\vspace{2mm}

For this purpose, we approach the issue from the perspective of the Mean Exit Time (MET) problem associated to \eqref{SDE intro}.
Given a domain $\Omega\subset\mathbb{R}^d$, the first exit time for the stochastic process $X_t$ starting at $x\in\Omega$ is given by
\begin{equation*}
  \tau_\Omega^x=\inf\lbrace t>0:X_t\not\in\Omega,X_0=x\rbrace\,.
\end{equation*}
 The MET is the function
\begin{equation*}
  u(x)=\mathbb{E}[\tau_\Omega^x]\,,
\end{equation*}
which satisfies the following elliptic equation \cite[Chapter 7.2]{Pavliotis}
\begin{equation}\label{Intro Mean Exit Time}
  \begin{cases}
    -\mathcal{A}u(x) = 1, & \mbox{in } \Omega  \\
    u(x)=0, & \mbox{on }\partial \Omega
  \end{cases}
\end{equation}
where
$$\mathcal{A}u(x) = \varepsilon^2\mbox{tr}\left(Q(x)D^2 u(x)\right)-\nabla L(x)\cdot\nabla u(x).$$

We devote Appendices \ref{AppendixMET} and \ref{sec Kramers Law} to recall the connection between MET and PDEs and other classical results  in the case of isotropic and non-degenerate diffusion, leading  to the well-known Kramers' Law \cite{Belgrund2013,Belgrund2006,Kramers}.

In order to cope with the diffusion matrix $Q$ in \eqref{SDE intro}, which is far from being uniformly elliptic, we provide new estimates for the MET associated to our problem.  {In this case we will consider viscosity solutions to the problem of MET, since this notion allows us to obtain the desired estimates from comparison principles. This notion applies to the problem of the MET as it is the dual problem, or Ornstein-Uhlenbeck problem, of the Fokker-Planck equation, whose solutions are probability distributions and hence weak solutions in the usual distributional sense \eqref{eq:def_weaksol}, or obtained as limits in the sense of measures, as we do for Theorem \ref{theo:existence}.}

 {The next two Theorems will be proved in Section \ref{sec Diffusion Regime} along with more general results.}

\begin{teo}[Lower bound for MET]\label{Thm Lower MET}
 {Let $x_0 \in \R^d, R > 0$. Assume $L_i\in C^{1,1}(\overline{B_R(x_0)})$ for every $i\ge1$, let $0\le Q(x)\le\sigma I_{d\times d}$ for all $x \in B_R(x_0)$, and let $L$ be $\lambda$-convex in $B_{R}(x_0)$ with a minimum at $x_0$.} Let $x\in B_r(x_0)$ with $0< r \le R$ and let the Mean Exit Time $\mathbb{E}\big[\tau_{B_R(x_0)}^x\big]$ be a viscosity solution of \eqref{Intro Mean Exit Time}.
 Then,
  \begin{equation}
    \mathbb{E}\left[\tau_{B_R(x_0)}^x\right]\ge\frac{R^2-r^2}{2\varepsilon^2 \sigma\,d}\,.
  \end{equation}
\end{teo}

Proving an upper bound for a general degenerate matrix $Q$ is more challenging, since the MET will  not be bounded if there is no diffusion. Nevertheless, assuming some non-degeneracy of the diffusion in just one direction is sufficient to provide an upper estimate.

\begin{teo}[Upper bound for MET]\label{Thm upper bound MET}
Let $\Omega \subset \R^d$ be a bounded domain, and let $R_\Omega > 0$ be such that $\Omega \subseteq B(0, R_\Omega)$. Let $L_i\in C^{1,1}(\overline{\Omega})$ for every $i\ge1$, let $M = \max_{x \in \overline{\Omega}}|\nabla L(x)|$. Assume that there exist $\beta>0$ and a vector $v\in\mathbb{S}^{d-1}$ such that
  \begin{equation}\label{assumption upper bound MET}
      v^TQ(x)v\ge\beta \qquad\forall x\in \Omega\,.
  \end{equation}
For $x\in \Omega$, let the Mean Exit Time $\mathbb{E}\left[\tau_{\Omega}^x\right]$ be a viscosity solution of \eqref{Intro Mean Exit Time}. Then,
  \begin{equation}\label{eq:METUpperBound}
    \mathbb{E}\left[\tau_\Omega^x\right]\le\frac{1}{M}\left(e^{\frac{2M (R_\Omega + 1)^2}{\beta\varepsilon^2}}-1\right).
  \end{equation}
\end{teo}
For the validity of the last result in the case of unbounded domains, see Remark \ref{rem:unboundedcomparison}.

\begin{rem}
It's worth comparing Theorem \ref{Thm Lower MET} with the previous result by \cite[Theorem 2]{HuJunchiLiLiLiu2019}. In that work, the authors were the first to prove that, under the assumptions that $Q$ is nondegenerate, and that $L$ is $C^2$, then asymptotically with the learning rate going to zero, the MET behaves as $e^{c/\varepsilon^2}$. With \eqref{eq:METUpperBound} we prove the same behavior, but this time the bound is not asymptotic, and it holds for general degenerate matrices and under a weaker smoothness assumption on the loss function.
\end{rem}

\begin{rem}
Condition \eqref{assumption upper bound MET}, locally means that $Q \not\equiv 0$. Let us be more precise on this. Since $Q(x)$ is symmetric and positive semidefinite at all $x \in \R^d$ by construction, having $Q(x) \neq 0$ at a point $x \in \R^d$ means that the largest eigenvalue $\lambda_1(x)$ is nonzero. Let $v_1(x)$ be a normalized eigenvector for the $\lambda_1(x)$ eigenspace. Since $Q : \R^d \to \R^{d \times d}$ is continuous by the hypothesis of having $L_i\in C^{1,1}(\overline{\Omega})$ for every $i\ge1$, and by the continuity of the spectrum (see e.g. \cite[Theorem 5.14]{Kato1980}), we have that $\lambda_1:\R^d \to \R_+$ is continuous. Thus, if $\Omega$ is small enough, one can obtain a continuous $v_1 : \Omega \to \mathbb{S}^{d-1}$. If $\underline{x} := \arg\min_{x \in \overline{\Omega}}\lambda_1(x)$, let us set $v := v_1(\underline{x})$. Then, for all $x \in \Omega$,
\[
v^TQ(x)v \geq \lambda_1(\underline{x}) (v_1(x)\cdot v_1(\underline{x}))^2.
\]
Hence, again if $\Omega$ is small enough, so that $(v_1(x)\cdot v_1(\underline{x}))^2$ has a nonzero minimum in $\Omega$, then \eqref{assumption upper bound MET} holds.
\end{rem}

\vspace{4mm}

The last question we consider in this manuscript concerns the asymptotic behaviour of the learning process. We aim to characterize and show the convergence of the parameter distribution for large times. Hence, we can state the last question as follows:

\vspace{4mm}

\noindent
\textbf{Q3: Does the distribution of neural network parameters converge using SGD?}

\vspace{4mm}
To address this question we implement two different methods relying on the known literature: a \emph{duality method} and an \emph{entropy method}.

\noindent
\textsc{Duality Method.}
In the non-degenerate case we exploit the dual equation of the Fokker-Planck equation \eqref{FP intro}, which is of Ornstein-Uhlenbeck type. We found that the recent results of Porretta \cite{Por} are especially useful in our setting, when the matrix is non-degenerate, and help to construct the stationary solutions by means of non-degenerate approximations. One of the key points in Porretta's approach is that the asymptotic stabilization of the solution  $\rho$ of \eqref{FP intro} for large times, is equivalent to the oscillation decay of the solution of the dual equation. As far as we know, this method has a drawback: it requires the diffusion matrix $Q$ to be uniformly elliptic. This is the reason why we introduce in Section \ref{sec Duality method} a variant of the \eqref{SGD Main} that we call Noisy Stochastic Gradient Descent \eqref{NSGD} by adding some gaussian noise in each iteration, something that kills the non-degeneracy. Namely,

\begin{equation}\label{NSGD}\tag{NSGD}
\theta_{n+1}=\theta_n-\frac{\eta}{\b}\sum_{b_i\in B_n}\nabla L_{b_i}(\theta_n)+\eta Z_n\qquad\forall n\ge 0\,,
\end{equation}
with $\lbrace Z_n\rbrace_{n\ge 0}$ being a family of i.i.d. gaussian processes satisfying for some $\delta>0$
\[
Z_n\sim\mathcal{N}(0,\delta I_{d\times d})\qquad\forall n\ge 0\,.
\]
 {We will show that, with the same approximation technique from \cite{LiTaiWeinan2017} that obtains \eqref{SDE intro} from \eqref{SGD Main}, the diffusion matrix of the Fokker-Planck equation associated to \eqref{NSGD} is $Q_\delta(x) = \varepsilon^2 Q(x) + \eta \delta I_{d\times d}$, so it is uniformly elliptic}. The results of \cite{Por} prove the convergence of the NSGD to a unique stationary measure, see Corollary \ref{Asymptotic convergence nondegenerate}.
This convergence is qualitative in several aspects: on the one hand, it only works in non-degenerate cases, which are not realistic in ML applications. On the other hand the convergence rate towards equilibrium cannot be quantified, since it is obtained by a non-constructive proof.
However, the above convergence for non-degenerate Fokker-Planck equations, gives us an approximation that allows to prove the existence of stationary measures for the ``original'' degenerate problem in the limit $\delta \rightarrow 0$. Indeed, in Section \ref{sec:existence} we will prove the following theorem.

\begin{teo}[Existence of steady states]\label{theo:existence}
Assume that the diffusion matrix $Q$ satisfies that there exist $\sigma_0, \sigma_1>0$ such that
	\begin{equation}\label{Assumption on Sigma intro}
	\|\sqrt{Q(\cdot)}\|_\infty\le\sigma_0\qquad\quad\mbox{and}\qquad\quad\|\sqrt{Q(x)}-\sqrt{Q(y)}\|\le \sigma_1|x-y| \qquad\forall x,y\in \R^d\,.
	\end{equation}
	Assume that the drift term $\nabla L(x)$ satisfies that there exist $\alpha,R>0$ and $\gamma\ge 2$ such that
	\begin{equation}\label{confining potential intro}
	\nabla L(x)\cdot x\ge\alpha|x|^\gamma\qquad\forall x\in \R^d\quad\mbox{with}\quad |x|\ge R,
	\end{equation}
    and	 there exists $c_0>0$ such that
	\begin{equation}\label{Assumption on L intro}
	(\nabla L(x)-\nabla L(y))\cdot(x-y)\ge-c_0|x-y|\qquad\quad\forall x,y\in\R^d\,.
	\end{equation}
Then there exists  at least one invariant probability measure $\rho_\infty\in\mathcal{P}(\mathbb{R}^d)$ such that
\begin{equation}
    \nabla\cdot\bigg(\varepsilon^2\nabla\cdot\left(Q\rho_\infty\right)+\rho_\infty\nabla L\bigg)=0\,,
\end{equation}
in the sense of measures, that is, for all $\varphi\in C^{1,2}(\mathbb{R}^d)$ it holds that
	\begin{equation}
	\int_{\mathbb{R}^d}\bigg[\varepsilon^2\mbox{\rm tr}\left(Q(x)D^2\varphi(x)\right)-\nabla L(x)\cdot\nabla\varphi(x)\bigg]\d\rho_\infty(x)=0\,.
	\end{equation}

\end{teo}

\noindent
\textsc{Entropy Method. }
In Section \ref{sec Entropy method} we discuss in detail the problem of finding the right entropy for the class of Fokker-Planck equations studied in this paper. This is in general a delicate task, and in particular for our setting it requires to modify in a substantial way the common approach. Indeed, for the standard Fokker-Plank equation:
\[
\partial_t \rho  = \nabla\cdot(\nabla \rho + \rho  \nabla L)
\]
it is immediate to see that $\rho_\infty=e^{-L}$ is a stationary state, since it annihilates the vector field inside the divergence, that is $\nabla \rho_\infty= - \rho_\infty \nabla L$. Then, a natural choice for the relative entropy would be the standard weighted $L^2$-norm: $\mathcal{E}[\rho|\rho_\infty]=\int_{\R^d} |\rho/\rho_\infty\, -1|^2 \rho_\infty\dx$. Indeed, in this case, standard arguments allows one to deduce that $\mathcal{E}$ decreases along the flow. However, these arguments exploit in an essential way the fact that $\rho_\infty$ annihilates the vector field inside the divergence. In our case \textsl{it is not true in general that the inner vector field vanishes on the stationary state}: we can only rely on the full equation
\[
\nabla\cdot\bigg(\varepsilon^2\nabla\cdot\left(Q(x)\rho_\infty\right)+\rho_\infty\nabla L(x)\bigg) = 0
\]
in a suitable weak sense. As a consequence, the computation of the entropy production becomes nontrivial. We were able to overcome this difficulty, and we provide in Lemma \ref{EEP-new} a proof of the monotonicity of the entropy $\mathcal{E}$ along the ``degenerate flow'', which we consider a significant novelty, to the best of our knowledge. This is the result that, when combined with suitable weighted Poincar\'e inequalities, provides exponential convergence towards equilibrium. The validity of such a Poincar\'e inequality for degenerate quantities is a delicate matter, and in really elementary cases\footnote{ The simplest case in which a Poincar\'e inequality in the sense of Definition \ref{poincare ineq} does not hold is when $Q$ is a diagonal constant coefficient matrix with all eigenvalues zero except one.} it may simply fail. Hence, this problem presents its own interest, and it requires a thorough study which goes beyond the scope of this paper.

In the particular situations of a diffusion matrix that is degenerate but constant, and of a quadratic loss function $L(x)=\frac{1}{2}x^TCx$, equation \eqref{FP intro} becomes\vspace{-2mm}
\begin{equation}
  \partial_t\rho=\nabla\cdot\left(Q_0\nabla\rho+\rho Cx\right).
\end{equation}
In this case, the Bakry-Émery \cite{Bakry-Emery} approach applies, as proven by Arnold and Erb \cite{Arnold2014}, when the following two conditions are satisfied
\begin{enumerate}[leftmargin=8.5mm]
    \myitem[I]\label{I intro}  {\it Confinement potential:} $C$ is positive stable, i.e. all eigenvalues have positive real part.
    \myitem[II]\label{II intro} {\it Hörmander's condition:} There are no eigenvectors  of $C$ in ker$\,Q_0$.
\end{enumerate}
The first assumption ensures the $\lambda$-convexity of $L$, while the second ensures that the diffusion is spread in the whole space by the drift.

When condition \eqref{II intro} fails, the stationary state may not have an $L^1$ density, showing that our existence Theorem \ref{theo:existence} is somehow sharp. We shall consider next an example in which exponential convergence in a suitable topology holds even when condition \eqref{II intro} is not met. More precisely, we will prove convergence for the degenerate non-Hörmander case which reads
\begin{equation}\label{Generalized simple FP intro}
    \partial_t u=\nabla_{x,y}\cdot\left(
    \begin{pmatrix}
      Q_0 & 0 \\
      0 & 0
    \end{pmatrix}
    \begin{pmatrix}
      \nabla_x u  \\
      \nabla_y u
    \end{pmatrix}
    +u
    \begin{pmatrix}
      C_0 & 0 \\
      0 & C_3
    \end{pmatrix}
    \begin{pmatrix}
      x \\
      y
    \end{pmatrix}
    \right)
\end{equation}
with $u:(0,\infty)\times\mathbb{R}^n\times\mathbb{R}^{d-n}\rightarrow\mathbb{R}^d$, assuming only that condition \eqref{II intro} holds for $Q_0$ and $C_0$, i.e. for the variables that we have called $x$.

In this framework, we prove the asymptotic convergence of \eqref{Generalized simple FP intro} in the 2-Wasserstein distance and the exponential convergence of the second moments to a steady state of the form
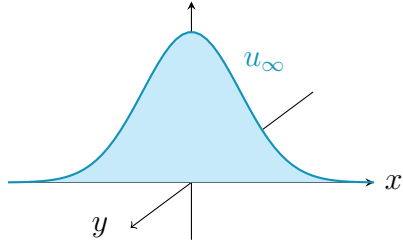
\begin{multicols}{2}
\begin{equation*}
  u_\infty(x,y)=g_\infty(x)\,\delta_0(y)\,.
\end{equation*}
\color{white} as
\color{black}
\begin{Figure}
  \centering
\pgfmathdeclarefunction{gauss}{2}{%
    \pgfmathparse{1/(#2*sqrt(2*pi))*exp(-((x-#1)^2)/(2*#2^2))}%
}

\begin{tikzpicture}[scale=0.8]
    \begin{axis}[
        no markers,
        axis x line = center,
        axis y line = center,
        xlabel = {$ x $}, xlabel style = {right},
        ylabel = {}, ylabel style = {above},
        xmin = -3, xmax = 3,
        ymin = -0.19, ymax = 0.6,
        hide obscured x ticks=true,
        xtick=\empty,
        ytick = \empty,
        x = 1cm, y = 5cm,
        domain = -5:5,
        samples = 100
    ]
        \addplot [<-,black,domain=-1:0]{0.15*(x)};
        \addplot [black,dashed,domain=0:1.17]{0.15*(x)};
        \addplot [black,domain=1.17:2]{0.15*(x)};
        \addplot [fill=cyan!20, draw=none, domain=-5:5,opacity=0.3] {gauss( 0, 0.8)} \closedcycle;
        \addplot [thick, cyan!70!black] {gauss( 0, 0.8)};

    \node[font = \color{cyan!70!black}] at (1.2,0.4) {$u_\infty$};
    \node[] at (-1.5,-0.15) {$y$};

    \end{axis}
\end{tikzpicture}
\end{Figure}

\end{multicols}
\noindent
where $g_\infty(x) = c\,e^{-\frac{x^TK^{-1}x}{2}}$ and $K \in \R^{d \times d}$ is the unique positive definite matrix satisfying $2Q_0=C_0K+KC_0$.

\begin{teo}[Convergence in the Non-Hörmander case]
     Let us consider the equation \eqref{Generalized simple FP intro} with $u_0\in L^1(\mathbb{R}^d)\cap\mathcal{P}_2(\mathbb{R}^d)$ and with marginal on the $x$-variable in $L^2(\mathbb{R}^n,g_\infty^{-1}\dx)$. Let $Q_0,C_0\in\mathbb{R}^{n\times n}$ satisfy assumptions \eqref{I} and \eqref{II} in $\mathbb{R}^n$. If $C_3\in\mathbb{R}^{(d-n)\times (d-n)}$ is positive stable, then
  \begin{equation*}
    u(t) \xrightarrow[]{t\to \infty} u_\infty
  \end{equation*}
  in 2-Wasserstein. Moreover, the following convergence of the second moments holds
  \begin{equation}
    \left|\int_{\mathbb{R}^n}\int_{\mathbb{R}^{d-n}}(|x|^2+|y|^2)(u(t,x,y)-u_\infty(x,y))\dx\dy
    \right|\le \kappa e^{-\lambda t}\,,
  \end{equation}
  with $\lambda>0$, and $\kappa$ is a constant depending on $u_0, Q_0, C_0$.
\end{teo}
We will prove an expanded version of this statement, with a quantitative rate $\lambda$, in Theorem \ref{Thm Wasserstein Convergence}.

This may not seem a realistic model for our ML problem. However, this may be considered as a description of the local behaviour of the Fokker-Planck equation \eqref{FP intro}, through a linearization of $Q$ and $L$ around the local minima of $L$. Indeed, in Section \ref{sec Open Questions}, we  exploit this approximation argument by defining, for $h \in (-1,1)$ small,
\begin{equation*}
  u_h(t,z)=\rho(t,x_0+ hz)
\end{equation*}
where $x_0\in\mathbb{R}^d$ is a local minimum of $L$. Thus, the equation for the new function $u_h = u_h(t,z)$ reads
\begin{equation}\label{linearized eq}
  \partial_t u_h=\nabla_z\cdot\left(\varepsilon^2Q(x_0)\nabla_z u_h+u_hD^2L(x_0)z\right)
\end{equation}
up to an error of order $|h|$.

We conclude this manuscript by presenting some open questions concerning the approximation of the steady state of \eqref{FP intro} by
\begin{equation}\label{linearized approximation}
  \rho_\infty(x)\approx\sum_{i=1}^{M}m_i(\infty)\;u_{i,\infty}\left(\frac{x-x_i}{h}\right)\,,
\end{equation}
where $u_{i,\infty}$ denotes the steady state of \eqref{linearized eq} linearized around the local minimum $x_i\in\mathbb{R}^d$, and $m_i(\infty)$ represent the mass partition such that $\sum_{i=1}^{M}m_i(\infty)=1$.

\vspace{4mm}

\noindent\textbf{Notation: }Let $I_{d\times d}$ be the identity matrix of dimension $d\times d$. We denote by $\mathcal{M}(\mathbb{R}^d)$ the space of probability measures in $\mathbb{R}^d$. For these measures, we define the norm $\|\mu\|_{\mathcal{M}_k}=\|\mu\|_{\mathcal{P}_k}:=\int_{\mathbb{R}^d}(1+|x|^2)^{k/2}\d|\mu|(x)$ and the space $\mathcal{M}_k(\mathbb{R}^d):=\lbrace \mu\in\mathcal{M}(\mathbb{R}^d):\|\mu\|_{\mathcal{M}_k}<+\infty\rbrace$. Analogously, we write $\mathcal{P}_k(\mathbb{R}^d)$ for the probability measures with $k$ finite moments.

\subsection{Related work on the SGD and the SDE approximation}\label{sec:methodology}

Stochastic gradient descent (SGD) has a long history as an optimization technique in machine learning in general, and for neural networks in particular (see e.g. \cite{BachMoulines2011,LeCun1998}). In recent years, SGD has been the object of several theoretical and experimental studies, considering different aspects of the algorithm, due to its unique mixture of desirable properties of quality of the results, robustness, and generalization capabilities. Special attention has been devoted to the implicit bias, or implicit regularization, in the optimization (for recent reviews see \cite{GoyalBengio2022, Vardi2023}), motivated by seminal works like \cite{Zhang2017} which show that neural networks are effectively trained by SGD for a number of parameters that exceeds the number of data point (overparametrized), and that results are qualitatively unaffected by explicit regularization. On the other hand, it has been observed that more efficient adaptive optimization methods may generalize worse than SGD \cite{Wilson2017, Zhou2020}. While implicit regularization in gradient descent algorithms can be attributed to multiple factors, involving interactions between the optimization algorithm and the properties of both the parametrization and the dataset \cite{AroraCohenHuLuo2019, BarrettDherin2021, CCFF2024,Gunasekar2018,Soudry2018}, the stochasticity of SGD is believed to play a relevant role, in particular due to its state-dependent noise \cite{Blanc2020, HaoChen2021}. Another crucial aspect of SGD design is the dependence of the results on the ratio between learning rate and batch size, also called linear scaling rule, which appears to be a key factor to obtain so-called flat minima, ensuring better generalization \cite{Goyal2017, Jastrzebski2018} (but see also \cite{MastersLuschi2018}).

The continuous stochastic approximation of SGD with an SDE that we consider in this paper was introduced and proved in \cite{LiTaiWeinan2017}. However, the idea of considering a stochastic iterative algorithm as a discretization of a continuous time stochastic process, or, equivalently, that a continuous-time SDE could be obtained as a limiting procedure, is classical (see e.g. \cite[Chapter 9]{KushnerYin2003}). More recently, an SDE model for SGD has been proposed in \cite{MandtHoffmanBlei2016, MandtHoffmanBlei2017}, where the authors assume a gaussian state-independent covariance matrix as a model for a large number of samples, and in \cite{Jastrzebski2018}, where the authors emphasize that different learning rates with the same scaling to batch size give rise to the same continuous approximation and discuss the common practice of assuming the covariance to be approximately equal to the Hessian of the loss (sometimes called label noise). In \cite{ShiSuJordan2023} a different, but still state-independent, covariance matrix is considered, together with a time-dependent learning rate, which allow the authors to prove convergence to a stationary state with similar arguments to the ones used in the present paper for state-dependent diffusion matrices.  {Here, for the sake of comparison, we prove existence of a stationary state with a degenerate and nonconstant diffusion, and we provide new arguments in terms of entropy methods to attack the problem to convergence to a stationary state for the original SDE model.} A learning rate that decreases with time has also been considered in \cite{MoucerTaylorBach2023}, to prove convergence of SGD for convex optimization problems, again in terms of the SDE continuous model.  {On the other hand, nonconvex scenarios for stochastic learning were considered already in \cite{RRT17}, although using a spatially independent nondegenerate diffusion: we recall that, as mentioned after equation \eqref{SDE intro},  diffusion in the SDE approximation needs to be degenerate and, in the common overparameterized situations, is highly degenerate for dimensional issues. Also, spatial dependence is highly nontrivial in real world scenario, because it contains the different behaviour of each data with respect to the system parameterization.} A general approach to the stochastic approximation with SDE of online learning in terms of the theory of semigroups was given in \cite{FengLiLiu}. A different proof of the stochastic approximation of SGD was provided in \cite{ChengYinBartlettJordan2020}, where the authors show experimentally the limiting behaviors for general state-dependent noise. Key properties of the diffusion that approximates SGD are given in \cite{HuJunchiLiLiLiu2019}, where the authors discuss the limiting behavior and the escape time for non-degenerate diffusion.  {We can compare in particular our treatment of Question 2 with \cite[Theorem 2]{HuJunchiLiLiLiu2019}, where the authors obtain a exponential time for the escape from local minima from non spatially homogeneous Fokker-Planck equations by means of large deviations theory. Due to the technique used, they need nondegenerate diffusion, and can deduce the exponential decay only asymptotically for small learning rates. On the other hand, with a different technique, with Theorem \ref{Thm upper bound MET} we provide a quantitative exit time result that is not asymptotic, and applies to degenerate diffusion, that is the generic situation.}. Two regimes in the SDE model of SGD, that dominated by drift, and the diffusive one, were considered in \cite{BenArousGheissariJagannath2022} and applied to elementary architectures: these results should be compared directly with our theoretical study in Section \ref{sec Drift regime}, where we can obtain general quantitative results of mass concentration around minima in the drift regime.

Two observed features of the SGD that are reproduced by the SDE continuous model are the scaling law, which is intrinsic in the limiting procedure, and the implicit regularization. This second one has been studied for SDE in a long series of works. In \cite{ChaudhariSoatto2018}, under the assumption of the existence and uniqueness of a steady state obtained by a divergence-free force, the authors show that the Fokker-Planck trajectories are monotone minimizer of the KL-divergence with respect to the steady state, hence justifying a Bayesian inference of the SGD, and show that such a state is not a minimum of the loss function itself but rather of a regularization in terms of the Shannon entropy. In \cite{XieSatoSugiyama2021} the authors assume the covariance of the SDE to be close to the Hessian of the loss function and deduce different escape time estimates from minima depending on the local spectral properties of the Hessian, hence defining an implicit bias towards flat minima. Still modeling the noise covariance with the Hessian of the loss, \cite{LiWangArora2022} discuss yet another form of implicit bias in SGD through a continuous SDE model: they show that in the presence of a manifold of minima, SGD does not randomly walks on that manifolds once it reaches it, but rather continues to minimize the trace of the Hessian, hence stabilizing around the flattest region of the minima. They obtain this result by first observing that the dynamics that is normal to the manifold approximately behaves like a constant-covariance Ornstein-Uhlenbeck process, and deduce asymptotically reaching the manifold of minima, and then study the tangent dynamics. The techiques that we introduce that allow us to obtain a rigorous proof of existence of stationary states for the SDE with the true covariance matrix, and a proof of the decay in the non degenerate case, may provide insight on how to consider the tangent dynamics for a general state dependent noise.

Limitations of the SDE model of SGD have been pointed out in several occasions. In \cite{Yaida2019} the whole approximation procedure is questioned in terms of heuristic arguments concerning the forcing of the scaling law in the limiting procedure, and the difference between the covariance matrix and the more commonly considered Hessian. More recently, \cite{LiMalladiArora2021} have considered the possible deviations of the SDE approximation of SGD for large values of the quotient between learning rate and batch size, by directly comparing trajectories of SGD with a fine discretization of the SDE. Large learning rates are indeed often used in practice to provide an implicit regularization bias \cite{Andriushchenko2023,Goyal2017, LiWeiMa2019}.

Extensions of SDE modeling of SGD currently include proof of stochastic approximation for adaptive schemes such as ADAM \cite{MallladiLyuPanigrahiArora2022}, while \cite{Gurbuzbalaban2021,Simsekli2019} have considered non-Gaussian noise.

Another notable mathematical approach to modeling the dynamics of neural network parameters trained via SGD relies on mean field limit techniques.
When the number of neurons in a layer grows to infinity, the empirical distribution of the network parameters evolves according to a deterministic partial differential equation, expressed as a Wasserstein gradient flow, see for instance the survey \cite{FernandezFigalli2022}.
This perspective allows one to bypass the complexities of individual parameter updates and instead analyze the global behavior of the parameter distribution. However, this approach is limited to shallow networks or to deep linear networks, where a well-posed infinite-width limit can be established. In particular, \cite{Rotskoff2019} rigorously proves a Law of Large Numbers and a Central Limit Theorem quantifying fluctuations around the mean field trajectory, while \cite{CCFF2024} shows that in a deep linear network, the infinite-width limit exhibits exponential convergence to a continuous-time limit.

\section{Analysis in the drift regime}\label{sec Drift regime}

The dynamics of the parameters evolving according to the SGD present two different regimes of behaviour. On the first regime, which we call the drift regime, the diffusion is weaker than the drift term $\nabla L$ and the latter is the one who determines the behaviour of the learning process. This fact leads to a concentration phenomena around the local minima of the loss function $L$ whenever the effective learning rate $\varepsilon>0$ is small enough. As we will discuss in Section \ref{sec Diffusion Regime}, once the concentration at small time has occurred, the diffusion tends to spread out the parameters. In this section, we are interested in quantifying this concentration behaviour at small times,  {and addressing the question Q1 formulated in the introduction: \emph{How do parameters evolve in the first stage of training with SGD?}}

\

To this end, we study the probability density function $\rho$ of the stochastic process associated to the SGD. Let us recall that $\rho$ satisfies the equation
\begin{equation}\label{Fokker-Planck 2}
  \begin{cases}
  \quad\;\,\partial_t\rho=\nabla\cdot\bigg(\varepsilon^2\nabla\cdot (Q(x)\rho)+\rho\nabla L(x)\bigg)\qquad\mbox{on}\quad(0,\infty)\times\R^d\,,\\[4mm]
  \rho(0,x)=\rho_0(x)\hspace{55mm}\mbox{in}\quad\R^d\,,
  \end{cases}
\end{equation}
with $\nabla\cdot A=(\nabla\cdot A_{1\cdot},\dots,\nabla\cdot A_{d\cdot})$ and $A\in\mathbb{R}^{d\times d}$. In this framework, we will address the above question by obtaining lower bounds for the following quantity:
\begin{equation*}
  \int_{B_R(x_0)}\rho(t,x)\dx\,,
\end{equation*}
with $x_0\in\mathbb{R}^d$ being a local minimum of $L$ and $R>0$ being certain radius. Nevertheless, before performing a local analysis, let us show a global concentration estimates when $L$ is strictly convex in the whole space $\mathbb{R}^d$. This can be easily deduced from the following upper bound of the second moments of $\rho$.
\begin{lem}\label{second moments convergence for lambda convex}
  Assume that $L$ is $\lambda$-convex with $\lambda>0$, $L$ has a unique minimum in $x_0$ with $L(x_0)=0$ and  $0\le \mbox{\rm tr}(Q(x))\le\sigma$ for every $x\in\mathbb{R}^d$. If $\rho_0\in \mathcal{P}_2(\mathbb{R}^d)$, then $\forall t>0$
  $$\int_{\mathbb{R}^d}|x-x_0|^2\rho(t,x)\dx\le\left(\int_{\mathbb{R}^d}|x-x_0|^2\rho_0(x)\dx-\varepsilon^2\frac{2\sigma}{\lambda}\right)e^{-\lambda t}+\varepsilon^2\frac{2\sigma}{\lambda}\,.$$
\end{lem}
\begin{proof}
  Let us differentiate the second moment of the solution.
  \begin{align*}
    \frac{\d}{\dt}\int_{\mathbb{R}^d}|x-x_0|^2\rho(t,x)\dx&=
    \varepsilon^2\int_{\mathbb{R}^d}\mbox{tr}(QD^2(|x-x_0|^2))\rho(t,x)\dx\\
    &\;\;-\int_{\mathbb{R}^d}\nabla\left(|x-x_0|^2\right)\cdot\nabla L\,\rho(t,x)\dx\\
    &=2\varepsilon^2\int_{\mathbb{R}^d}\mbox{tr}(Q(x))\rho\dx-2\int_{\mathbb{R}^d}(x-x_0)\cdot\nabla L(x)\rho\dx
  \end{align*}
Note that since $L$ is $\lambda$-convex and $x_0$ is any minimum, we have
$$(x-x_0)\cdot\nabla L(x)\ge L(x)-L(x_0)+\frac{\lambda}{2}|x-x_0|^2\ge\frac{\lambda}{2}|x-x_0|^2\,.$$
Using this estimate and $0\le\mbox{tr}(Q(x))\le\sigma$ we obtain
\begin{align*}
  \frac{\d}{\dt}\int_{\mathbb{R}^d}|x-x_0|^2\rho(t,x)\dx&\le2\varepsilon^2\sigma-\lambda\int_{\mathbb{R}^d}|x-x_0|^2\,\rho(t,x)\dx\,.
\end{align*}
The result follows by integrating this expression in $[0,t]$.
\end{proof}
Once we have estimated the second moment of $\rho$, it is possible to obtain a concentration estimate by Chebyshev's inequality.
\begin{cor}\label{mass concentration for lambda convex}
  Under the assumptions of Lemma \ref{second moments convergence for lambda convex}, we have that for every $R>0$ and $t>0$
  \begin{equation}\label{concentration of mass}
    \int_{B_R(x_0)}\rho(t,x)\dx\ge1- \frac{1}{R^2}\left[\left(\int_{\mathbb{R}^d}|x-x_0|^2\rho_0(x)\dx-\varepsilon^2\frac{2\sigma}{\lambda}\right)e^{-\lambda t}+\varepsilon^2\frac{2\sigma}{\lambda}\right]\,.
  \end{equation}
\end{cor}
\begin{proof}
For estimate \eqref{concentration of mass}, we use Chebyshev's inequality: For any $\mu$-measurable function $f$, $R>0$  and $1<p<\infty$ it holds that
$$\mu\left(\left\{x\in\mathbb{R}^d:|f(x)|>R\right\}\right)\le\frac{1}{R^p}\int_{\mathbb{R}^d}|f(x)|^p\d\mu(x)\,.$$
Hence, considering $\mu=\rho(t)$, $f(x)=|x-x_0|$ and $p=2$, we obtain
$$\int_{|x-x_0|>R}\rho(t,x)\dx\le \frac{1}{R^2}\int_{\mathbb{R}^d}|x-x_0|^2\rho(t,x)\dx\,.$$
Then, we use Lemma \ref{second moments convergence for lambda convex} and the fact that $\rho(t)\in\mathcal{P}(\mathbb{R}^d)$ to control the second moment and to conclude.
\end{proof}

\noindent{\sc Example 4: non degenerate case}.
We remark that Lemma \ref{second moments convergence for lambda convex} and Corollary \ref{mass concentration for lambda convex} does not imply the convergence to $\delta_{x_0}(x)$ in any sense due to the error term $\varepsilon^2\frac{2\sigma}{\lambda}$. However, this error seems to be unavoidable whenever there is a diffusion term. Let us consider the case $Q(x)=\frac{\sigma}{d} I$  and $L(x)=\frac{\lambda}{2}|x-x_0|^2$ for every $x\in\mathbb{R}^d$ with $\sigma,\lambda>0$. In this case, the fundamental solution is smooth and reads as
$$\rho(t,x)=\frac{1}{\left(2\pi\frac{\varepsilon^2\sigma}{\lambda d}(1-e^{-2\lambda t})\right)^{d/2}}\exp\left(-\frac{\lambda d|x-x_0|^2}{2\varepsilon^2\sigma(1-e^{-2\lambda t})}\right)\,.$$
Moreover, we know that $\rho(t)\rightarrow\rho_\infty$ uniformly in compacts, in 2-Wasserstein distance and in $L^2(\mathbb{R}^d,\rho_\infty^{-1})$ with
$$\rho_\infty(x)= \frac{1}{\left(2\pi\frac{\varepsilon^2\sigma}{\lambda d}\right)^{d/2}}\exp\left(-\frac{\lambda d|x-x_0|^2}{2\varepsilon^2\sigma}\right)\,.$$
However, $\rho_\infty$ does not concentrate in $x_0$ due to the exponential tails. Indeed, the second moment is
$$\int_{\mathbb{R}^d}|x-x_0|^2\rho_\infty(x)\dx=\mbox{tr}\left(\frac{\varepsilon^2\sigma}{\lambda d}I\right)=\frac{\varepsilon^2\sigma}{\lambda}\,.$$

One may think that this is not a representative example since the diffusion is non-degenerate. Nevertheless, let us consider the degenerate case where the operator is not even hypocoercive.

\vspace{2mm}
\noindent{\sc Example 5: degenerate case}.
Assume $L(x)=\frac{\lambda}{2}|x|^2$ with global minimum at $x_0=0$ and the diffusion matrix
$$Q(x)=\begin{pmatrix}
         \sigma & 0\dots &0 \\
         0 & 0\dots &0 \\
         \vdots &\;\;\;\; \ddots&\vdots\\
         0&\;\;\;\dots&0
       \end{pmatrix}\,,$$
with $\sigma,\lambda>0$. Note that this operator is not hypocoercive since $\lambda I e_i=\lambda e_i$ and $e_i\in\ker(Q)$ for any $i=2,\dots,d$. Given an initial datum $\rho_0\in\mathcal{P}_2(\mathbb{R}^d)$, the solution to
$$\partial_t \rho=\nabla\cdot\left(\varepsilon^2Q\nabla \rho+\lambda x\rho\right)\,,$$
is the following:
$$\rho(t,x)=\frac{e^{\lambda(d-1)t}}{\sqrt{2\pi\frac{\varepsilon^2\sigma}{\lambda}(1-e^{-2\lambda t})}} \; \int_{\mathbb{R}}\exp\left(-\frac{\lambda|x_1-\tilde{x}_1e^{-\lambda t}|^2}{2\varepsilon^2\sigma(1-e^{-2\lambda t})}\right)\;\rho_0\left(\tilde{x}_1,x'e^{\lambda t}\right)\d\tilde{x}_1\,,$$
with $x'=(x_2,\dots,x_d)\in \mathbb{R}^{d-1}$. Due to the tail decay of $\rho_0\in\mathcal{P}_2(\mathbb{R}^d)$, the stationary solution is
$$\rho_\infty(x)=\frac{1}{\sqrt{2\pi\frac{\varepsilon^2\sigma}{\lambda}}}\; \exp\left(-\frac{\lambda|x_1|^2}{2\varepsilon^2\sigma}\right)\;\delta_{0'}(x')$$
See Theorem \ref{Thm Wasserstein Convergence} for a 2-Wasserstein convergence of $\rho(t)\rightarrow\rho_\infty$ when $t\rightarrow\infty$.
Note that even in this very degenerated case the second moment of the invariant measure is given by
$$\int_{\mathbb{R}^d}|x|^2\rho_\infty(x)\dx =\frac{1}{\sqrt{2\pi\frac{\varepsilon^2\sigma}{\lambda}}} \int_{\mathbb{R}}|x_1|^2\exp\left(-\frac{\lambda|x_1|^2}{2\varepsilon^2\sigma}\right)\dx_1=\frac{\varepsilon^2\sigma}{\lambda}$$

Based on this simple example, we aim to estimate this error that the diffusion generates in the concentration phenomena around local minima. As in the example, the main influent parameters are the $\lambda$-convexity of the loss function in the basin of attraction, the effective learning rate $\varepsilon>0$ and the $L^\infty$-norm of the matrix $Q$, which quantify the diffusion.
From an application perspective, these parameters establish the likelihood that the weights of the neural network will fall into the nearest local minimum.

\subsection{ Local mass concentration}\label{Mass concentration in the drif regime}

In this section, we show that in the first steps of the learning process, the parameters tends to concentrate  with high probability around the local minima of the loss function. Assuming without loss of generality that $0$ is a local minimum of $L$, we would like to obtain an estimate on the following quantity, for small values of $r$
\begin{equation}\label{local mass}
  \mathbb{P}\lbrace X_t\in B_r(0)\rbrace=\int_{B_r(0)}\rho(t,x)\dx\,.
\end{equation}
The main problem is that the variation of this quantity depends on the flux on the boundary of $B_r(0)$ which is a priori unknown. Namely, if we try to  differentiate this quantity using the equation of the probability density function, we obtain
\begin{align*}
  \frac{\d}{\dt}\int_{B_r(0)}\rho(t,x)\dx & =\int_{B_r(0)}\nabla \cdot\bigg(\varepsilon^2\nabla\cdot(\rho Q)+\rho\nabla L\bigg)\dx\\
  &=\int_{|x|=r}\bigg(\varepsilon^2\nabla\cdot(\rho Q)+\rho\nabla L\bigg)\cdot \frac{x}{|x|}\dx
\end{align*}
Since this variation is complicated to estimate with \eqref{local mass} to close the differential inequality, we will consider a different quantity with a controllable variation. To motivate this new approach, let us consider the following continuity equation ($\varepsilon=0$)
\begin{equation}\label{transport}
\begin{cases}
  \partial_t \mu=\nabla\cdot\left(\mu\nabla L\right)\\
  \mu(0,x)=\mu_0(x)
\end{cases}
\end{equation}
with $L$ being $\lambda$-convex with $\lambda>0$ and a global minimum at $0$. It is well known that if $\mu_0\in\mathcal{P}_2(\mathbb{R}^d)$, then the solution $\mu(t)$ converge to $\delta_0$ in the Wasserstein topology \cite{Ambrosio-Brue-Semola}, but we are interested in the speed of concentration. Moreover, the associated flow dynamic is
\begin{equation*}
  \begin{cases}
    \frac{\d}{\dt} \Phi_t(x)&=-\nabla L(\Phi_t(x)) \\
    \Phi_0(x)&=x.
  \end{cases}
\end{equation*}
and the solution of \eqref{transport} is the pushforward $\mu(t,x)=(\Phi_t)_\#\mu_0(x)$. In order to show the concentration velocity of $\Phi_t(x)$, let us compute the variation in time of the Euclidean distance from $\Phi_t(x)$ to the global minimum at 0:
\begin{align*}
  \frac{\d}{\dt}|\Phi_t(x)-0|^2 & =2 \frac{\d}{\dt}\Phi_t(x)\cdot\left(\Phi_t(x)-0\right)=-2\nabla L(\Phi_t(x))\cdot \left(\Phi_t(x)-0\right) \\
  &\le- \lambda |\Phi_t(x)-0|^2\,,
\end{align*}
which implies that
\begin{equation}\label{flow map convergence}
  |\Phi_t(x)|\le e^{-\frac{\lambda}{2}t}|x|\qquad\forall t>0.
\end{equation}
Hence, the flow tends to concentrate particles starting from a point $x\in\mathbb{R}^d$ around $0$ with exponential velocity. This phenomena can be understood in the sense of the following lemma.
\begin{lem}\label{transport shrinking radius}
Let $L$ be a $\lambda$-convex function with $\lambda>0$ and let $\mu$ be a solution to \eqref{transport}  in the sense of \cite{Ambrosio-Brue-Semola} with $\mu_0\in\mathcal{P}_2(\mathbb{R}^d)$. Then, for every $t>0$ it holds that
  \begin{equation*}
  \int_{B_{R(t)}(0)}\mu(t,x)\dx\ge\int_{B_{R_0}(0)}\mu_0(x)\dx
\end{equation*}
with $R(t)=R_0\,e^{-\frac{\lambda}{2} t}$.
\end{lem}
\begin{proof}
  Recall that $\mu(t,x)=(\Phi_t)_\#\mu_0(x)$. Then,
  \begin{equation*}
     \int_{B_{R(t)}(0)}\mu(t,x)\dx=\int_{\Phi_t^{-1}(B_{R(t)}(0))}\mu_0(x)\dx\,.
  \end{equation*}
  Since $R(t)=R_0\,e^{-\frac{\lambda}{2} t}$ and \eqref{flow map convergence} holds, we have that
  $$\Phi_t(B_{R_0}(0))\subseteq B_{R(t)}(0)\,,$$
  and the result follows.
\end{proof}
Considering the diffusion equation \eqref{Fokker-Planck 2} for $\rho$. Compared with the result in the lemma above, we expect that the probability of a suitable shrinking ball should be non decreasing up to an error depending on $\varepsilon$ and before  certain time. Namely, our goal is to obtain a lower bound in terms of the initial datum $\rho_0$ of the form
\begin{equation*}
  \int_{B_{R(t)}(0)}\rho(t,x)\dx\ge \int_{B_{R_0}(0)}\rho_0(x)\dx-c(\varepsilon,t)
\end{equation*}
with an appropiate shrinking radius $R(t)$ and $c(\varepsilon,t)$ tending to 0 as $\varepsilon$ goes to 0. Since we want to use the equation of $\rho$ and test it with the characteristic function $\chi_{B_{\R(t)}(0)}(x)$, we use the following $C^{1,2}$ approximation cut-off function
\begin{equation}\label{smooth characteristic}
  \varphi(t,r)=\begin{cases}
           1, & \mbox{if } \qquad\qquad\qquad\; r\le R(t) \\
           1-\frac{2}{\delta^2R(t)^2}(r-R(t))^2, & \mbox{if }\quad\qquad R(t)<r\le(1+\frac{\delta}{2})R(t) \\
           \frac{2}{\delta^2 R(t)^2}((1+\delta)R(t)-r)^2, & \mbox{if } (1+\frac{\delta}{2})R(t)<r\le(1+\delta)R(t)\\
            0, & \mbox{if } (1+\delta)R(t)<r\,.
         \end{cases}
\end{equation}
 for some radius $R(t)$ to be chosen later and some small constant $\delta>0$. In terms of this test function, the result reads as follows.

\begin{teo}\label{Thm shrinking radius}
  Assume that $L$ is $\lambda$-convex in $B_{(1+\delta)R_0}(0)$ with a minimum at 0 and $\lambda>0$. Let $\rho$ be a weak solution of \eqref{FP intro}
  with $0\le Q(x)\le \sigma I_{d\times d}$ for every $x\in B_{(1+\delta)R_0}(0)$. Then, for any time dependent radius $R(t)$ satisfying
  \begin{equation}\label{Shrinking radius}
    R'(t)\ge-\frac{\lambda}{2}R(t)\quad\forall t>0\,,
  \end{equation}
  it holds that for every $T>0$ and $\varphi$ as in \eqref{smooth characteristic}
  \begin{equation}\label{Concentration ODE}
    \int_{\mathbb{R}^d}\varphi(T,|x|)\rho(T,x)\dx\ge\int_{\mathbb{R}^d}\varphi(0,|x|)\rho_0(x)\dx -C \varepsilon^2 \int_{0}^{T}\frac{1}{R(t)^2}\dt
  \end{equation}
  with $C>0$ depending only on $\delta,\sigma$ and $d$.
\end{teo}

\begin{proof}
  We want to differentiate the quantity in \eqref{Concentration ODE}, therefore let us compute the derivatives of the test function $f$.
  \begin{equation*}
    \partial_t\varphi(t,r)=\begin{cases}
           \frac{4}{\delta^2R(t)^2}(r-R(t))r\frac{R'(t)}{R(t)}, & \mbox{if }\quad\qquad R(t)<r\le(1+\frac{\delta}{2})R(t) \\
           \frac{4}{\delta^2 R(t)^2}((1+\delta)R(t)-r)r\frac{R'(t)}{R(t)}, & \mbox{if } (1+\frac{\delta}{2})R(t)<r\le(1+\delta)R(t)\\
            0, & \mbox{otherwise}\,.
         \end{cases}
  \end{equation*}

   \begin{equation*}
    \partial_r\varphi(t,r)=\begin{cases}
           -\frac{4}{\delta^2R(t)^2}(r-R(t)), & \mbox{if }\quad\qquad R(t)<r\le(1+\frac{\delta}{2})R(t) \\
           -\frac{4}{\delta^2 R(t)^2}((1+\delta)R(t)-r), & \mbox{if } (1+\frac{\delta}{2})R(t)<r\le(1+\delta)R(t)\\
            0, & \mbox{otherwise}\,.
         \end{cases}
    \end{equation*}
    \begin{equation*}
    \partial_{rr}\varphi(t,r)=\begin{cases}
           -\frac{4}{\delta^2R(t)^2}, & \mbox{if }\quad\qquad R(t)<r\le(1+\frac{\delta}{2})R(t) \\
           \frac{4}{\delta^2 R(t)^2}, & \mbox{if } (1+\frac{\delta}{2})R(t)<r\le(1+\delta)R(t)\\
            0, & \mbox{otherwise}\,.
         \end{cases}
  \end{equation*}
  Let us recall the gradient and the Hessian matrix of radial functions with respect to $x$.
  \begin{align*}
    \nabla \varphi(t,|x|)&=\partial_r \varphi(t,|x|)\frac{x}{|x|}\,,\\
     D^2\varphi(t,|x|)&=\partial_{rr}\varphi(t,|x|)P_0 +\frac{\partial_{r}\varphi(t,|x|)}{|x|}(I_{d\times d}-P_0 )\,,
  \end{align*}
  with $P_0=\frac{x}{|x|}\otimes\frac{x}{|x|} $ being the projection matrix with respect to the vector $\frac{x}{|x|}$. Now, let us differentiate the desired quantity and use the equation of $\rho$:
  \begin{align*}
    \frac{\d}{\dt}\!\!\int_{\mathbb{R}^d}\!\!\!\!\varphi(t,|x|)\rho(t,x)\dx
    =&\int_{\mathbb{R}^d}\partial_t \varphi(t,|x|)\rho(t,x)\dx\\
    +&\varepsilon^2\int_{\mathbb{R}^d}\!\!\!\mbox{tr}\!\left(Q(x)D^2 \varphi(t,|x|)\right)\!\rho(t,x)\dx\!-\!\!\int_{\mathbb{R}^d}\!\!\!\nabla \varphi(t,|x|)\!\cdot\!\nabla L(x)\rho(t,x)\dx\\
  \end{align*}
  Note that this differentiation is formal, however we can make it rigourous by integrating the expression above in time and using Definition \ref{weak solution} for weak solutions of \eqref{FP intro}.

  Then, we have to use the previously calculated derivatives of $\varphi$. We will compensate the integral term of $\nabla L$ with the one with $\partial_t \varphi$ and the integral with $\varepsilon$  will be the error. In order to simplify the computation we split the integrals in two regions, depending on the time dependent radius $R=R(t)$:
  \begin{align*}
    &\int_{\mathbb{R}^d}\partial_t \varphi\,\rho\dx+\varepsilon^2\int_{\mathbb{R}^d}\mbox{tr}\left(Q\,D^2 \varphi\right)\rho\dx-\int_{\mathbb{R}^d}\nabla \varphi\cdot\nabla L\,\rho\dx\\
    =&\int_{R\le|x|\le(1+\frac{\delta}{2})R}\left[\partial_t \varphi+\varepsilon^2\mbox{tr}\left(Q\,D^2 \varphi\right)-\nabla \varphi\cdot\nabla L\right]\rho\dx\\
    &+\int_{(1+\frac{\delta}{2})R\le|x|\le(1+\delta)R}\left[\partial_t \varphi+\varepsilon^2\mbox{tr}\left(Q\,D^2 \varphi\right)-\nabla \varphi\cdot\nabla L\right]\rho\dx\\
    =&I+II
  \end{align*}
  Let us estimate the first term.
  \begin{align*}
    I =&\int_{R\le|x|\le(1+\frac{\delta}{2})R}\left[\partial_t \varphi-\partial_r \varphi\,\frac{x}{|x|}\cdot\nabla L+\varepsilon^2\mbox{tr}\left(Q\left(\partial_{rr} \varphi \,P_0+\frac{\partial_r \varphi}{|x|}(I_{d\times d}-P_0)\right)\right)\right]\rho\dx\\
    =&\frac{4}{\delta^2R^2}\int_{R\le|x|\le(1+\frac{\delta}{2})R}\left[(|x|-R)|x|\frac{R'}{R}+(|x|-R)\frac{x}{|x|}\cdot\nabla L\right]\rho\dx\\
    &-\frac{4\varepsilon^2}{\delta^2R^2}\int_{R\le|x|\le(1+\frac{\delta}{2})R}\left[\mbox{tr}\left(Q\,P_0\right) +\frac{(|x|-R)}{|x|}\mbox{tr}\left(Q\left(I_{d\times d}-P_0\right)\right)\right]\rho\dx\,.
  \end{align*}
  Using the $\lambda$-convexity of $L$ and the bounds for the traces of the product of nonnegative matrices  $0\le \mbox{tr}(Q P_0)\le \sigma d$ and $0\le \mbox{tr}(Q (I_{d\times d}-P_0))\le \sigma d(d-1)$, we obtain
  \begin{align*}
    I\ge&\frac{4}{\delta^2R^2}\int_{R\le|x|\le(1+\frac{\delta}{2})R}(|x|-R)|x|\left(\frac{R'}{R}+\frac{\lambda}{2}\right)\rho\dx\\
    &-\frac{4\varepsilon^2}{\delta^2R^2}\int_{R\le|x|\le(1+\frac{\delta}{2})R}\left[\sigma d+\left(1-\frac{R}{|x|}\right)\sigma d(d-1)\right]\rho\dx\,.
  \end{align*}
  Therefore, if the time dependent radius satisfies $R'(t)\ge-\frac{\lambda}{2}R(t)$, the first integral above is nonnegative. Moreover, since $\rho$ is a probability measure we obtain the following lower bound for the second integral
  \begin{equation}\label{Estimate first integral}
    I \ge-\frac{4\sigma d^2}{\delta^2}\frac{\varepsilon^2}{R^2(t)}\,.
  \end{equation}
  For the integral on the other region, II, we follow the same argument. First, we use the explicit expressions of the derivatives of $\varphi$.
  \begin{align*}
     II \!=\!&\!\int_{(1+\frac{\delta}{2})R\le|x|\le(1+\delta)R}\left[\partial_t \varphi-\partial_r \varphi\,\frac{x}{|x|}\cdot\nabla L+\varepsilon^2\mbox{tr}\left(Q\left(\partial_{rr} \varphi \,P_0+\frac{\partial_r \varphi}{|x|}(I_{d\times d}-P_0)\right)\right)\right]\rho\dx\\
    =&\frac{4}{\delta^2R^2}\int_{(1+\frac{\delta}{2})R\le|x|\le(1+\delta)R}\left[\bigg((1+\delta)R-|x|\bigg)|x|\frac{R'}{R}+\bigg((1+\delta)R-|x|\bigg)\frac{x}{|x|}\cdot\nabla L\right]\rho\dx\\
    &+\frac{4\varepsilon^2}{\delta^2R^2}\int_{(1+\frac{\delta}{2})R\le|x|\le(1+\delta)R}\left[\mbox{tr}\left(Q\,P_0\right) -\frac{(1+\delta)R-|x|}{|x|}\mbox{tr}\left(Q\left(I_{d\times d}-P_0\right)\right)\right]\rho\dx\,.
  \end{align*}
  Then, we use the $\lambda$-convexity of $L$ and $0\le \mbox{tr}(Q P_0)\le \sigma d$ and $0\le \mbox{tr}(Q (I_{d\times d}-P_0))\le \sigma d(d-1)$ to obtain that
  \begin{align*}
    II\ge&\frac{4}{\delta^2R^2}\int_{(1+\frac{\delta}{2})R\le|x|\le(1+\delta)R}\bigg((1+\delta)R-|x|\bigg)|x|\left(\frac{R'}{R}+\frac{\lambda}{2}\right)\rho\dx\\
    &-\frac{4\varepsilon^2}{\delta^2R^2}\int_{(1+\frac{\delta}{2})R\le|x|\le(1+\delta)R}\left(\frac{(1+\delta)R}{|x|}-1\right)\sigma d (d-1)\rho\dx\,.
  \end{align*}
Hence, the differential inequality for the time dependent radius $R'(t)\ge-\frac{\lambda}{2}R(t)$ implies that
\begin{align*}
  II\ge-\frac{4(1+\delta)\sigma d^2}{\delta^2(1+\frac{\delta}{2})}\frac{\varepsilon^2}{R^2(t)}\,.
\end{align*}
Finally, we conclude by combining both estimates of $I$ and $II$,
\begin{align*}
  \frac{\d}{\dt}\!\int_{\mathbb{R}^d}\!\!\varphi(t,|x|)\rho(t,x)\dx=I+II\ge -C\frac{\varepsilon^2}{R(t)^2}\,,
\end{align*}
with $C=\frac{4\sigma d^2}{\delta^2}(1+\frac{1+\delta}{1+\delta/2})$ and integrating the inequality in time on $[0,T]$.
\end{proof}

Note that the shrinking condition \eqref{Shrinking radius} of the radius $R(t)$ coincide with the concentration velocity of the flow map \eqref{flow map convergence} associated to the continuity equation \eqref{transport}. The main idea is that if the radius decreases slower than the flow map concentrates, the mass in the shrinking balls will be almost preserved, up to an error depending on $\varepsilon$.

Despite the error term not being sharp, it is impossible to prove the estimate \eqref{Concentration ODE} with $C=0$ due to the diffusive nature of the equation of $\rho$. Indeed, if we consider the case $Q(x)=I$ in $\mathbb{R}^2$ with a double well potential $L(x,y)=(x^2-1)^2+y^2$, it is easy to see that even if the  mass of the initial datum is concentrated in one of the local minima, half of the mass will reach the neighbourhood of the other local minimum. A simple proof can be done studying the local mass of the unique steady state $\rho_\infty(x,y)=c \exp(-\frac{L(x,y)}{\varepsilon^2})$.

Now, we use Theorem \ref{Thm shrinking radius} with a decreasing radius satisfying $R'(t)\ge-\frac{\lambda}{2}R(t)$ in order to obtain an explicit control of the error term. Thus, a natural choice $R(t)$ is the exponentially decreasing radius, which we have considered previously in the non diffusive case in Lemma \ref{transport shrinking radius} and satisfies condition \eqref{Shrinking radius} with equality.
\begin{cor}
  Under assumptions of Theorem \ref{Thm shrinking radius}, if we choose the radius to be
  \begin{equation*}
    R(t)=R_0\, e^{-\frac{\lambda}{2}t}\,,
  \end{equation*}
  then for every $T>0$ it holds that
  \begin{equation}\label{error for exponential radius}
    \int_{\mathbb{R}^d}\varphi(T,|x|)\rho(T,x)\dx\ge\int_{\mathbb{R}^d}\varphi(0,|x|)\rho_0(x)\dx-\varepsilon^2\frac{C}{R_0^2\lambda}\left(e^{\lambda T}-1\right)\,,
  \end{equation}
   with $C>0$ depending only on $\delta,\sigma$ and $d$.
\end{cor}

\begin{proof}
  Let us substitute the definition of the radius in the error term of \eqref{Concentration ODE}.
\[
    -C\int_{0}^{T}\frac{\varepsilon^2}{R(t)^2}\dt=-\varepsilon^2\frac{C}{R_0^2}\int_{0}^{T}e^{\lambda t}\dt
    =-\varepsilon^2\frac{C}{R_0^2\lambda}\left(e^{\lambda T}-1\right)\,. \qedhere
\]
  \end{proof}

  The above result gives us a lower bound for the mass concentrated around a ball centered in the local minimum of $L$.  We are now in the position to prove Theorem \ref{Thm Concentration}.

  \medskip

\noindent{\it Proof of Theorem \ref{Thm Concentration}. }Letting $T=T(\varepsilon, C)$ in \eqref{error for exponential radius}, with $T(\varepsilon, C) = \frac{1}{\lambda}\log\left(1 +\frac{R_0^2 \lambda}{C \varepsilon^{2 \alpha}}\right)$, gives that for every $0<t<T_\varepsilon$  we have
 \begin{align*}
   \int_{\mathbb{R}^d}\varphi(t,|x|)\rho(t,x)\dx
   &\ge \int_{\mathbb{R}^d}\varphi(0,|x|)\rho_0(x)\dx -\varepsilon^2\frac{C}{R_0^2\lambda}\left(e^{\lambda T_\varepsilon}-1\right)\\
   &=\int_{\mathbb{R}^d}\varphi(0,|x|)\rho_0(x)\dx -\varepsilon^{2-2\alpha}.
 \end{align*}
Finally, we notice that we can replace $t=0$ by $t=t_0$ by the time shift invariance of the equation. \hfill \qed

\section{Analysis in the diffusion regime}\label{sec Diffusion Regime}

After a first concentration phenomena near local minima of the loss function $L$, the diffusion will allow the parameters of the optimization to jump from one local minimum of $L$ to another one. This property of the diffusion is the main advantage of using the SGD as a learning process with respect to the classical Gradient Descent, since it allows to avoid local minima of $L$ which could be attractors at first steps. Thus, the main question that we want to address in this section is  {what we formulated as Q2 in the introduction: \emph{How long does SGD take to escape from a local minimum?}}

\

\noindent To this end, we focus on how the diffusion affects the dynamics of the learning process. In particular, we want to estimate the mean exit time from a neighborhood of a local minimum of the loss function. Lower bounds for the mean exit time can be obtained without further assumption. On the other hand, some non-degeneracy condition on the diffusion is needed in order to ensure that the mean exit time is finite. Indeed, as a counterexample, consider the case with no diffusion at all, presented in \eqref{transport}, which leads to the deterministic flow map
 \begin{equation*}
   \begin{cases}
     \frac{\d}{\dt}\Phi_t(x)=-\nabla L(\Phi_t(x))\\
     \Phi_0(x)=x\,.
   \end{cases}
 \end{equation*}
In this case, all the mass splits and concentrates in the local minima of $L$, with no jump from one minimum to another one, i.e. the mean exit time is infinite.

Recall that the stochastic process describing the dynamics of the parameters when learning with SGD satisfies \eqref{SDE intro}, that is,
\begin{equation*}
  \begin{cases}
    \d X_t&=-\nabla L(X_t)\dt+\sqrt{2\varepsilon^2 Q(X_t)}\d W_t\\
    X_0&=x\,,
  \end{cases}
\end{equation*}
with the degenerate but nonnegative matrix $Q(x)\ge 0$, i.e. $v^TQ(x)v\ge 0$ for every $v\in\mathbb{R}^d$ and every $x\in\mathbb{R}^d$. Due to the degeneracy of $Q$, obtaining estimates on the Mean Exit Time (MET) from a basin of attraction of $L$ is not obvious. Recall also that the MET $u(x)=\mathbb{E}[\tau_\Omega^x]$ from an open set $\Omega \subset \mathbb{\R}^d$ given an initial point $x \in \Omega$
satisfies \eqref{Intro Mean Exit Time}, that is
\[
  \begin{cases}
    -\mathcal{A}u(x) = 1, & \mbox{in } \Omega  \\
    u(x)=0, & \mbox{on }\partial \Omega
  \end{cases}
\]
where $$\mathcal{A}u(x) = \varepsilon^2\mbox{tr}\left(Q(x)D^2 u(x)\right)-\nabla L(x)\cdot\nabla u(x).$$

This family of equations contains, as special cases, the non diffusive scenario ($Q(x)\equiv0$, i.e. pure drift) and the non-degenerate diffusive scenario ($Q(x)\ge \delta I_{d\times d}$).

In the non-diffusive case, the mean exit time is $u(x)=+\infty$ for all $x\in\Omega$. Indeed, the flow map $X_t$ is deterministic and describes the concentration of the parameters around the local minima of the loss function $L$. Hence, there is no possibility for $X_t$ to escape from a neighbourhood of a local minimum of $L$ in any time.

In the non-degenerate  diffusive case, by Kramers' Law, $u(x)<+\infty$ for every $x\in\Omega$. Indeed, even if the drift term $\nabla L$ pushes $X_t$ into the local minimum, the diffusion in every direction ensures the boundedness, see Appendix \ref{sec Kramers Law}.

In this section, we provide lower and upper bounds for the viscosity solution $u$ to \eqref{Intro Mean Exit Time} under mild conditions on $Q$ and $L$, allowing for degeneracies.
 {Our strategy is based a variation of the comparison principle given in \cite{CrIsLi}, for viscosity solutions of \eqref{Intro Mean Exit Time}. In order to prove it, we adapt the arguments of \cite[Theorem 3.3]{CrIsLi} to our setting, following the indications in \cite[Section 5C]{CrIsLi}. Before stating the result, we rewrite \eqref{Intro Mean Exit Time} in the notation of Definition \ref{def:viscosity}, by setting
\[
F\left(x,\nabla u(x),D^2u(x)\right) = -\mathcal{A}u(x) - 1 = -\varepsilon^2\mbox{tr}\left(Q(x)D^2u(x)\right)+\nabla L(x)\cdot\nabla u(x) - 1\,.
\]
The comparison principle that we need is the following.
\begin{teo}[Degenerate Comparison Principle]\label{th:DCP}
Let $\Omega\subset\mathbb{R}^d$ be a bounded domain, let $L_i\in C^{1,1}(\overline{\Omega})$ for every $i\ge1$, and let $u\in C(\overline{\Omega})$ be a viscosity solution of
\begin{equation}\label{eq}
F\left(x,\nabla u(x),D^2u(x)\right) = 0.
\end{equation}
\begin{enumerate}[label=\roman*)]
\item Assume that there exists $C > 0$ such that, for every $\delta>0$, there exists $\psi_\delta\in C^2(\overline{\Omega})$ satisfying $|\psi_\delta(x)|\le C\delta$ for every $x\in\Omega$. Let $w\in C(\overline{\Omega})$ such that, for all $\delta > 0$, $w_\delta=w+\psi_\delta$ is a viscosity solution of
\begin{equation}\label{eqI}
\left\{
\begin{array}{rcll}
    F\left(x,\nabla w_\delta(x),D^2w_\delta(x)\right) & \le & -\delta & \ , \quad x \in \Omega\\
    w_\delta & \le & 0 & \ , \quad x \in \partial\Omega\,.
\end{array}
\right.
\end{equation}
Then,
\begin{equation}
    w(x)\le u(x) \qquad\forall x\in\Omega\;.
\end{equation}
\item Assume that there exists $C > 0$ such that, for every $\delta>0$, there exists $\psi_\delta\in C^2(\overline{\Omega})$ satisfying $|\psi_\delta(x)|\le C\delta$ for every $x\in\Omega$. Let $w\in C(\overline{\Omega})$ such that, for all $\delta > 0$, $w_\delta=w+\psi_\delta$ is a viscosity solution of
\begin{equation}\label{eqII}
\left\{
\begin{array}{rcll}
    F\left(x,\nabla w_\delta(x),D^2w_\delta(x)\right) & \ge & \delta & \ , \quad x \in \Omega\\
    w_\delta & \ge & 0 & \ , \quad x \in \partial\Omega\,.
\end{array}
\right.
\end{equation}
Then,
\begin{equation}\label{comparison}
w(x)\ge u(x) \qquad\forall x\in\Omega\;.
\end{equation}
\end{enumerate}
\end{teo}
\begin{proof}
As the arguments in both cases are analogous, we only present the proof of {\it ii)}. Observe that the assumption $L_i\in C^{1,1}(\bar\Omega)$ for every $i\ge 1$ guarantees the continuity of $F$, so that we can refer to the notion of viscosity solutions of Definition \ref{def:viscosity}.\\
The proof is immediate once we have that, for every $\delta>0$,
\begin{equation}\label{comparison psi}
u(x)\le w_\delta(x)=w(x)+\psi_ \delta(x)   \qquad\forall x\in\Omega\;.
\end{equation}
Indeed, since $|\psi_\delta(x)|\le C\delta$ for every $x\in\Omega$ with $C>0$ independent of $\delta$ and $x$, we obtain \eqref{comparison} by simply taking the limit as $\delta \to 0$ in \eqref{comparison psi}. We will prove \eqref{comparison psi} following \cite[Section 3]{CrIsLi}.
Note first that, since $\overline{\Omega}$ is compact and $u,w_\delta\in C(\overline{\Omega})$, it holds that for any $\alpha>0$
\[
u(z)-w_\delta(z) \le M_\alpha < +\infty\,,
\]
where
$$
M_\alpha:=\sup_{(x,y) \in \Omega\times\Omega}\left(u(x)-w_\delta(y)-\frac{\alpha}{2}|x-y|^2\right)=u(x_\alpha)-w_\delta(y_\alpha)-\frac{\alpha}{2}|x_\alpha-y_\alpha|^2\,,
$$
for some $(x_\alpha,y_\alpha)\in\Omega\times\Omega$, since $u\le w_\delta$ on $\partial\Omega$. Assume now, by contradiction, that $u(z)>w_\delta(z)$ for some $z\in\Omega$, so that
\begin{equation}\label{Contradiction Hypothesis}
0 < u(z)-w_\delta(z) \le M_\alpha < +\infty\,.
\end{equation}
We will use equations \eqref{eq} and \eqref{eqII} to estimate $M_\alpha$ and contradict \eqref{Contradiction Hypothesis}. By applying Ishii's Lemma \cite[Theorem 3.2]{CrIsLi} to the function $u(x)-w_\delta(y)-\frac{\alpha}{2}|x-y|^2$ at the local maximum $(x_\alpha,y_\alpha)$, we have that there exist $X,Y\in \mathbb{R}^{d\times d}$ symmetric matrices such that
\begin{equation}\label{matrix ineq}
\begin{pmatrix}
X&0\\
0&-Y
\end{pmatrix}
\le 3\alpha
\begin{pmatrix}
I&-I\\
-I&I
\end{pmatrix}
\end{equation}
and
\[
(\alpha(x_\alpha-y_\alpha),X) \in J^{2,+}_\mathcal{O}u(x_\alpha) \ , \quad (\alpha(x_\alpha-y_\alpha),Y) \in J^{2,-}_\mathcal{O}w_\delta(y_\alpha).
\]
This implies that, by the hypotheses on $u$ and $w_\delta$ and the definition of viscosity solution
\begin{equation*}
F(x_\alpha,\alpha(x_\alpha-y_\alpha),X)=0 \quad , \qquad  F(y_\alpha,\alpha(x_\alpha-y_\alpha),Y)\ge\delta\,.
\end{equation*}
Therefore,
    \begin{equation}\label{Contradiction ineq}
    \begin{split}
        0<\delta&\le F(y_\alpha,\alpha(x_\alpha-y_\alpha),Y)-F(x_\alpha,\alpha(x_\alpha-y_\alpha),X)\\
        &=\varepsilon^2\mbox{tr}\left(Q(x_\alpha)X-Q(y_\alpha)Y\right)+\alpha\left(\nabla L(y_\alpha)-\nabla L(x_\alpha)\right)\cdot(x_\alpha-y_\alpha)\,.
    \end{split}
    \end{equation}
A contradiction will be now obtained by proving that the right hand side of \eqref{Contradiction ineq} tends to 0 as $\alpha\rightarrow \infty$. Let us analyze each of the terms of the right hand side of inequality \eqref{Contradiction ineq}:
\begin{enumerate}[label=\alph*)]
\item Diffusion term: this is done as in \cite[Example 3.6]{CrIsLi}. First, let us recall from \eqref{SDE intro} that \begin{equation*}
Q(x)=T(x)^*T(x)\qquad\mbox{with}\qquad T(x)=\begin{pmatrix}
\frac{1}{N}\nabla L_1(x)-\nabla L(x)\\
\vdots\\
\frac{1}{N}\nabla L_N(x)-\nabla L(x)
\end{pmatrix}
\end{equation*}
and observe that the following $2d \times 2d$ symmetric matrix is nonnegative:
\begin{equation*}
S(x_\alpha,y_\alpha) =
\begin{pmatrix}
T^*(x_\alpha) T(x_\alpha)& T^*(y_\alpha) T(x_\alpha)\\
T^*(x_\alpha) T(y_\alpha)& T^*(y_\alpha) T(y_\alpha)\\
\end{pmatrix}.
\end{equation*}
Thus, multiplying by $S(x_\alpha,y_\alpha)$ the inequality \eqref{matrix ineq}, and taking traces, we have
\begin{align*}
\mbox{tr} & \bigg( T^*(x_\alpha) T(x_\alpha)X- T^*(y_\alpha) T(y_\alpha)Y\bigg) = \mbox{tr}\left(S(x_\alpha,y_\alpha)\binom{X \ \phantom{-}0 \ }{\ 0 \ -Y}\right)\\
& \leq 3 \alpha \, \mbox{tr}\left(S(x_\alpha,y_\alpha)\binom{\phantom{-}I \ -I}{-I \ \phantom{-}I}\right) = 3 \alpha \, \mbox{tr}\bigg((T^*(x_\alpha)-T^*(y_\alpha))(T(x_\alpha)-T(y_\alpha))\bigg).
\end{align*}
Here, by the definition of $T$ and $L$, we have
\[
T(x_\alpha)-T(y_\alpha) = \begin{pmatrix}
\frac{1}{N}\sum_{j \neq 1}(\nabla L_j(y_\alpha) - \nabla L_j(x_\alpha))\\
\vdots\\
\frac{1}{N}\sum_{j \neq N}(\nabla L_j(y_\alpha) - \nabla L_j(x_\alpha))
\end{pmatrix}
\]
so
\[
\mbox{tr}\bigg((T(x_\alpha)-T(y_\alpha))^*(T(x_\alpha)-T(y_\alpha))\bigg) = \frac{1}{N^2}\sum_{i = 1}^N \bigg| \sum_{j \neq i}\Big(\nabla L_j(y_\alpha) - \nabla L_j(x_\alpha)\Big)\bigg|^2
\]
Hence, using the fact that $\nabla L_i$ are Lipschitz in $\bar\Omega$ for every $i\ge 1$, and denoting by $c_L$ the maximum of their Lipschitz constant, we deduce
\begin{align*}
\varepsilon^2 & \mbox{tr}\left(Q(x_\alpha)X-Q(y_\alpha)Y\right) = \varepsilon^2\;\mbox{tr}\bigg( T^*(x_\alpha) T(x_\alpha)X- T^*(y_\alpha) T(y_\alpha)Y\bigg)\\
& \leq \frac{3 \alpha \varepsilon^2}{N^2} \, \sum_{i = 1}^N \bigg| \sum_{j \neq i}\Big(\nabla L_j(y_\alpha) - \nabla L_j(x_\alpha)\Big)\bigg|^2\\
& \le 3 \varepsilon^2 N c_L^2 \alpha |x_\alpha - y_\alpha|^2.
\end{align*}
\item Drift term: again since $\nabla L_i$ are Lipschitz in $\bar\Omega$ for every $i\ge 1$, then
\begin{equation*}
\alpha\left(\nabla L(y_\alpha)-\nabla L(x_\alpha)\right)\cdot(x_\alpha-y_\alpha)\le c_L\; \alpha |x_\alpha-y_\alpha|^2\,.
\end{equation*}
\end{enumerate}
Now, we are in a position to obtain the contradiction from \eqref{Contradiction ineq} combining a) and b):
\begin{equation}
\begin{split}
0<\delta
&\le\varepsilon^2\mbox{tr}\left(Q(x_\alpha)X-Q(y_\alpha)Y\right)+\alpha\left(\nabla L(y_\alpha)-\nabla L(x_\alpha)\right)\cdot(x_\alpha-y_\alpha)\\
&\le \left(3\varepsilon^2 N c_L^2 + c_L\right)\;\alpha|x_\alpha-y_\alpha|^2\rightarrow 0\,,
\end{split}
\end{equation}
where the limit as $\alpha\rightarrow\infty$ is due to Lemma 3.1 of \cite{CrIsLi}.
\end{proof}
}

\begin{teo}\label{lower bound MET}
 {Let $R > 0$. Assume $L_i\in C^{1,1}(\overline{B_R(0)})$ for every $i\ge1$, let $0\le Q(x)\le\sigma I_{d\times d}$ for all $x \in B_R(0)$, and let $L$ be $\lambda$-convex in $B_{R}(0)$ with a minimum at $0$.} Let $x\in B_r(0)$ with $0< r \le R$ and let the Mean Exit Time $\mathbb{E}\big[\tau_{B_R(0)}^x\big]$ be a viscosity solution of \eqref{Intro Mean Exit Time}.  {with $\lambda>0$}.
Then,  {for every $\delta > 0$},
\[
w_{ {\delta}}(x) =  {(1-\delta)}\frac{R^2 - |x|^2}{2\varepsilon^2 d \sigma}
\]
is a  {classical solution (and hence a viscosity solution) of \eqref{eqI}.}
\end{teo}
\begin{proof}
Note that
\begin{align*}
\nabla w_{ {\delta}}(x) = - \frac{ {(1-\delta)}}{\varepsilon^2\sigma d}x\qquad\mbox{and}\qquad D^2w_{ {\delta}}(x)=-\frac{ {(1-\delta)}}{\varepsilon^2\sigma d}I_{d\times d}\,.
\end{align*}
Therefore,  {for every $x \in B_R(0)$}
\begin{align*}
 {F(x,\nabla w_\delta(x),D^2 w_\delta(x))} & =-\varepsilon^2\mbox{tr}\left(Q(x)D^2w_{ {\delta}}\right)+\nabla L(x)\cdot\nabla w_{ {\delta}}   {\, - 1}\\
& =  {(1-\delta)}\left(\frac{\mbox{tr}\left(Q(x)\right)}{\sigma d}-\frac{1}{\varepsilon^2\sigma d}x\cdot\nabla L(x)\right)  {- 1}\\
& \le  {(1-\delta)}\left(\frac{\mbox{tr}\left(Q(x)\right)}{\sigma d}-\frac{\lambda}{2\varepsilon^2\sigma d}|x|^2\right)  {- 1 \leq -\delta}\,,
\end{align*}
where we have used the $\lambda$-convexity of $L$ in the first inequality and the bound tr$\left(Q\right)\le\sigma d$ in the second inequality. Since $w_{ {\delta}}(x)=0$ if $|x|=R$, we conclude that $w_{ {\delta}}$ is a  {classical solution of \eqref{eqI}}.
\end{proof}

In order to obtain an upper bound for the Mean Exit Time, some non-degeneracy in the diffusion is required. Indeed, it suffices to have a positive lower bound of $Q$ in one direction and some smoothness condition in $L$.
\begin{teo}\label{upper bound MET}
Let $\Omega$ be a bounded domain, let $L_i\in C^{1,1}(\overline{\Omega})$ for every $i\ge1$, and let $Q(x) \geq 0$ for every $x\in \Omega$. Assume that there exist $\beta > 0$ and a vector $v\in\mathbb{S}^{d-1}$ such that
$$
v^TQ(x)v\ge\beta \quad \forall \ x \in \Omega.
$$
Let $\Lambda, c > 0$ be such that $\left(\nabla L(x)\cdot v\right)\left(v\cdot x + c\right)\le\Lambda (v \cdot x + c)^2$ for all $x\in \Omega$, and define $R_v :=\max\lbrace |v\cdot x + c| : x\in\Omega\rbrace$. Then, for every $\delta>0$,
$$
w_{ {\delta}}(x)=\frac{ {(1+\delta)}}{\Lambda}\left(e^{\frac{\Lambda R_v^2}{2\beta\varepsilon^2}}-e^{\frac{\Lambda (v\cdot x + c)^2}{2\beta\varepsilon^2}}\right)
$$
is a nonnegative classical solution (and hence a viscosity solution) of \eqref{eqII} in $\Omega$.
\end{teo}
\begin{proof}
Note that, since $\nabla L$ is continuous and $\Omega$ is bounded, the constants $\Lambda, c, R_v$ exist and are finite. Now, if we compute the derivatives of $\omega_\delta$, we get
\begin{align*}
\nabla w_\delta(x)&=-\frac{(1+\delta)}{\beta\varepsilon^2} e^{\frac{\Lambda (v\cdot x + c)^2}{2\beta\varepsilon^2}}(v\cdot x + c)v\\
D^2w_{\delta}(x)&=-\frac{(1+\delta)}{\beta\varepsilon^2} e^{\frac{\Lambda (v\cdot x + c)^2}{2\beta\varepsilon^2}}\left(\frac{\Lambda}{\beta\varepsilon^2}(v\cdot x + c)^2+1\right)v\otimes v\,.
\end{align*}
Therefore, using the assumptions on $Q$ and on $L$ along with $\mbox{tr}(Qv\otimes v)=v^TQv$, we obtain
\begin{align*}
& F(x,\nabla w_\delta,D^2w_\delta) = -\varepsilon^2\mbox{tr}\left(Q(x)D^2w_{\delta}\right)+\nabla L(x)\cdot\nabla w_{\delta}-1\\
& = (1+\delta)e^{\frac{\Lambda (v\cdot x + c)^2}{2\beta\varepsilon^2}} \left(\frac{\mbox{tr}(Q(x)v\otimes v)}{\beta} \Big(\frac{\Lambda}{\beta\varepsilon^2}(v\cdot x + c)^2+1\Big)-\frac{1}{\beta\varepsilon^2} (\nabla L(x)\cdot v)(v\cdot x + c) \right) -1\\
& \ge (1+\delta) \left(\frac{\Lambda}{\beta\varepsilon^2}(v\cdot x)^2+1-\frac{\Lambda}{\beta\varepsilon^2}  (v\cdot x)^2 \right) -1 = \delta\,,
\end{align*}
since $e^{\frac{\Lambda (v\cdot x + c)^2}{2\beta\varepsilon^2}}\ge 1$. Moreover, by definition of $R_v$, for $x \in \Omega$ we have that $(v\cdot x + c)^2\le R_v^2$ and hence $w_\delta(x)\ge 0$.
\end{proof}
\begin{rem}[The case of unbounded domains]\label{rem:unboundedcomparison}
A close inspection of the proof of the previous theorem reveals that the result is still valid in an unbounded domain $\Omega$, under the assumption of the existence of a nondegenerate direction (characterized by the vector $v$ and the parameter $\beta$), the existence of $\Lambda, c > 0$ such that $\left(\nabla L(x)\cdot v\right)\left(v\cdot x + c\right)\le\Lambda (v \cdot x + c)^2$ for all $x\in \Omega$, and the condition that $R_v :=\max\lbrace |v\cdot x + c| : x\in\Omega\rbrace < \infty$, provided the validity of the comparison principle on $\Omega$. We are not aware of a proof of the degenerate comparison principle Theorem \ref{th:DCP} in unbounded domains.
\end{rem}

Thanks to the semigroup property of the Markovian process described by \eqref{SDE intro}, we can use the subsolution and supersolution of Theorems \ref{lower bound MET} and \ref{upper bound MET} starting at any time $t_0>0$. This allows us to estimate the time that the learning process will spend after the concentration phenomena occurred in the drift regime.

We are now in the position to prove Theorem \ref{Thm Lower MET} and \ref{Thm upper bound MET}. Note that, in order to make use of the classical comparison principle \cite{CrIsLi} in the next proofs, we are requiring that the mean exit time is a viscosity solution of \eqref{Intro Mean Exit Time}.

\vspace{2mm}
\noindent{\it Proof of Theorem \ref{Thm Lower MET}.}
 {By the degenerate} comparison principle  {of Theorem \ref{th:DCP}, Part i)}, we obtain that
$$u(x)= \mathbb{E}\left[\tau_{B_R(0)}^x\right]\ge w(x)=\frac{R^2 -|x|^2}{2\varepsilon^2\sigma d} \,,$$
since $u$ is a solution of \eqref{Intro Mean Exit Time} by construction and $w$ is a subsolution of \eqref{Intro Mean Exit Time} by Theorem \ref{lower bound MET}. Note that the Markovian property of the stochastic process \eqref{SDE intro} allows the time shifting of the initial datum. We conclude recalling that $|x|\le r$.
    \hfill \qed

By a similar argument we can obtain an upper bound for the mean exit time.

\vspace{2mm}
\noindent
 {
{\it Proof of Theorem \ref{Thm upper bound MET}}
Since $u$ is a solution of \eqref{Intro Mean Exit Time} by construction and $w$ is a supersolution of \eqref{Intro Mean Exit Time} by Theorem \ref{upper bound MET}, by the degenerate comparison principle of Theorem \ref{th:DCP}, Part ii), we obtain that
$$u(x)= \mathbb{E}\left[\tau_{\Omega}^x\right]\le w(x)=\frac{1}{\Lambda}\left(e^{\frac{\Lambda R_v^2}{2\beta\varepsilon^2}}-e^{\frac{\Lambda (v\cdot x)^2}{2\beta\varepsilon^2}}\right).$$
To complete the proof, it only remains to estimate the constants $c, \Lambda, R_v$. We recall that $\Omega \subseteq B(0,R_\Omega)$, and that $|v| = 1$. Let $c = R_\Omega + 1$. Then
\[
\frac{|\nabla L(x)\cdot v|}{|v\cdot x + c|} \leq \frac{M}{\min_{x \in \Omega}|v\cdot x + c|} \leq M
\]
where the last inequality follows by the fact that $|v\cdot x + c| \geq |c| - |x| \geq 1$. Thus, setting $\Lambda = M$, and noting that $|v\cdot x + c| \leq 2(R_\Omega + 1)$, we get
\[
\hspace{140pt} \mathbb{E}\left[\tau_{\Omega}^x\right] \le \frac{1}{M}\left(e^{\frac{2M(R_\Omega + 1)^2}{\beta\varepsilon^2}} - 1\right). \hspace{140pt} \qed
\]
}

\section{Asymptotic behaviour of the SGD}\label{sec Asymptotic}

In this section, we analyze the asymptotic behaviour of the distribution of parameters optimized by SGD. The main motivation is neural network training: starting from an initial parameter configuration and a dataset, we aim to determine the final parameter probability distribution after prolonged SGD implementation.  We will study this problem within the framework of the continuous stochastic process that approximates the SGD and its associated probability measure. Here we will address  {question Q3 from the introduction: \emph{Does the distribution of neural network parameters converge using SGD?}}

\

To answer this question is delicate, and we propose two different methods: duality and entropy methods. Both methods provide partial results about the convergence of the transition probability of the stochastic approximation to an invariant probability measure.

The duality method for a wide class of nondegenerate Fokker-Planck equations was introduced by Porretta in 2024 \cite{Por}. The first step consists of proving the time decay of oscillations of the solutions of the dual equation, often referred to as Ornstein-Uhlenbeck or backward Kolmogorov. As a second step, this decay can be translated into a decay of the moments of the probability measure that solves the original Fokker-Planck equation. Finally, the third step consists of showing that the convergence of the moments implies the existence of an invariant probability measure and the convergence to it. The main drawback of this method is that, to the best of our knowledge, it only works with nondegenerate diffusion, i.e. $0<\delta I\le Q\le \sigma I$ with $\sigma\geq\delta>0$. In order to adapt this method to the present setting, we propose a variant of the SGD adding some random noise at each step of the learning process.

A nowadays more standard approach to study the asymptotic behavior of Fokker-Planck equations is the entropy method \cite{Bakry-Emery,Bakry-Ledoux}, also known as the Bakry-Emery method. The relative entropy - a special Lyapunov function - quantifies the distance between the probability at time $t$ and the invariant probability measure. The first stability property, i.e. convergence to an invariant probability measure, can be proven by showing that the relative entropy decays to 0 as time goes to infinity. A more delicate analysis is then required to quantify the possible decay rates of the entropy. In general, even for the standard Fokker-Planck equation ($Q = I$ and $L(x) = |x|^2$), exponential decay holds only under additional assumptions, typically the boundedness of the first and second moment for the initial data (see e.g. \cite[Section 4.1]{Vaz18}). A quantitative exponential decay of the relative entropy can be proven by means of a weighted Poincaré type inequality which in turn is equivalent to the entropy-entropy production inequality.

In the present setting, the exponential decay follows by establishing a suitable Poincar\'e inequality obtained by using the results of \cite{Arnold2014} and  \cite{Menz}. The main limitations of this approach are that either the diffusion matrix must be non-degenerate with $Q\equiv \sigma I$ or, when $Q$ is allowed to be degenerate, the loss function $L$ must be quadratic.

\subsection{Duality method for Noisy SGD}\label{sec Duality method}

Recall that, for a data set $\lbrace (x_i,y_i) \rbrace_{i=1}^N$ and a total loss function
$$L(\theta)=\frac{1}{N}\sum_{i=1}^{N}L_i(\theta)\,,$$
where each $L_i(\theta)$ for $i=1,\dots,N$ is the partial loss function associated to the datum $(x_i,y_i)$, the SGD with constant learning rate $\eta>0$ and batch size $|B_n|=\b>0$ reads as in \eqref{SGD Main}
\[
\theta_{n+1}=\theta_n-\frac{\eta}{\b}\sum_{b_i\in B_n}\nabla L_{b_i}(\theta_n)\qquad\forall n\ge 0\,,
\]
with $B_n=\lbrace b_i\rbrace_{i=1}^{\b}\subset\lbrace 1,2,\dots,N\rbrace$ being a batch of indexes chosen with uniform distribution at each step $n$. Its SDE counterpart, from \cite{LiTaiWeinan2019}, as defined by \eqref{SDE intro}, reads
\begin{align*}
\d X_t=-\nabla L(X_t)\dt +\sqrt{\frac{\eta}{\b}}\Sigma(X_t)\d W_t\,,
\end{align*}
where the noise is always degenerate, because rank$(\Sigma(x)) \le N -1$ (and, in overparameterized scenarios, the degeneracy is more severe).

In order to obtain a nondegenerate SDE, we propose another learning iteration with a suitable extra noise $\lbrace Z_n\rbrace_{n\ge 0}$, which we have called Noisy Stochastic Gradient Descent in \eqref{NSGD}. This is
\[
\theta_{n+1}=\theta_n-\frac{\eta}{\b}\sum_{b_i\in B_n}\nabla L_{b_i}(\theta_n)+\eta Z_n.
\]
In what follows, our choice for $\lbrace Z_n\rbrace_{n\ge 0}$ is that of a family of i.i.d. Gaussian variables
$$
Z_n\sim\mathcal{N}(0, \delta I_{d\times d})\qquad\forall n\ge 0 \, .
$$
This extra noise does not depend on the characteristics of the learning process, and a priori may provide an invariant probability measure that is independent of the data set and the loss function. However, more sophisticated choices of noise could be considered, based on the loss function. We have chosen the simplest noise that allows to apply the results of \cite{Por}, in the form of Theorem \ref{Theorem oscillation decay} below, indeed, our choice of noise leads us to the following nondegenerate SDE.
\begin{lem}
{The process \eqref{NSGD} is the Euler-Maruyama approximation of}
\begin{align}\label{Q lambda SDE}
\d X_t=-\nabla L(X_t)\dt +\sqrt{\eta\, Q_\delta(X_t)}\d W_t\,,
\end{align}
with $Q_\delta(x):= \b^{-1} Q(x)+\delta I_{d\times d}$.
\end{lem}
\begin{proof}
  The proof follows by a slight modification of the one in \cite[Theorem 1]{LiTaiWeinan2017}. See Appendix \ref{app: NSGD} for the details.
\end{proof}
By standard theory, we can associate to \eqref{Q lambda SDE} the following Ornstein-Uhlenbeck type equation for $u(t,x) = \mathbb{E}[X_t]$
\begin{equation}\label{O-U}
\begin{cases}
\partial_t u=\frac{\eta}{2}\mbox{tr}(Q_\delta(x)D^2u)-\nabla L(x)\cdot\nabla u=:-\L_\delta u\\
u(0,x)=u_0(x)\,,
\end{cases}
\end{equation}
Analogously, the probability density function of $X_t$ satisfies the $L^2(\R^d)$-dual equation of \eqref{O-U}, that is, the Fokker-Planck equation
\begin{equation}\label{F-P}
\begin{cases}
\partial_t \rho =\nabla\cdot\left(\frac{\eta}{2}\nabla\cdot(\rho\, Q_\delta(x))+\rho\nabla L(x)\right)=:-\L^*_\delta \rho\\
\rho(0,x)=\rho_0(x)\,,
\end{cases}
\end{equation}
where $\nabla\cdot(\rho \,Q_\delta(x))$ is the divergence taken column-wise.

The goal is to prove convergence results for the distribution of the parameters $\theta$ in \eqref{NSGD}. In order to do that, we prove that there exists a stationary state of equation \eqref{F-P} and the solution $\rho$ converges to it as time goes to infinity. We are in the position to use the theory of \cite{Por} and establish the asymptotic behavior of the above non-degenerate Ornstein-Uhlenbeck and Fokker-Planck type equations, by exploiting the duality between \eqref{O-U} and \eqref{F-P}, and proving that the weighted oscillation of $u$ implies convergence of the moments of $m$.
The weighted oscillation that we will consider is the following
$$
[u]_{\langle\cdot\rangle^k} = \sup\limits_{x,y\in\R^d}\frac{|u(x)-u(y)|}{\langle x\rangle^k+\langle y\rangle^k}\,,
$$
with $\langle x\rangle^k=(1+|x|^2)^{k/2}$. This turns out to be a Lyapunov function of $\L_\delta$ (see \cite{Por}).

\begin{teo}\label{Theorem oscillation decay}{\rm \cite[ {Theorem 4.2}]{Por}}
Assume that the diffusion matrix $Q_\delta(x)=\b^{-1}Q(x)+\delta I_{d\times d}$ with $\delta>0$ satisfies that there exist $\sigma_0,\sigma_1>0$ such that
\begin{equation}\label{Assumption on Sigma}
	\|\sqrt{Q_\delta(\cdot)}\|_\infty\le\sigma_0 \ , \textnormal{and } \|\sqrt{Q_\delta(x)}-\sqrt{Q_\delta(y)}\|\le \sigma_1|x-y| \qquad\forall x,y\in \R^d\,.
\end{equation}
	Assume that the drift term $\nabla L(x)$ satisfies that there exist $\alpha,R>0$ and $\gamma\ge 2$ such that
	\begin{equation}\label{confining potential1}
	\nabla L(x)\cdot x\ge\alpha|x|^\gamma\qquad\forall x\in \R^d\quad\mbox{with}\quad |x|\ge R,
	\end{equation}
    and	 there exists $c_0>0$ such that
	\begin{equation}\label{Assumption on L}
	(\nabla L(x)-\nabla L(y))\cdot(x-y)\ge-c_0|x-y|\qquad\quad\forall x,y\in\R^d\,.
	\end{equation}
	If $\rho\in C([0,T):\mathcal{M}_k(\mathbb{R}^d))$ is a  {weak} solution of \eqref{F-P}, then
	\begin{equation}\label{oscillation decay.2}
		\|\rho(t)\|_{\mathcal{M}_k}\le K_\delta\, e^{-\omega t}\|\rho_0\|_{\mathcal{M}_k}\qquad\quad\forall t\in(0,T)\,,
	\end{equation}
	provided that $\rho_0\in\mathcal{M}_k(\mathbb{R}^d)$ and $\int_{\mathbb{R}^d}\d\rho_0(x)=0$.
\end{teo}

The above theorem is written for zero mean valued solutions. As shown in \cite{Por}, it can be easily extended to the case of non-zero mean valued solutions  with finite $k$-moments: the linearity of the equation allows to show both existence of a stationary solution $\rho_\infty^\delta$, and to show the convergence $\rho(t)\xrightarrow[]{t\to\infty}\rho_\infty^\delta$ in appropriate topologies, noticing that the mass of both $\rho$
 and $\rho^\delta_\infty$ need to be the same.
\begin{cor}{\rm \cite[Theorem 5.7]{Por}}\label{Asymptotic convergence nondegenerate}
	Under the assumptions of Theorem \ref{Theorem oscillation decay}, if $\rho_0\in \mathcal{P}_k(\mathbb{R}^d)$, then there exists a unique stationary measure $\rho_\infty^\delta\in\mathcal{P}(\mathbb{R}^d)$ such that
	\begin{equation}
		\L^*_\delta\rho_\infty^\delta=0\qquad\quad\mbox{in}\quad\mathbb{R}^d\,.
	\end{equation}
	Moreover, it holds that
	\begin{equation}
		\|\rho(t)-\rho_\infty^\delta\|_{\mathcal{P}_k}\le K_\delta\,e^{-\omega t}\qquad\quad\forall t>0\,.
	\end{equation}
\end{cor}

\begin{rem}[About convergence rates]\rm \begin{enumerate}[label=(\roman*)]
\item In the next section we will consider the limits as $\delta\rightarrow0^+$. A careful examination of the proof of Theorem \ref{Theorem oscillation decay} reveals that the constant $K_\delta$ of formula \eqref{oscillation decay.2} blows up when $\delta\rightarrow 0^+$. On the other hand, \textit{the convergence rate $\omega>0$ only depends on $\alpha$ and $k$, but not on $\delta$.} However, it can not be computed explicitly because the proof is not constructive. Also, as far as we know, the above convergence results cannot be applied to the general framework of degenerate diffusion matrices, as described in \cite{Por}.
\item  Corollary \ref{Asymptotic convergence nondegenerate} holds for the nondegenerate equation\eqref{F-P} and proves exponential convergence of solution $\rho(t)$ to the unique invariant measure $\rho_\infty$ with the same mass. This holds whenever the growth of the potential $L$ is (super)quadratic, namely $\gamma\ge 2$. A complementary result of \cite{Por}, that we not consider here, allows to show the same convergence, but only with polynomial decay rates, in the case when $\gamma\in(0,2)$. Indeed, since we are considering approximations of Machine Learning processes, in practice, one can always modify $L$ at ``infinity'' to have (super)quadratic growth.
\item To the best of our knowledge, convergence to a stationary measure in the case of degenerate Fokker-Planck type equations with a non-quadratic loss function $L$, has not been proved.
\end{enumerate}
\end{rem}
We conclude this section by recalling a technical lemma of \cite{Por}, used in the proof of the above results, that we will need in what follows.

\begin{lem}\label{Lyapunov function}\cite[Lemma 3.1]{Por}
If there exist $\alpha,\gamma,R>0$ such that
\begin{equation}\label{confining potential}
	\nabla L(x)\cdot x\ge\alpha|x|^\gamma\qquad\forall x\in \R^d\quad\mbox{with}\quad |x|\ge R,
\end{equation}
and there exists $\sigma<+\infty$ such that
\begin{equation}
	\|Q_\delta(x)\|_\infty\le \sigma \qquad\quad\forall x\in \R^d\,,
\end{equation}
then, for every $\beta>0$ there exists $K_\beta>0$ such that
\begin{equation}
	\L_\delta \langle x\rangle^k \ge(\alpha-\beta)k \, \langle x\rangle^{k+\gamma-2}-K_\beta\qquad\quad\forall x\in\R^d
\end{equation}
\end{lem}

\subsection{Existence of steady states: proof of Theorem \ref{theo:existence}}\label{sec:existence}

In this section we prove, under mild condition, the existence of at least one stationary measure for our Fokker-Planck type equations, as stated in Theorem \ref{theo:existence}, whose statement we recall here.

\noindent
\emph{Consider the degenerate Fokker-Planck equation
	\begin{equation}
	\partial_t \rho=\nabla\cdot\bigg(\frac{\eta}{2\b}\nabla\cdot(\rho Q(x))+\rho\nabla L(x)\bigg)=:-\L^*\rho\,.
	\end{equation}
	If \eqref{Assumption on Sigma}, \eqref{confining potential1} and \eqref{Assumption on L} hold, then there exists  an invariant probability measure $\rho_\infty\in\mathcal{P}(\mathbb{R}^d)$ such that
	\begin{equation}\label{Stat.Eq}
	\int_{\mathbb{R}^d}\L\varphi(x)\d\rho_\infty(x)=0\qquad\quad \forall\varphi\in C_b(\Omega)\,.
	\end{equation}
}
	The idea of the proof is to consider the invariant measures of the nondegenerate equation with diffusion matrix $Q_\delta(x)=\b^{-1}Q(x)+\delta I_{d\times d}$ and apply Prokhorov's Theorem. By Corollary \eqref{Asymptotic convergence nondegenerate}, we know that for each $\delta>0$ there exists $\rho^\delta_{\infty}\in\mathcal{P}(\mathbb{R}^d)$ such that
	$$-\L_\delta^*\rho_\infty^\delta=\nabla\cdot\left(\frac{\eta}{2}\nabla\cdot(\rho^\delta_\infty Q_\delta(x))+\rho^\delta_\infty \nabla L(x)\right)= 0\,.$$
	In order to prove that the sequence $\lbrace \rho^\delta_\infty\rbrace_{\delta>0}$ is tight we will use $\langle x\rangle^k$ as a Lyapunov function. Notice that for any $\varepsilon\in(0,1)$, by Lemma \ref{Lyapunov function}, it holds that
	\begin{equation*}
	\L_\delta\langle x\rangle^k\ge k\langle x\rangle^{k-2}(\alpha|x|^\gamma-\bar K)\,,
	\end{equation*}
	with $\bar K:=d(\sigma+1)(d+k-2)$. Hence, we can find two nonnegative functions $\phi,\chi\in C^2(\mathbb{R}^2)$ such that $\lim\limits_{|x|\rightarrow\infty}\phi(x)=+\infty$, $\chi$ is compactly supported and
\begin{equation}\label{epsilon invariant measure}
\L_\delta\langle x\rangle^k\ge\phi(x)-\chi(x)\qquad\quad\forall\delta\in(0,1)\,.
\end{equation}
	Now, let us prove that $\lbrace \rho_\infty^\delta\rbrace_{0<\delta<1}$ is tight, that is, for any $\varepsilon>0$ there exists $U_\varepsilon\subset\mathbb{R}^d$ such that  $\int_{U_\varepsilon}\d \rho_\infty^\delta(x)>1-\varepsilon$ for every $\delta\in(0,1)$. Using $\L_\delta^*\rho_\infty^\delta=0$ and inequality \eqref{epsilon invariant measure}, we obtain
	\begin{align*}
	0&=\int_{\R^d}\langle x\rangle^k\L_\delta^*\rho^\delta_\infty\dx=\int_{\R^d}\L_\delta\langle x\rangle^k \d\rho^\delta_\infty(x)\ge\int_{\R^d}(\phi(x)-\chi(x))\d\rho^\delta_\infty(x)\,,
	\end{align*}
	and hence
	$$0\le\int_{\mathbb{R}^d}\phi(x)\d\rho^\delta_\infty(x)\le \int_{\rm supp(\chi)}\chi(x)\d\rho_\infty^\delta(x)\le \|\chi\|_\infty\qquad\forall\delta\in(0,1)\,.$$
	On the other hand, since $\phi(x)\rightarrow+\infty$ when $|x|\rightarrow\infty$, for any $n>0$ there exists $R_n>0$ such that $\phi(x)\ge n$ for $x\in B_{R_n}$. Thus,
	\begin{align*}
	\int_{B^c_{R_n}}n\,\d \rho_\infty^\delta\le\int_{B^c_{R_n}}\phi(x)\,\d \rho_\infty^\delta\le\|\chi\|_\infty\qquad\Rightarrow\qquad\int_{B^c_{R_n}}\d \rho_\infty^\delta\le \frac{\|\chi\|_\infty}{n}\,,
	\end{align*}
	which implies
	\begin{align*}
	\int_{B_{R_n}}\d \rho_\infty^\delta>1-\frac{\|\chi\|_\infty}{n}\qquad\quad\forall n>0\,.
	\end{align*}	
	Therefore, we can apply Prokhorov's Theorem to obtain that there exists $\rho_\infty\in\mathcal{P}(\mathbb{R}^d)$ such that $\rho_\infty^\delta\rightharpoonup\rho_\infty$ in $\mathcal{P}(\mathbb{R}^d)$ up to a subsequence.
	
	It remains to prove that $\rho_\infty$ solves the equation. For any $\varphi\in C_c^\infty(\mathbb{R}^d)$, it holds that
	\begin{align*}
		0&\le\left|\int_{\mathbb{R}^d}\L\varphi(x)\d\rho_\infty(x)\right|=\left|\int_{\mathbb{R}^d}\L\varphi(x)\d\rho_\infty(x) -\int_{\mathbb{R}^d}\L_\delta\varphi(x)\d\rho_\infty^\delta(x)\right|\\
		&\le\left|\int_{\mathbb{R}^d}\L\varphi(x)\;\d(\rho_\infty-\rho_\infty^\delta)(x)\right| +\left|\int_{\mathbb{R}^d}(\L\varphi(x)-\L_\delta\varphi(x))\d\rho_\infty^\delta(x)\right|\\ &\le\langle\L\varphi,\rho_\infty-\rho_\infty^\delta\rangle_{C_b(\mathbb{R}^d),\mathcal{M}(\mathbb{R})}+\|\L\varphi-\L_\delta\varphi\|_{C_b(\mathbb{R}^d)} \stackrel{\delta\rightarrow0^+}{\longrightarrow}0\,,
	\end{align*}
	which concludes the proof, since $\|\L\varphi-\L_\delta\varphi\|_{C_b(\mathbb{R}^d)}\le C\,\delta\|\varphi\|_{C^2(\mathbb{R}^d)}$\,, for some $C>0$\,. \hfill \qed

\medskip

Existence of stationary states -possibly measures- for quite general Fokker-Plank equations nowadays is a classical topic, we refer for instance to Chapter 2 of \cite{BogaKrylov} and references therein. In particular, similar results,    which do not apply directly to our specific case,   have been obtained in \cite[Corollary 2.4.4]{BogaKrylov}, with slightly different assumptions and proofs. Uniqueness is in general a delicate issue and in general may turn out to be false.

\subsection{The question of convergence to stationary measures}\label{sec: question}

Once existence of stationary measures is established, the next natural question would be which solutions converge to it (for large times), and this will be explored in the next section. In general we cannot expect better than convergence in a topology compatible with measures, since stationary solutions may happen to be purely measures (not functions), in view of the Examples 6 and 7 below, see also Example 8 in the next section. Also, since the diffusion matrix can be highly degenerate, so that, heuristically, part of the flow can be driven by ``pure transport'', in which case, the stationary solution can be a ``pure measure''. This motivates the following example.

\noindent{\sc Example 6.}
	Let us consider the following transport equation
	\begin{equation}\label{Example Transport}
	\begin{cases}
	\partial_t\rho(t,x)\!\!\!\!!&=\nabla\cdot\left(Cx\rho(t,x)\right)\qquad\mbox{in}\quad(0,\infty)\times\mathbb{R}^d\,,\\
	\rho(0,x)&=\rho_0(x) \qquad\quad\qquad\qquad\;\,\mbox{in}\quad\mathbb{R}^d\,,
	\end{cases}
	\end{equation}
    where $C\in \mathbb{R}^{d\times d}$ with $C>0$.
	Notice that $\rho_\infty(x)=\delta_0(x)$ is a stationary measure since if we multiply by a test function $\varphi\in C_c^1(\mathbb{R}^d)$ and we integrate, it holds that
	$$\int_{\mathbb{R}^d}x\cdot\nabla\varphi(x)\,\d\delta_0(x)=x\cdot\nabla\varphi(x)\bigg|_{x=0}=0\qquad\qquad\forall\varphi\in C_c^1(\mathbb{R}^d)\,.$$
	The equation above has the following explicit solution $\rho(t,x)=e^{\mbox{\footnotesize tr}(C)t}\rho_0(e^{Ct}x)$, since
	$$\partial_t\rho(t,x)=\mbox{tr}(C)e^{\mbox{\footnotesize tr}(C)t}\rho_0(e^{Ct}x)+e^{\mbox{\footnotesize tr}(C)t}\left(\left(Cx\right)\cdot \nabla \rho_0(e^{Ct}x)\right)=\nabla\cdot\left(Cx\rho(t,x)\right)\,.$$
    Notice that the solution $\rho(t,x)$ preserves the mass: using the change of variable $$y=e^{Ct}x\qquad \mbox{and}\qquad \dy=\left|\det{e^{Ct}}\right|\dx=e^{\mbox{\footnotesize tr}(C)t}\dx\,,$$
     we obtain
    $$\int_{\R^d}\rho(t,x)\dx=\int_{\R^d}e^{\mbox{\footnotesize tr}(C)t}\rho_0(e^{Ct}x)\dx=\int_{\R^d}\rho_0(y)\dy\qquad\forall t>0\,.$$

  Let us show the convergence of the solutions of \eqref{Example Transport} to $\delta_0$ whenever $\rho_0\in\mathcal{P}_2(\R^d)$. Let us consider the equation above as a gradient flow for the second moments in the Wasserstein space, namely,
\begin{equation*}
	\begin{cases}
	\partial_t\rho(t,x)\!\!\!\!!&=-\mbox{grad}_{W_2}\E[\rho]:=\nabla\cdot\left(\rho(t,x)\nabla\left(\frac{\delta \E}{\delta \rho}[\rho]\right)\right)\\
	\rho(0,x)&=\rho_0(x)\in\mathcal{P}_2(\mathbb{R}^d)\,.
	\end{cases}
	\end{equation*}
with $$\E[\rho]:=\frac{1}{2}\int_{\R^d}x^T\,C^TC\,x\;\d\rho(x)\qquad\quad\&\quad\qquad\frac{\delta \E}{\delta \rho}[\rho]=Cx\,.$$
Notice that the global minimum of the energy is attained at $\E[\delta_0]=0$. As it can be seen in  \cite[Lemma 4.4.3]{Figalli-Glaudo} since $\E$ is $\lambda$-convex with some $\lambda>0$ (in particular, for any $\lambda\in(0,\lambda_1(C^TC)$), it holds that for every $\rho\in \mathcal{P}_2(\R^d)$
$$\frac{\lambda}{2}W^2_2(\rho,\delta_0)\le\E[\rho]-\E[\delta_0]\le\frac{1}{2\lambda}\langle\mbox{grad}_{W_2}\E[\rho],\mbox{grad}_{W_2}\E[\rho]\rangle_{\rho}\,.$$
This estimates above for the $\lambda$-convex energy are the analogous of
$$\varphi(x)-\varphi(x_0)\ge \frac{\lambda}{2}|x-x_0|^2\qquad\&\qquad|\nabla \varphi(x)|^2\ge2\lambda\left(\varphi(x)-\varphi(x_0)\right)\,,$$
whenever $x_0\in\R^d$ is the unique minimum of a $\lambda$-convex function $\varphi\in C^1(\R^d)$.

Hence, following equation (4.18) of \cite{Figalli-Glaudo} we obtain
$$\frac{\d}{\dt}\E[\rho]=\langle\mbox{grad}_{W_2}\E[\rho],\partial_t\rho\rangle_{W_2} =-\langle\mbox{grad}_{W_2}\E[\rho],\mbox{grad}_{W_2}\E[\rho]\rangle_{\rho}\le -2\lambda\E[\rho]\,,$$
which implies the decay of the energy
$$\E[\rho]\le e^{-2\lambda t}\E[\rho_0]\,.$$
Indeed, since $\E[\delta_0]=0$, we have
$$\frac{\lambda}{2}W^2_2(\rho,\delta_0)\le\E[\rho]-\E[\delta_0]=\E[\rho]\le e^{-2\lambda t}\E[\rho_0]\,.$$
Thus, there is exponential Wasserstein convergence to the Dirac's delta at zero.
In conclusion, if we assume a condition on the decay of the initial datum $\rho_0(x)$ when $|x|\rightarrow\infty$, such as finite second moments, then the solution converge to $\delta_0$.

\medskip

A similar phenomena can also happen in the case of non-zero diffusion matrix (still very degenerate), as the next example shows. See also Example 8 in the next section for a generalization to higher dimensions.

\medskip

\noindent{\sc Example 7.} Let us analyze the fundamental solution of the following degenerate Fokker-Planck equation:
\begin{equation}\label{Nonsmooth fundamental solution}
   \partial_t u=\nabla\cdot
   \left(\begin{pmatrix}
              1 & 0 \\
               0 & 0
        \end{pmatrix} \nabla u
   +u\begin{pmatrix}
       1 & 0 \\
       0 & 1
     \end{pmatrix}
     \begin{pmatrix}
       x \\
       y
     \end{pmatrix}
   \right)\qquad\mbox{in}\quad(0,\infty)\times\R^2\,.
\end{equation}
In this case, condition \eqref{II} is not satisfied since
$$\underbrace{\begin{pmatrix}
       1 & 0 \\
       0 & 0
     \end{pmatrix}}\limits_{Q_0}
     \begin{pmatrix}
       0 \\
       0
     \end{pmatrix}
     =
      \begin{pmatrix}
       0 \\
       1
     \end{pmatrix}
\qquad\mbox{and}\qquad
\underbrace{\begin{pmatrix}
       1 & 0 \\
       0 & 1
     \end{pmatrix}}\limits_{C}
     \begin{pmatrix}
       0 \\
       1
     \end{pmatrix}
     =
     \begin{pmatrix}
       0 \\
       1
     \end{pmatrix}\,.$$
This orthogonality between the diffusion direction and the pure drift direction suggest that the fundamental solution should be of the form
$$H(t,x)=g(t,x)h(t,y)\,,$$
where $g$ is the fundamental solution for the one dimensional problem
\begin{equation}\label{Example diffusion term}
\partial_t g=\partial_{x_1}\left(\partial_{x}g+xg\right)\qquad\mbox{in}\quad(0,\infty)\times\R\,,
\end{equation}
and $h$ solves the one dimensional transport equation
\begin{equation}\label{Example transport term}
  \partial_t h=\partial_{y}\left(y h\right)\qquad\mbox{in}\quad(0,\infty)\times\R\,.
\end{equation}
Using the explicit solutions of \eqref{Example diffusion term} and \eqref{Example transport term}, we conclude that the fundamental solution of \eqref{Nonsmooth fundamental solution} is
$$H(t,x,y)=\underbrace{\frac{2}{\sqrt{2\pi}(1-e^{-2t})}\, e^{-\frac{2\,x^2}{(1-e^{-2t})}}}\limits_{g(t,x)}\;\underbrace{e^t\delta_0(e^ty)}\limits_{h(t,y)}\,.$$
Thus, if we want to obtain the solution of the Cauchy problem associated to  \eqref{Nonsmooth fundamental solution} with initial datum $u_0$ regular enough, we have just to convolve $u_0$ with the fundamental solution:
\begin{align*}
u(t,x,y)&=\big(H(t,\cdot,\cdot)*u_0\big)(x,y)=\int_{\R}\int_{\R}g(t,x-\tilde{x})h(t,y-\tilde{y})u_0(\tilde{x},\tilde{y})\d\tilde{x}\d\tilde{y}\\
&=e^t\int_{\R}g(t,x-\tilde{x})u_0(\tilde{x},e^ty)\d\tilde{x}
\end{align*}
Notice that for the asymptotic behaviour is not clear what should be the steady states. On the one hand, fixed $y\in\mathbb{R}$, $\big(g(t,\cdot)*u_0(\cdot,y)\big)(x)$ converge to the gaussian $$g_\infty(x)=\frac{1}{\sqrt{2\pi}}e^{-\frac{x^2}{2}}\,.$$
But on the other hand, if we fix $x\in\mathbb{R}$, then $\big(h(t,\cdot)*u_0(x,\cdot)\big)(y)$ converge to $\delta_0(y)$ as a measure. Therefore, our educated guess for the steady state is
\vspace{-2mm}
\begin{multicols}{2}
\begin{equation*}
  u_\infty(x,y)=\frac{1}{\sqrt{2\pi}}e^{-\frac{x^2}{2}}\,\delta_0(y)\,.
\end{equation*}
\color{white} as
\color{black}
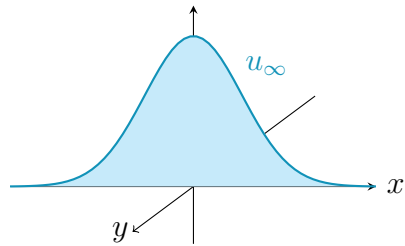
\begin{Figure}
  \centering
\pgfmathdeclarefunction{gauss}{2}{%
    \pgfmathparse{1/(#2*sqrt(2*pi))*exp(-((x-#1)^2)/(2*#2^2))}%
}

\begin{tikzpicture}[scale=0.8]
    \begin{axis}[
        no markers,
        axis x line = center,
        axis y line = center,
        xlabel = {$ x $}, xlabel style = {right},
        ylabel = {}, ylabel style = {above},
        xmin = -3, xmax = 3,
        ymin = -0.19, ymax = 0.6,
        hide obscured x ticks=true,
        xtick=\empty,
        ytick = \empty,
        x = 1cm, y = 5cm,
        domain = -5:5,
        samples = 100
    ]
        \addplot [<-,black,domain=-1:0]{0.15*(x)};
        \addplot [black,dashed,domain=0:1.17]{0.15*(x)};
        \addplot [black,domain=1.17:2]{0.15*(x)};
        \addplot [fill=cyan!20, draw=none, domain=-5:5,opacity=0.3] {gauss( 0, 0.8)} \closedcycle;
        \addplot [thick, cyan!70!black] {gauss( 0, 0.8)};

    \node[font = \color{cyan!70!black}] at (1.2,0.4) {$u_\infty$};
    \node[] at (-1.2,-0.15) {$y$};

    \end{axis}
\end{tikzpicture}
\end{Figure}

\end{multicols}
\begin{lem}\label{Example mixed convergence}
  Let $u$ be a solution of \eqref{Nonsmooth fundamental solution} with initial datum $u_0\in L^1(\mathbb{R}^2)\cap \mathcal{P}_2(\mathbb{R}^2)$. Then,
  \begin{equation*}
    u(t)\rightarrow u_\infty\qquad\mbox{as }t\rightarrow\infty\,,
  \end{equation*}
  in 2-Wasserstein sense.
\end{lem}
\begin{proof}
As shown in \cite[Theorem 2.7]{Ambrosio-Gigli}, we can prove the 2-Wasserstein convergence by
\begin{enumerate}[label=\roman*)]
  \item Convergence of second moments:
  $$\int_{\mathbb{R}^2}(|x|^2+|y|^2)u(t,x,y)\dx\dy\stackrel{t\rightarrow \infty}{\longrightarrow}\int_{\mathbb{R}^2}(|x|^2+|y|^2)u_\infty(x,y)\dx\dy\,.$$
  \item Weak convergence of measures: $\forall \varphi\in \mbox{Lip}(\mathbb{R}^2)\cap L^\infty(\mathbb{R}^2)$
   $$\int_{\mathbb{R}^2}\varphi(x,y) u(t,x,y)\dx\dy\stackrel{t\rightarrow \infty}{\longrightarrow}\int_{\mathbb{R}^2}\varphi(x,y)u_\infty(x,y)\dx\dy\,.$$
\end{enumerate}
\noindent{\bf Step 1.}{\it Convergence of second moments.}
Note that the second moment of $u_\infty$ is
\begin{align*}
  \int_{\mathbb{R}^2}\left(|x|^2+|y|^2\right)u_\infty(x,y)\dx\dy
  =\int_{\mathbb{R}}|x|^2g_\infty(x)\dx=1\,.
\end{align*}
Then, it suffices to prove that the following tends to 0:
\begin{align*}
  &\left| \int_{\mathbb{R}^2}\left(|x|^2+|y|^2\right)u(t,x,y)\dx\dy-1\right|\\
  =&\left| \int_{\mathbb{R}^2}\left(|x|^2+|y|^2\right)\left(\int_{\mathbb{R}}g(t,x-\tilde{x})e^tu_0(\tilde{x},e^ty)\d\tilde{x}\right)\dx\dy-1\right|\\
  \le&\left| \int_{\mathbb{R}^2}|x|^2\left(\int_{\mathbb{R}}g(t,x-\tilde{x})e^tu_0(\tilde{x},e^ty)\d\tilde{x}\right)\dx\dy-1\right|\\
  &+\left| \int_{\mathbb{R}^2}|y|^2\left(\int_{\mathbb{R}}g(t,x-\tilde{x})e^tu_0(\tilde{x},e^ty)\d\tilde{x}\right)\dx\dy\right|\\
  \le&\left| \int_{\mathbb{R}}\left(\int_{\mathbb{R}}|x|^2g(t,x-\tilde{x})\dx\right)\left(\int_{\mathbb{R}}u_0(\tilde{x},y)\dy\right)\d\tilde{x}-1\right|\\
  &+e^{-2t}\left| \int_{\mathbb{R}^2}g(t,x-\tilde{x})\left(\int_{\mathbb{R}}|y|^2u_0(\tilde{x},y)\dy\right)\d\tilde{x}\dx\right|\\
  =&:I+II\,,
\end{align*}
with
$$u_0^{X}(\tilde{x})=\int_{\mathbb{R}}u_0(\tilde{x},y)\dy\,.$$
In order to estimates  $I$, note that
\begin{align*}
  \int_{\mathbb{R}}|x|^2g(t,x-\tilde{x})\dx=1+e^{-2t}(|\tilde{x}|^2-1)\,.
\end{align*}
Since $u_0^{X}\in\mathcal{P}_2(\mathbb{R})$, it holds that
\begin{align*}
  I&=\left| \int_{\mathbb{R}}\left(\int_{\mathbb{R}}|x|^2g(t,x-\tilde{x})\dx\right)u_0^{X}(\tilde{x})\d\tilde{x}-1\right|\\
  &=\left| \int_{\mathbb{R}}\left(1+e^{-2t}(|\tilde{x}|^2-1)\right)u_0^{X}(\tilde{x})\d\tilde{x}-1\right|\\
  &\le e^{-2t}\left(\int_{\mathbb{R}^d}(|x|^2+|y|^2)u_0(x,y)\dx\dy+1\right)\stackrel{t\rightarrow\infty}{\longrightarrow}0\,.
\end{align*}
On the other hand, we have that
\begin{align*}
  II&=e^{-2t}\left|\int_{\mathbb{R}}\left(\int_{\mathbb{R}}g(t,x-\tilde{x})\dx\right)\left(\int_{\mathbb{R}}|y|^2u_0(\tilde{x},y)\dy\right)\d\tilde{x}\right|\\
  &=e^{-2t}\left|\int_{\mathbb{R}}\left(\int_{\mathbb{R}}|y|^2u_0(\tilde{x},y)\dy\right)\d\tilde{x}\right|\\
  &\le e^{-2t}\left(\int_{\mathbb{R}^2}(|\tilde{x}|^2+|y|^2)u_0(\tilde{x},y)\dy\right)\stackrel{t\rightarrow\infty}{\longrightarrow}0\,,
\end{align*}
and the convergence of the second moments follows.

\noindent{\bf Step 2. }{\it Weak convergence of measures.} For any $\varphi\in\mbox{Lip}(\mathbb{R}^2)$ we want to prove that the following quantity tends to 0 as $t$ goes to $\infty$:
\begin{align*}
  \left|\int_{\mathbb{R}^2}\varphi(x,y)\left(u(t,x,y)-u_\infty(x,y)\right)\dx\dy\right|\,.
\end{align*}
By substituting the explicit expression of $u$ and $u_\infty$, we obtain
 \begin{align*}
  &\left|\int_{\mathbb{R}^2}\varphi(x,y)\left(\int_{\mathbb{R}}g(t,x-\tilde{x})e^tu_0(\tilde{x},e^ty)\d\tilde{x}-g_\infty(x)\delta_0(y)\right)\dx\dy\right|\\
  =&\left|\int_{\mathbb{R}^2}\left(\varphi(x,e^{-t}y)-\varphi(x,0)+\varphi(x,0)\right) \left(\int_{\mathbb{R}}g(t,x-\tilde{x})u_0(\tilde{x},y)\d\tilde{x}\right)\dx\dy\right.\\
  &        \left. -\int_{\mathbb{R}}\varphi(x,0)g_\infty(x)\dx\right|\\
  =&\left|\int_{\mathbb{R}^2}\left(\int_{0}^{e^{-t}y}\partial_r\varphi(x,r)\d r \right) \left(\int_{\mathbb{R}}g(t,x-\tilde{x})u_0(\tilde{x},y)\d\tilde{x}\right)\dx\dy\right.\\
  &        \left. -\int_{\mathbb{R}}\varphi(x,0)\left(\int_{\mathbb{R}}g(t,x-\tilde{x})u_0^{X}(\tilde{x})\d\tilde{x}-g_\infty(x)\right)\dx\right|\\
  \le&\,A+B\,,
\end{align*}
where
\begin{align*}
  A&=\left|\int_{\mathbb{R}^2}\left(\int_{0}^{e^{-t}y}\partial_r\varphi(x,r)\d r \right) \left(\int_{\mathbb{R}}g(t,x-\tilde{x})u_0(\tilde{x},y)\d\tilde{x}\right)\dx\dy\right|\\
  &\le  e^{-t}\|\nabla\varphi\|_{L^{\infty}(\mathbb{R}^2)} \int_{\mathbb{R}^2}g(t,x-\tilde{x})\left(\int_{\mathbb{R}}|y|u_0(\tilde{x},y)\dy\right)\d\tilde{x}\dx\\
  &\le  e^{-t}\|\nabla\varphi\|_{L^{\infty}(\mathbb{R}^2)} \left(\int_{\mathbb{R}^2}(|x|+|y|)u_0(x,y)\dx\dy\right)\stackrel{t\rightarrow\infty}{\longrightarrow}0
\end{align*}
and
\begin{align*}
 B&=\left|\int_{\mathbb{R}}\varphi(x,0)\left(\int_{\mathbb{R}}g(t,x-\tilde{x})u_0^{X}(\tilde{x})\d\tilde{x}-g_\infty(x)\right)\dx\right|\\
 &\le\left(\int_{\mathbb{R}}\left|\varphi(x,0)\right|^2g_\infty(x)\dx\right)^{1/2} \left(\int_{\mathbb{R}}\left[u^{X}(t,x)-g_\infty(x)\right]^2g_\infty^{-1}(x)\dx\right)^{1/2}\\
 &=\left(\int_{\mathbb{R}}\left|\varphi(x,0)\right|^2g_\infty(x)\dx\right)^{1/2} \|u^{X}(t)-g_\infty\|_{L^2(g_\infty^{-1})}\;\;\stackrel{t\rightarrow \infty}{\longrightarrow}0\,,
\end{align*}
with $u^{X}(t,x)=\int_{\mathbb{R}}g(t,x-\tilde{x})u_0^{X}(\tilde{x})\d\tilde{x}$ being the solution to \eqref{Example diffusion term} with initial datum $u_0^{X}$.
\end{proof}

\subsection{Entropy method}\label{sec Entropy method}

Once we have established the existence of a stationary measure (in the general case) we address the question of  convergence towards it, possibly with rates. As already mentioned, we shall use entropy methods, and for this reason we shall assume that the stationary measure $\rho_\infty$ is indeed an $L^1$ function, that we fix throughout this section. We explore possible adaptations to our problem of the classical entropy method introduced by Bakry and Émery in \cite{Bakry-Emery}, see also \cite{Bakry-Ledoux}.
Recall the  stochastic approximation of order 1 of the SGD defined by
\begin{equation*}
  \d X_t=-\nabla L(X_t)\dt+\sqrt{\frac{\eta}{\b}Q(X_t)} \, \textnormal{d} W_t\,,
\end{equation*}
together with the following Fokker-Planck equation for the associated probability density function $\rho$\,:
\begin{equation}\label{FP variant}
  \begin{cases}
    \partial_t\rho&=\nabla\cdot\bigg(\varepsilon^2\nabla\cdot\left(Q(x)\rho\right) +\rho\nabla L(x)\bigg)\\
    \rho(0,\cdot)&=\rho_0\in L^2(\mathbb{R}^d, \rho_\infty^{-1}\dx).
  \end{cases}
\end{equation}
where $\varepsilon^2= \frac{\eta}{2\b}$.
Entropy methods aim at obtaining differential inequalities between the entropy functional and its time derivative (entropy production), by means of suitable functional inequalities, typically of weighted Poincar\'e type. There may be several possible entropies for the above equation, however the following choice seems the most appropriate.

\begin{lem}[Entropy and Entropy production]\label{EEP-new}
  Let us consider the following relative entropy associated to $\rho_\infty$,
  \begin{equation*}
    \mathcal{E}\left(\mu\,|\,\rho_\infty\right)=\frac{1}{2}\int_{\mathbb{R}^d}\left(\frac{\mu(x)}{\rho_\infty(x)}-1\right)^2\rho_\infty(x)\dx\,.
  \end{equation*}
  Then, its derivative along the Fokker-Planck flow \eqref{FP variant} is given by the Fisher information, or entropy production:
  \begin{equation*}
    \mathcal{I}\left(\mu\,|\,\rho_\infty\right) =\int_{\mathbb{R}^d}\nabla\left(\frac{\mu(x)}{\rho_\infty(x)}\right)^TQ(x)\nabla\left(\frac{\mu(x)}{\rho_\infty(x)}\right)\rho_\infty(x)\dx
  \end{equation*}
  More precisely,
  \begin{equation}\label{E-EP.aaa}
  \frac{\rm d}{\dt}\mathcal{E}(\rho(t)|\rho_\infty)=-\varepsilon^2 \mathcal{I}\left(\rho(t)\,|\,\rho_\infty\right)\,, \qquad\mbox{for all $t>0$}.
  \end{equation}
\end{lem}
\begin{rem}An immediate consequence of the assumption $\rho_0\in L^2(\mathbb{R}^d, \rho_\infty^{-1}\dx)$, is that the entropy is finite for all times, indeed $\mathcal{E}(\rho_0|\rho_\infty)<\infty$, hence the above lemma implies $\mathcal{E}(\rho(t)|\rho_\infty)<\infty$ for all $t>0$.
\end{rem}
\begin{proof}Let us differentiate the entropy and use the equation of $\rho$ to obtain
  \begin{align*}
    \frac{\d}{\dt}\mathcal{E}(\rho(t)\,|\,\rho_\infty)=&\int_{\mathbb{R}^d}\partial_t\rho(t)\left(\frac{\rho(t)}{\rho_\infty}-1\right)\dx\\
    =&\int_{\mathbb{R}^d}\nabla\cdot\bigg(\varepsilon^2 Q\nabla\rho(t)+\rho(t)\left(\nabla L+\varepsilon^2\nabla\cdot Q\right)\bigg)\left(\frac{\rho(t)}{\rho_\infty}-1\right)\dx\\
    =&\int_{\mathbb{R}^d}\nabla\cdot\bigg(\rho_\infty\varepsilon^2 Q\nabla\left(\frac{\rho(t)}{\rho_\infty}\right)\bigg)\left(\frac{\rho(t)}{\rho_\infty}-1\right)\dx\\
    &+\int_{\mathbb{R}^d}\nabla\cdot\bigg[\rho(t)\left(\varepsilon^2Q\frac{\nabla\rho_\infty}{\rho_\infty}+\varepsilon^2\nabla\cdot Q+\nabla L\right)\bigg]\left(\frac{\rho(t)}{\rho_\infty}-1\right)\dx\,.
  \end{align*}
  After integration by parts in the second integral above, we have that
  \begin{align*}
    \frac{\d}{\dt}\mathcal{E}(\rho(t)|\rho_\infty)= &
    -\varepsilon^2\mathcal{I}(\rho(t)\,|\,\rho_\infty)-\int_{\mathbb{R}^d}\left(\varepsilon^2Q\frac{\nabla\rho_\infty}{\rho_\infty}+\varepsilon^2\nabla\cdot Q+\nabla L\right)\cdot\nabla\left(\frac{\rho(t)}{\rho_\infty}\right)\rho(t)\dx\\
    =& -\varepsilon^2\mathcal{I}(\rho(t)\,|\,\rho_\infty)\!-\!\frac{1}{2}\int_{\mathbb{R}^d}\!\!\left(\varepsilon^2Q\frac{\nabla\rho_\infty}{\rho_\infty}+\varepsilon^2\nabla\cdot Q+\nabla L\right)\cdot\nabla\left(\frac{\rho^2(t)}{\rho^2_\infty}\right)\rho_\infty\!\dx\\
    =& -\varepsilon^2\mathcal{I}(\rho(t)\,|\,\rho_\infty)\!-\!\frac{1}{2}\!\int_{\mathbb{R}^d}\!\!\bigg(\!\varepsilon^2Q\nabla\rho_\infty+\rho_\infty\!\left(\varepsilon^2\nabla\cdot Q+\nabla L\right)\bigg)\cdot\nabla\left(\frac{\rho^2(t)}{\rho^2_\infty}\right)\!\dx\\
    =& -\varepsilon^2\mathcal{I}(\rho(t)\,|\,\rho_\infty)+\frac{1}{2}\int_{\mathbb{R}^d}\nabla\cdot\left(\varepsilon^2\nabla\cdot(Q\rho_\infty)+\rho_\infty\nabla L\right)\frac{\rho^2(t)}{\rho^2_\infty}\dx\,.
  \end{align*}
  Note that the second term vanishes since $\rho_\infty$ is a steady state, satisfying equation \eqref{Stat.Eq}.
\end{proof}

In order to establish a differential inequality for the entropy, we need a relationship between the entropy and its entropy production, which typically takes the form of a weighted Poincaré inequality associated to \eqref{FP variant}.
\begin{defi}[Poincaré inequality]\label{poincare ineq}
Let $ Q(x)\ge 0$ in the sense of matrices, for every $x\in\mathbb{R}^d$. We say that a \emph{Poincaré inequality }holds with respect to  $\rho_\infty\in\mathcal{P}(\mathbb{R}^d)$ if there exists $\lambda > 0$ such that, for all $f\in L^2(\mathbb{R}^d,\rho_\infty\dx)$
\begin{equation}\label{Poinc.aaa}
\lambda\int_{\mathbb{R}^d}\left(f-\int_{\mathbb{R}^d}f\rho_\infty\dx\right)^2 \rho_\infty\dx\le  \varepsilon^2 \int_{\mathbb{R}^d}\left(\nabla f^TQ(x)\nabla f\right)\rho_\infty\dx.
\end{equation}
\end{defi}
Indeed, letting $f=\mu/\rho_\infty$ with $\int_{\mathbb{R}^d}\mu(x)\dx=1$, the above inequality reads as the entropy-entropy production inequality
\begin{equation}\label{Poinc.E-EP}
  \mathcal{E}(\mu|\rho_\infty)\le \frac{\varepsilon^2}{\lambda}\mathcal{I}\left(\mu\,|\,\rho_\infty\right).
\end{equation}
Note that $f\in L^2(\mathbb{R}^d,\rho_\infty\dx)$ if and only if $\mu\in L^2(\mathbb{R}^d,\rho_\infty^{-1}\dx)$\,.

\medskip

The above Poincar\'e inequality is relevant because it allows to derive exponential decay of the entropy, and show convergence towards $\rho_\infty$ with precise rates, given by $\lambda>0$.

\begin{pro}\label{Poincare to entropy convergence}
A Poincaré inequality rewritten in the form \eqref{Poinc.E-EP} holds for any $\mu\in L^2(\mathbb{R}^d,\rho_\infty^{-1}\dx)$ if and only if there is an exponential convergence (with rate $\lambda>0$) of the relative entropy between the solution $\rho(t)$ of \eqref{FP variant} starting at $\rho_0\in L^2(\mathbb{R}^d,\rho_\infty^{-1}\dx)$ and the steady state $\rho_\infty$, that is
\begin{equation}\label{exp.decay.aaa}
  \mathcal{E}\left(\rho(t)\,|\,\rho_\infty\right)\le e^{-\lambda t}\mathcal{E}\left(\rho_0\,|\,\rho_\infty\right)\qquad\forall t\ge 0\,.
\end{equation}
\end{pro}

\begin{proof}One implication follows by using the Poincaré inequality with $f=\rho(t)/\rho_\infty$ in  the equivalent form \eqref{Poinc.E-EP} to obtain a closed differential inequality
  \begin{align*}
    \frac{\d}{\dt}\mathcal{E}(\rho(t)\,|\,\rho_\infty)=-\varepsilon^2\mathcal{I}(\rho(t)\,|\,\rho_\infty) \le-\lambda\mathcal{E}(\rho(t)\,|\,\rho_\infty)\,.
  \end{align*}
  The result follows by integrating the above inequality in $[0,t]$.

The other implication, follows by differentiating in time \eqref{exp.decay.aaa} and using \eqref{E-EP.aaa}:
  \begin{equation*}
    -\varepsilon^2\mathcal{I}(\rho(t)\,|\,\rho_\infty)=\frac{\d}{\dt}\mathcal{E}(\rho(t)\,|\,\rho_\infty)\le -\lambda e^{-\lambda t}\mathcal{E}(\rho_0\,|\,\rho_\infty)\,,
  \end{equation*}
  Finally, the Poincar\'e inequality \eqref{Poinc.aaa} for $f= \rho_0/\rho_\infty$, follows by choosing $t=0$ and $\rho_0=\mu$.
  \end{proof}
Proving weighted Poincaré inequalities for a general nonnegative matrix $Q$ and a general loss function $L$ is a difficult challenge, in particular since the form of $\rho_\infty$ is not explicit (except some special cases). To the best of our knowledge, only very partial results, under quite strong conditions, appear in the literature. We collect them hereafter, adapted to our notations.
\begin{teo}\cite{Menz}\label{Menz}
Assume that $Q(x)=\sigma I_{d\times d}$ for every $x\in\mathbb{R}^d$ and $L$ is a nonnegative Morse function satisfying
\begin{align*}
  \liminf\limits_{|x|\rightarrow\infty}\left|\nabla L(x)\right|&\ge c_1\,,\\
  \liminf\limits_{|x|\rightarrow\infty}\left(\left|\nabla L(x)\right|^2-\Delta L(x)\right)&\ge-c_2\,,
\end{align*}
for some $c_1,c_2>0$.
Then the unique invariant probability measure of \eqref{FP variant}
is
\begin{equation*}
  \rho_\infty(x)=c\, e^{-\frac{L(x)}{\varepsilon^2\sigma}}\,,
\end{equation*}
with $c>0$ being a normalization constant,
and the Poincaré inequality \eqref{Poinc.aaa} holds.
\end{teo}
In the isotropic framework presented in Theorem \ref{Menz}, we can use  Proposition \ref{Poincare to entropy convergence} to conclude the following entropy convergence.
\begin{cor}
  Under the conditions of Theorem \ref{Menz}, it holds that
  \begin{equation}
      \mathcal{E}\left(\rho(t)\,|\,\rho_\infty\right)\le e^{-\lambda t}\mathcal{E}\left(\rho_0\,|\,\rho_\infty\right)\qquad\forall t>0\,.
  \end{equation}
\end{cor}
If we want to consider more general diffusion matrices $Q$ which can be degenerate, only few results are available.
In this context of entropy methods, Fokker Planck equations with constant yet degenerate diffusion matrix $Q_0$ and quadratic loss function $L(x)=\tfrac{1}{2}x^T C x$ are studied by Arnold and Erb in \cite{Arnold2014}.  Namely,
\begin{equation}\label{Simple Degenerate FP}
  \begin{cases}
    \quad\;\,\partial_t u=\nabla\cdot\left(Q_0\nabla u+u\,Cx\right)\qquad\mbox{in}\quad (0,\infty)\times\R^d \,,\\
    u(0,x)=u_0(x)\hspace{33.5mm}\mbox{in }\quad \R^d \,,
  \end{cases}
\end{equation}
For related Fokker Planck equations see \cite{Arnold2015,Arnold-Toscani2001}. In \cite{Arnold2014}, the authors adapted the entropy method for equation \eqref{Simple Degenerate FP} under following assumptions:
\begin{enumerate}[leftmargin=8.5mm]
    \myitem[I]\label{I}  {\it Confinement potential:} $C$ is positive definite  {and symmetric}\footnote{Note that the authors in \cite{Arnold2014} address a slightly more general problem where $C$ is not necessarily symmetric and hence  they assume that $C$ is positive stable, i.e. all eigenvalues have positive real part.}.
    \myitem[II]\label{II} {\it Hörmander condition:} There are no eigenvectors  of $C$ in ker$\,Q_0$.
\end{enumerate}
Assumption \eqref{I} ensures that the mass is not escaping at infinity along the flow. However, the idea behind Hörmander condition \eqref{II} in the equation \eqref{Simple Degenerate FP} is  that if there is a point where diffusion is not acting, then the drift term will push it to the diffusion regime, see the discussion of \eqref{Simple Degenerate FP} in \cite{Hormander}.

 According to \cite[Theorems 3.1 and 4.9]{Arnold2014}, conditions \eqref{I} and \eqref{II} are equivalent to the existence and uniqueness of a steady state in $L^1(\R^d)$ and imply the convergence of the solution of \eqref{Simple Degenerate FP} to it.

\begin{teo}\cite[Theorem 3.1]{Arnold2014}\label{Arnold steady state}
  There exists a unique steady state $u_\infty\in L^1(\R^d)$ of \eqref{Simple Degenerate FP} if and only if conditions \eqref{I} and \eqref{II} hold. Moreover, the steady state is a non-isotropic Gaussian of the form
  $$u_\infty(x)=c\,e^{-\frac{x^TK^{-1}x}{2}}\,,$$
  with $K$ being the unique, symmetric, and positive definite solution of the Lyapunov equation
  $$2Q_0=CK+KC\,.$$
\end{teo}
We will denote by $\lambda_{\max}(C)$ and $\lambda_{\min}(C)$ the maximum and minimum eigenvalues of a nonnegative definite matrix $C$, respectively.

 \begin{teo}\cite[Theorem 4.9]{Arnold2014}\label{Theorem Arnold}
  Let conditions \eqref{I} and \eqref{II} hold, and let $u$ be the solution to \eqref{Simple Degenerate FP} with $u_0\in L^1(\mathbb{R}^d)\cap \mathcal{P}(\mathbb{R}^d)$ and $\mathcal{E}(u_0\,|\,u_\infty)<\infty$. Then, for every $t>0$
  \begin{equation*}
    \mathcal{E}\left(u(t)\,|\,u_\infty\right)\le c \;e^{-2\gamma t}\mathcal{E}(u_0\,|\,u_\infty)
  \end{equation*}
  with $c>1$ and $\gamma=\lambda_{\min}(C)$ if all the eigenvalues of $C$ are different. If $C$ has repeated eigenvalues, then $\gamma=\lambda_{\min}(C)-\delta$ for every $\delta>0$.
\end{teo}

There are interesting cases that fall out of the hypothesis of Theorem \ref{Theorem Arnold}. Let us consider the following example where Hörmander's condition \eqref{II} does not hold on a given subspace. In this case there is still a steady state, but it will not be smooth anymore. As a consequence, the convergence to equilibrium can only happen in a suitable weak sense.

\vspace{3mm}

\noindent{\sc Example 8.}  Let us consider the equation for the function $u:(0,\infty)\times\mathbb{R}^n\times\mathbb{R}^{d-n}$ with $x\in\mathbb{R}^n$, $y\in\mathbb{R}^{d-n}$ and matrices $Q_0,C_0,C_1,C_2,C_3$ with the corresponding dimensions:
\begin{equation}\label{Generalized simple FP}
  \begin{cases}
    \partial_t u&=\nabla_{x,y}\cdot\left[
    \begin{pmatrix}
      Q_0 & 0 \\
      0 & 0
    \end{pmatrix}
    \begin{pmatrix}
      \nabla_x u  \\
      \nabla_y u
    \end{pmatrix}
    +u
    \begin{pmatrix}
      C_0 & C_1 \\
      C_2 & C_3
    \end{pmatrix}
    \begin{pmatrix}
      x \\
      y
    \end{pmatrix}
    \right]\qquad\mbox{in }\;(0,\infty)\times\mathbb{R}^n\times\mathbb{R}^{d-n}\\
    u(0)&=u_0\hspace{90mm}\mbox{on }\;\mathbb{R}^n\times\mathbb{R}^{d-n}
  \end{cases}
\end{equation}
In order to construct the solution of \eqref{Generalized simple FP}, we need the fundamental solution $g$ defined in $(0,\infty)\times\mathbb{R}^n$ of the equation
\begin{equation*}
  \partial_t g=\nabla\cdot\left(Q_0\nabla g+gC_0 x\right)\qquad\mbox{in }\;(0,\infty)\times\mathbb{R}^n\,.
\end{equation*}
If $Q_0$ and $C_0$ satisfy assumption \eqref{II}, then (see e.g. \cite[Lemma 2.5]{Arnold2014}) we have the explicit expression for $g$ given by
\begin{equation*}
  g(t,x)=\frac{1}{\left(2\pi\right)^{d/2}\det(W(t))}\exp\left(-x^TW^{-1}(t)x\right)\,,
\end{equation*}
where
$$W(t)=\int_{0}^{t}e^{C_0(s-t)}Q_0e^{C^T_0(s-t)}\ds$$
is a positive definite matrix for all $t>0$. In addition, if also condition \eqref{I} is satisfied there exists a unique steady state $g_\infty$ as in Theorem \ref{Arnold steady state}.
\begin{teo}\label{Thm Wasserstein Convergence}
  Assume $Q_0,C_0\in\mathbb{R}^{n\times n}$ satisfy  \eqref{I} and \eqref{II} in $\mathbb{R}^n$. Let $C_1=0\in \mathbb{R}^{n\times(d-n)},\;C_2=0\in\mathbb{R}^{(d-n)\times n}$ and $C_3\in\mathbb{R}^{(d-n)\times (d-n)}$ be positive definite. Let us consider the equation \eqref{Generalized simple FP} with $u_0\in L^1(\mathbb{R}^d)\cap\mathcal{P}_2(\mathbb{R}^d)$ such that
  $$\int_{\mathbb{R}^n}\left(\int_{\mathbb{R}^{d-n}}u_0(x,y)\dy\right)^2g_\infty^{-1}(x)\dx<+\infty\,.$$
   Let $\gamma=\gamma(Q_0,C_0)$ be as in Theorem \ref{Theorem Arnold} and let $\lambda=\min\lbrace 2\gamma,2\lambda_{\max}(C_3)\rbrace$. Then
    \begin{enumerate}[label=\roman*)]
      \item The solution $u(t)$ converge exponentially fast  to $u_\infty$ weakly in measures:
      \begin{align}\label{weak measure convregence}
  \left|\int_{\mathbb{R}^d}\varphi(x,y)\left(u(t,x,y)-u_\infty(x,y)\right)\dx\dy\right|\le\kappa_1 e^{-\lambda t}\qquad\forall\varphi\in \mbox{\rm Lip}(\mathbb{R}^d)\,,
\end{align}
with $\kappa_1$ depending on $u_0,\varphi$ and $g_\infty$.
      \item There is exponential decay of the second moments:
  \begin{equation}\label{convergence second moment}
    \left|\int_{\mathbb{R}^n}\int_{\mathbb{R}^{d-n}}(|x|^2+|y|^2)(u(t,x,y)-u_\infty(x,y))\dx\dy
    \right|\le \kappa_2 e^{-\lambda t}\,,
  \end{equation}
  with $\kappa_2$ depending on $u_0$ and $g_\infty$.
    \end{enumerate}
    As a consequence,
    \begin{equation*}
    u(t,x,y)\longrightarrow u_\infty(x,y):=g_\infty(x)\delta_0(y)
  \end{equation*}
  in the 2-Wasserstein distance as $t\rightarrow+\infty$.
\end{teo}
\begin{proof}
    We will prove i) and ii) for the solution $u$ which is given by
  $$u(t,x,y)=e^{\mbox{tr}(C_3)t}\int_{\mathbb{R}^n}g(t,x-\tilde{x})u_0(\tilde{x},e^{C_3t}y)\d\tilde{x}\,,$$
  see also Example 6 in Section \ref{sec: question}.

\noindent{\it Proof of i).}
By substituting the explicit expression of $u$ and $u_\infty$ and the change of variables $y'=e^{C_3 t}y$, we obtain
 \begin{align*}
 & \left|\int_{\mathbb{R}^2}\varphi(x,y)\left(u(t,x,y)-u_\infty(x,y)\right)\dx\dy\right|\\
  =&\left|\int_{\mathbb{R}^d}\varphi(x,y)\left(\int_{\mathbb{R}^n}g(t,x-\tilde{x})e^{\mbox{tr}(C_3)t}u_0(\tilde{x_1},e^{C_3 t}y)\d\tilde{x}-g_\infty(x)\delta_0(y)\right)\dx\dy\right|\\
  =&\left|\int_{\mathbb{R}^d}\left(\varphi(x,e^{-C_3t}y)-\varphi(x,0)+\varphi(x,0)\right) \left(\int_{\mathbb{R}^n}g(t,x-\tilde{x})u_0(\tilde{x},y)\d\tilde{x}\right)\dx\dy\right.\\
  &        \left. -\int_{\mathbb{R}^n}\varphi(x,0)g_\infty(x)\dx\right|\\
  \le&\int_{\mathbb{R}^d}\left|\varphi(x,e^{-C_3 t}y)-\varphi(x,0)\right| \left(\int_{\mathbb{R}^n}g(t,x-\tilde{x})u_0(\tilde{x},y)\d\tilde{x}\right)\dx\dy\\
  & +  \left|\int_{\mathbb{R}^d}\varphi(x,0)\left(\int_{\mathbb{R}^n}g(t,x-\tilde{x}) u_0(\tilde{x},y)\d\tilde{x}-g_\infty(x)\right)\dx\dy\right|\\
  =&:\,A+B\,.
\end{align*}
Here,
\begin{align*}
  A &= \int_{\mathbb{R}^d}\left|\varphi(x,e^{-C_3 t}y)-\varphi(x,0)\right| \left(\int_{\mathbb{R}^n}g(t,x-\tilde{x})u_0(\tilde{x},y)\d\tilde{x}\right)\dx\dy\\
  &\le \|\nabla\varphi\|_{L^\infty(\mathbb{R}^d)}\int_{\mathbb{R}^{d}}|e^{-C_3t}y|\left(\int_{\mathbb{R}^n}g(t,x-\tilde{x})\dx\right)u_0(x,y)\d\tilde{x}\dy\\
  &\le e^{-\|C_3\| t}\|\nabla\varphi\|_{L^\infty(\mathbb{R}^d)}
  \int_{\mathbb{R}^n}\int_{\mathbb{R}^{d-n}}|y|u_0(x,y)\dx\dy
 \;\stackrel{t\rightarrow\infty}{\longrightarrow}0\,,
\end{align*}
with the induced norm for positive matrices $\|C_3\|=\lambda_{\max}(C_3)$.
For the second term, let us define the marginal functions
$$u_0^X(x)=\int_{\mathbb{R}^{d-n}}u_0(x,y)\dy\qquad\mbox{and}\qquad u^X(t,x)=\int_{\mathbb{R}^n}g(t,x-\tilde{x})u_0^X(\tilde{x})\d\tilde{x}\,.$$
Then,
\begin{align*}
 B&=\left|\int_{\mathbb{R}^n}\varphi(x,0)\left(\int_{\mathbb{R}^d}g(t,x-\tilde{x})u_0^{X}(\tilde{x})\d\tilde{x}-g_\infty(x)\right)\dx\right|\\
 &\le\int_{\mathbb{R}^n}\big|\varphi(x,0)\big|\left|u^X(t,x)-g_\infty(x)\right|\dx\\
 &\le\left(\int_{\mathbb{R}^n}\left|\varphi(x,0)\right|^2g_\infty(x)\dx\right)^{1/2} \left(\int_{\mathbb{R}^n}\left[u^{X}(t,x)-g_\infty(x)\right]^2g_\infty^{-1}(x)\dx\right)^{1/2}\\
 &=\left(\int_{\mathbb{R}^n}\left|\varphi(x,0)\right|^2g_\infty(x)\dx\right)^{1/2} \left[\mathcal{E}\left(u^{X}(t)\,|\,g_\infty\right)\right]^{1/2}\\
 &\le c\; e^{-\lambda t}\left(\int_{\mathbb{R}^n}\left|\varphi(x,0)\right|^2g_\infty(x)\dx\right)^{1/2} \left[\mathcal{E}\left(u^{X}_0\,|\,g_\infty\right)\right]^{1/2}\;\;\stackrel{t\rightarrow \infty}{\longrightarrow}0\,.
\end{align*}
Note that in the second inequality above we have use Cauchy-Schwarz inequality and in the last inequality the entropy decay of Theorem \ref{Theorem Arnold}. Moreover, this convergence implies the weak convergence in measure, see \cite[Theorem 8.8]{Ambrosio-Brue-Semola}.

\noindent{\it Proof of ii).}
  First, note that
  \begin{align*}
  &\left|\int_{\mathbb{R}^d}\left(|x|^2+|y|^2\right) \left(u(t,x,y)-u_\infty(x,y)\right)\dx\dy\right|\\
=&\left|\int_{\mathbb{R}^d}\left(|x|^2+|y|^2\right)\left(e^{\mbox{tr}(C_3)t}\int_{\mathbb{R}^n}g(t,x-\tilde{x})u_0(\tilde{x},e^{C_3t}y)\d\tilde{x} -g_\infty(x)\delta_0(y)\right)\dx\dy\right|\\
\le&\left|\int_{\mathbb{R}^d}|x|^2\left(e^{\mbox{tr}(C_3)t}\int_{\mathbb{R}^n}g(t,x-\tilde{x})u_0(\tilde{x},e^{C_3t}y)\d\tilde{x} -g_\infty(x)\right)\dx\dy\right|\\
&+\left|\int_{\mathbb{R}^d}|y|^2\left(e^{\mbox{tr}(C_3)t}\int_{\mathbb{R}^n}g(t,x-\tilde{x})u_0(\tilde{x},e^{C_3t}y)\d\tilde{x}\right)\dx\dy\right|\\
=:&A+B\,.
\end{align*}
Let us estimate both addends independently. Note that using Fubini-Tonelli and the change of variable $y'=e^{C_3 t}y$ we obtain
\begin{align*}
    A&=\left|\int_{\mathbb{R}^n}|x|^2\left(\int_{\mathbb{R}^n}g(t,x-\tilde{x})u_0^X(\tilde{x})\d\tilde{x} -g_\infty(x)\right)\dx\dy\right|\\
    &=\left|\int_{\mathbb{R}^n}|x|^2\left(u^X(t,x) -g_\infty(x)\right)\dx\dy\right|\,,
\end{align*}
using the same notation for marginals functions as above.
Then, using Cauchy-Schwarz inequality and the entropy decay of Theorem \ref{Theorem Arnold}, it holds that
\begin{align*}
  A&\le\left(\int_{\mathbb{R}^n}|x|^4g_\infty(x)\dx\right)^{1/2}\left(\int_{\mathbb{R}^n}(u^X(t,x)-g_\infty(x))^2g_\infty^{-1}(x)\dx\right)^{1/2}\\
  &=\left(\int_{\mathbb{R}^n}|x|^4g_\infty(x)\dx\right)^{1/2}\left[\mathcal{E}\left(u^X(t)\,|\,g_\infty\right)\right]^{1/2}\\
  &\le  \left(\int_{\mathbb{R}^n}|x|^4g_\infty(x)\dx\right)^{1/2}c\,e^{-\lambda t}\left[\mathcal{E}\left(u^X_0\,|\,g_\infty\right)\right]^{1/2}\;\stackrel{t\rightarrow\infty}{\longrightarrow}0\,.
\end{align*}
For the second term B, consider again the change of variable $\tilde{y}=e^{C_3 t}y$
\begin{align*}
  B&= e^{-2\|C_3\| t}\left|\int_{\mathbb{R}^{d-n}}\int_{\mathbb{R}^n}|\tilde y|^2\left(\int_{\mathbb{R}^n}g(t,x-\tilde{x})\dx \right) u_0(\tilde{x},\tilde y)\d\tilde{x}\d\tilde{y}\right|\\
  &\le e^{-2\|C_3\| t} \int_{\mathbb{R}^n}\int_{\mathbb{R}^{d-n}}(|\tilde x|^2+|\tilde y|^2)u_0(\tilde x,\tilde y)\d{\tilde{x}}\d\tilde{y}\;\stackrel{t\rightarrow\infty}{\longrightarrow}0\,,
\end{align*}
since $|y|^2=e^{-2\|C_3\| t}|\tilde{y}|^2$ with the induced norm for positive matrices $\|C_3\|=\lambda_{\max}(C_3)$.

The convergence in Wasserstein metric follows by $i)$ and $ii)$, see the characterization \cite[Theorem 8.8]{Ambrosio-Brue-Semola}.
\end{proof}

\section{Some conclusions and open questions}\label{sec Open Questions}

In the previous section we have seen that in some cases, the convergence of the distribution of parameters towards a steady state can also be exponentially fast: this holds under quite restrictive conditions on $Q$ and $L$. However, we think that this special situation is what happens to the distribution of the parameters near local minima of the loss function $L$. The heuristic reason is that $L$ should be $\lambda$-convex close to a minimum, and also $Q$ should be (practically) constant there. Let us be more precise: assuming all the necessary regularity for $\rho$, the behaviour of the stochastic process with probability density function $\rho$ near the critical points of $L$ can be understood as follow. Recall that $\rho$ satisfies
\begin{equation}\label{Fokker-Planck}
  \begin{cases}
  \partial_t\rho=\nabla\cdot\bigg(\varepsilon^2Q(x)\nabla\rho+\rho\left(\nabla L(x)+\varepsilon^2\nabla\cdot Q(x)\right)\bigg)
  & \mbox{on}\quad(0,\infty)\times\R^d\,,\\[4mm]
  \rho(0,x)=\rho_0(x) &\mbox{in}\quad\R^d\,,
  \end{cases}
\end{equation}
with $\nabla L(x)=\frac{1}{N}\sum_{i=1}^{N}\nabla L_i(x)$ and
$$Q(x)=\frac{1}{N}\sum_{i=1}^{N}\nabla L_i(x)\otimes\nabla L_i(x)-\nabla L(x)\otimes\nabla L(x)\,.$$
Suppose that $x_0\in\mathbb{R}^d$ is a local minimum of $L$, so that $\nabla L(x_0)=0$. The goal is to obtain an asymptotic expansion of \eqref{Fokker-Planck} in a neighbourhood of $x_0$. We consider, for small $h \in (-1,1)$, and for $z \in \R^d$,
$$
x=x_0+ h z\,,
$$
and assuming enough regularity of the loss functions $L_i$ for $i=1,\dots,N$ we can compute the Taylor expansion of $\nabla L$ and $Q$ around $x_0$:
\begin{align*}
\nabla L(x)&={\nabla L(x_0)}+h D^2 L(x_0)z+h^2 D^3L(x_0)[z,z] +o\left(|hz|^2\right)\\
Q(x)&=\frac{1}{N}\sum_{i=1}^{N}\nabla L_i(x_0)\otimes\nabla L_i(x_0)-\nabla L(x_0)\otimes\nabla L(x_0)\\
&\ +\frac{h^2}{N}\sum_{i=1}^{N}\bigg[\!\left(D^2L_i(x_0)z\right)\otimes\left(D^2 L_i(x_0)z\right)-\left(D^2 L(x_0) z\right)\otimes\left( D^2 L(x_0)z\right)\!\bigg]\!+o(|hz|^2)\,.
\end{align*}
 Since $\nabla L(x_0)=0$, the first order expansion reads as follows:
\begin{align*}
  \nabla L(x) &=h D^2L(x_0) z+o(|hz|)\\
  Q(x)&=\frac{1}{N}\sum_{i=1}^{N}\nabla L_i(x_0)\otimes\nabla L_i(x_0)+o(|hz|)=Q(x_0)+o(|hz|)\,.
\end{align*}
Now, let us consider the function $\tilde u(t,z):=\rho(t,x(z))=\rho(t,x_0+h z)$. Assuming enough regularity for $\tilde u$, we have
\begin{align*}
  \nabla_z \tilde u(t,y) & =h \nabla_x\rho(t,x(z))\,, \\
  \nabla_z\cdot\left(A\;\nabla_z \tilde u(t,z)\right) & =h^2\nabla_x\cdot\left(A\;\nabla_x\rho(t,x(z))\right)\,,
\end{align*}
for any matrix $A\in \R^{d\times d}$. As a consequence we can write the equation for $\tilde u$ as follows:
\begin{align}\label{local FP with error}
  \partial_t \tilde u & = \nabla_z\cdot \left(\varepsilon^2 Q(x_0)\nabla_z \tilde u+\tilde u\;D^2L(x_0)\,z\right)+h^2\;\nabla_z\cdot\left(O(1)\nabla_z \tilde u+\tilde u \,O(1)\right)\,,
\end{align}
This suggests that the  behaviour of $\rho$ near a critical point $x_0$ of $L$ is governed by the following equation for $u_h = u_h(t,z)$:
\begin{equation}\label{localize FP}
  \begin{cases}
    \quad\;\,\partial_t u_h =\nabla_z\cdot\left(\varepsilon^2 Q_0\nabla_z u_h+u_h\,Cz\right)\qquad\mbox{in}\quad (0,\infty)\times\R^d \,,\\
    u_h(0,z)=\rho_0(x_0+h z)\qquad\qquad\qquad\qquad\quad\;\;\mbox{in }\quad \R^d \,,
  \end{cases}
\end{equation}
where $Q_0 = Q(x_0) \ge 0$ is constant and symmetric, and $C=D^2L(x_0)\in \mathbb{R}^{d\times d}$. Also,  since $x_0$ is a local minimum we have $D^2 L(x_0)\ge 0$.
This allows us to apply the convergence results of Theorem \ref{Thm Wasserstein Convergence}, and get exponential decay towards a steady state.

\vspace{3mm}

The conclusion that we draw from the above discussion is that the behaviour of the parameters during the learning process and their final distribution, can be understood through equation \eqref{localize FP}, more precisely, when Theorem \ref{Thm Wasserstein Convergence} applies we obtain that there exists a steady state $u_\infty$ so that weak convergence in measure is exponentially fast, that is \eqref{weak measure convregence} holds, and the second moments converge strongly and exponentially fast, namely \eqref{convergence second moment} holds, i.e.,
\begin{equation*}
    \left|\int_{\mathbb{R}^n}\int_{\mathbb{R}^{d-n}}(|x|^2+|y|^2)(u(t,x,y)-u_\infty(x,y))\dx\dy
    \right|\le \kappa_2 e^{-\lambda t}\,,
  \end{equation*}
for suitable $\kappa_2,\lambda>0$.

\

\noindent\textbf{Open Questions. }We shall present a number of technical issues that we do not prove in this paper and would eventually lead to a rigorous proof of the above heuristic result.

\medskip

\noindent i)\,\textbf{Regularity. }The first issue is undoubtedly the regularity of the solution $\rho$ to problem \eqref{Fokker-Planck}, that can be obtained through the regularity of the ``heat'' kernel or fundamental solution to \eqref{Fokker-Planck}. We expect partial regularity in the ``non degenerate'' variables and we conjecture this to be enough to justify rigorously the above expansions.

\medskip

\noindent ii)\,\textbf{Localization error estimates. }The next step would be to quantify the error between the original solution $\tilde{u}$ of  \eqref{local FP with error} and the localized solution $u$ of \eqref{localize FP} (in each neighbourhood of the local minima) when $h>0$ is small (and fixed momentarily). We need suitable norm estimates for all $t>0$, of the form\vspace{-1mm}
\begin{equation*}
  \|\tilde{u}(t)-u(t)\|\le c(h,t)\,.\vspace{-1mm}
\end{equation*}
The main issue would be to show that the error term $c(h,t)\xrightarrow[]{h\rightarrow 0^+}0$ possibly with quantitative estimates of its behaviour.

\medskip

\noindent iii)\,\textbf{Building a global approximated solution.} Assuming that the error term above can be controlled at each minima, it seems reasonable to approximate the solution $\rho$ of \eqref{Fokker-Planck} by a combination of the localized solutions of \eqref{localize FP} around each local minimum.

Let us be more precise: let $\lbrace x_1,\dots,x_M\rbrace$ denote the local minima of $L$, and $u_i$ the corresponding solution to the localized equation \eqref{localize FP} around the local minimum $x_i$ (for any $1\le i\le M$). We conjecture that the global behaviour of $\rho$ should be given by\vspace{-1mm}
\begin{equation}\label{approximated solution}
  \rho(t,x)\approx\sum_{i=1}^{M}m_i(t)\;u_i\left(t,\frac{x-x_i}{h}\right)\qquad\forall t>0\;\;\forall x\in\mathbb{R}^d\,,\vspace{-1mm}
\end{equation}
where $m_i(t)$ represents the mass concentration/splitting around each local minimum $x_i$ at time $t>0$ and satisfy $\sum_{i=1}^{M}m_i(t)=1$ for all $t>0$.

Another challenge would be to show that the above approximation holds for all (sufficiently large) times. This would give information about the invariant probability measure (i.e., the global steady state), that should be well approximated by
\begin{equation*}\vspace{-1mm}
  \rho_\infty(x)\approx\sum_{i=1}^{M}m_i(\infty)\;u_{i,\infty}\left(\frac{x-x_i}{h}\right)\qquad\forall x\in\mathbb{R}^d\,,\vspace{-1mm}
\end{equation*}
where $u_{i,\infty}$ are the stationary solutions as in Theorem \ref{Thm Wasserstein Convergence}, with diffusion matrix $Q(x_i)$ and drift matrix $D^2L(x_i)$.

\vspace{2mm}

\noindent iv)\,\textbf{Sharp mass displacement and splitting. }A crucial question in the above approximation would be to obtain sharp information about the portion of mass $m_i(t)$ that gradually concentrates around the minima $x_i$ of $L$, together with its dependence on the initial distribution $\rho_0$. The local masses $m_i$ represent the probability of the parameters of the optimization to be $x_i\in\mathbb{R}^d$ after training with SGD. As we have shown in Example 6,  when $Q(x)=0$ for every $x\in\mathbb{R}^d$, Equation \eqref{Fokker-Planck} becomes a ``pure transport'' equation and the initial datum $\rho_0$ will determine completely the asymptotic behaviour of $\rho(t)$. Given a datum $\rho_0$, a delicate issue is to be able to estimate in the sharpest possible way which portion of its mass concentrate around each of the $x_i$, after we have entered the asymptotic regime. Another crucial issue would be to relate the asymptotic regime, i.e. the smallest time for which the approximation of point $iii)$ is effective, with the mean exit times studied in Section \ref{sec Diffusion Regime}.

\newpage

\appendix
\addcontentsline{toc}{section}{Appendix}
\renewcommand{\thesubsection}{\thesection.\Alph{subsection}}
\begin{appendices}

\noindent\textbf{\huge Appendices}

\section{Approximation of the Noisy SGD}\label{app: NSGD}

\begin{teo}\label{thm C.1}
Let $\lbrace\theta_n\rbrace_{n\ge 0}$ denote the sequence given by \eqref{NSGD} and let $X_t$ be the stochastic process defined by
\begin{equation}\label{NSDE app}
\begin{cases}
  \d X_t=-\nabla L(X_t)+\sqrt{\eta\big(\b^{-1}Q(X_t)+\delta I_{d\times d}\big)}\,\d W_t\\
    X_0=\theta_0
\end{cases}
\end{equation}
Then, assuming the same conditions of \cite[Theorem 1]{LiTaiWeinan2017}, it holds that \eqref{NSDE app} is a weak approximation of order 1 of  \eqref{NSGD}, that is, for every function $g:\R^d\rightarrow\R$ with polynomial growth there exists $C>0$ independent of $\eta$ so that
\[
\bigg|\EE\big[g(X_{n\eta})\big]-\EE\big[g(\theta_n)\big]\bigg|<C\,\eta\qquad\forall n\ge 0\,.
\]
\end{teo}
\begin{proof}
  Let us follow the same steps of the proof of \cite[Theorem 1]{LiTaiWeinan2017}. First we recall the following Lemma.
  \begin{lem}\label{lem C.2}
 		Let $0<\eta<1$. Consider a stochastic process $\lbrace X_t\rbrace_{t\ge0}$, $t\geq 0$ satisfying
 		\begin{equation*}
 		dX_t=b(X_t)dt+\eta^{1/2}\sigma(X_t)dW_t,
 		\end{equation*}
 		with $X_0=\theta_0\in\mathbb{R}^d$. Define the one-step difference $\Lambda:=X_\eta-\theta_0$, then we have
 \begin{enumerate}[leftmargin=7mm,label=\normalfont\roman*)]
 			\item $\mathbb{E}[\Lambda_{i}]=b_{i}(\theta_0)\eta+\frac{\eta^2}{2}\left(\sum_{j=1}^{d}b_{j}(\theta_0)\partial_{j} b_{i}(\theta_0)\right)+\mathcal{O}(\eta^3).$
 			\item $\mathbb{E}\left[\Lambda_{i}\Lambda_{j}\right]=\left(b_{i}(\theta_0)b_{j}(\theta_0)+\sigma\sigma^T_{ij}(\theta_0)\right)\eta^2+\mathcal{O}(\eta^3)$.
 			\item $\mathbb{E}\left[\prod_{j=1}^{s}\Lambda_{i_j} \right]=\mathcal{O}(\eta^3)$ for all $s\geq3$, $i_j=1,\dots,d.$
 		\end{enumerate}
 	\end{lem}
 Therefore, we apply Lemma \ref{lem C.2} with $X_t$ as in \eqref{NSDE app} and we get
 \begin{enumerate}[leftmargin=7mm,label=\normalfont\roman*)]
    \item $\EE\big[\Lambda_i\big]=-\eta\,\partial_i L(\theta_0)+\mathcal{O}(\eta^2)$.
    \item $\EE\big[\prod_{j=1}^{s}\Lambda_{i_j}\big]=\mathcal{O}(\eta^2)$ for all $s\le2$, $i_j=1,\dots,d$.
  \end{enumerate}
  On the other hand, we consider the first step of the iteration \eqref{NSGD}
  $$\bar{\Lambda}:=\theta_1-\theta_0=-\frac{\eta}{\b}\sum_{b_i\in B_1}\nabla L_{b_i}(\theta_0)+\eta Z_0\,,$$
  and we proceed with the analogous computations:
   \begin{enumerate}[leftmargin=7mm,label=\normalfont\roman*)]
   \item $\EE\big[\bar{\Lambda}_i\big]=-\eta\,\partial_i L(\theta_0)+\EE\big[Z_{0,i}\big]=-\eta\,\partial_i L(\theta_0)$\,.
   \item $\EE\big[\bar\Lambda_i\bar\Lambda_j\big]=\eta^2\,\partial_i L(\theta_0)\,\partial_j L(\theta_0)+\eta^2\EE\big[Z_{0,i}\,Z_{0,j}\big]=\eta^2\,\partial_i L(\theta_0)\,\partial_j L(\theta_0)+\eta^2\mbox{Cov}(Z_{0,i}\,,Z_{0,j})$\\[2mm]
       $\color{white}a \hspace{4.5mm} \qquad=\eta^2\,\partial_i L(\theta_0)\,\partial_j L(\theta_0)+\eta^2\delta\,\delta_{ij}=\mathcal{O}(\eta^2)$\,,
   \end{enumerate}
   where $\delta_{ij}$ is the Kronecker delta and we have used that $\EE\big[Z_0\big]=0$, $\mbox{Var}(Z_0)=\delta I_{d\times d}$ together with the fact that $\theta_0$ and $Z_0$ are independent.\\[2mm]

   \noindent Now, we will need a key result linking one step approximations to global approximations.
 	\begin{teo}[Theorem 2 and Lemma 5 of \cite{Milstein}]\label{Milstein}
 		Let $\alpha$ be a positive integer and let $L$ be smooth enough with controlled growth. Assume, in addition, that there exist two functions $G_1,G_2$ with polynomial growth so that
 		$$\Big|\mathbb{E}\big[\prod_{j=1}^{s}\Lambda_{i_j}\big]-\mathbb{E}\big[\prod_{j=1}^{s}\bar{\Lambda}_{i_j}\big]\Big|\leq G_1(x)\eta^{\alpha+1},$$
 		for $s=1,2,\dots,2\alpha+1$ and
 		$$\mathbb{E}\big[\prod_{j=1}^{2\alpha+2}|\bar{\Lambda}_{i_j}|\big]\leq G_2(x)\eta^{\alpha+1}.$$
 		Then, there exists a constant $C$ such that for all $g$ with polynomial growth we have
 		$$\big|\mathbb{E}[g(X_{n\eta})]-\mathbb{E}[g(x_n)]\big|\leq C\eta^\alpha\qquad\forall n\ge 0\,.$$
 	\end{teo}
 	Hence, we conclude the proof of Theorem \ref{thm C.1} by applying Theorem \ref{Milstein} with $\alpha=1$.

\end{proof}

\section{Deduction of the Mean Exit Time problem}\label{AppendixMET}
Let us recall some classical connections between stochastic processes and PDEs, and in order to focus on the main ideas, we keep it at the heuristic level. See \cite{BogaKrylov,Oksendal,Pavliotis,Schuss} for the precise statements with all the assumptions needed to make these results rigorous.

Consider a  stochastic process in $\mathbb{R}^d$ defined by the following stochastic differential equation
\begin{equation*}
  \begin{cases}
    \d X_t&=b(X_t)\dt+a(X_t)\d W_t,  \\
    X_0&=x
  \end{cases}
\end{equation*}
with sufficiently smooth functions $a:\mathbb{R}^d\rightarrow\mathbb{R}^{d\times d}$ and $b:\mathbb{R}^d\rightarrow\mathbb{R}^d$, where $W_t$ is the standard Brownian motion. The infinitesimal generator associated to this stochastic process is given by
\begin{align*}
  \mathcal{A} f(x)&=\lim\limits_{t\rightarrow 0}\frac{\mathbb{E}_x[f(X_t)]-f(x)}{t}\\
  &=\frac{1}{2}\mbox{tr}\bigg(a(x)^Ta(x)D^2 f(x)\bigg)-b(x)\cdot\nabla f(x)\,.
\end{align*}
The two main differential equations related with this operator are the following:

\noindent\textbf{Kolmogorov backward equation}: Using Ito's Lemma one can deduce
\begin{equation*}
  \begin{cases}
    \partial_t u(t,x)&=\mathcal{A}u(t,x),\qquad\mbox{in }\mathbb{R}_+\times\mathbb{R}^d\\
    u(0,x)&=f(x)\qquad\qquad\mbox{in }\mathbb{R}^d\,.
  \end{cases}
\end{equation*}
Under suitable assumptions on the functions $a,b$ it is known by Dynkin's formula \cite[Chapters 7-8]{Oksendal} that the solution represents the expected value of the function $f$ applied to the stochastic process, namely,
$$u(t,x)=\mathbb{E}_x\left[f(X_t)\right]\,.$$
\noindent\textbf{Fokker-Planck equation}: Let us consider $\mathcal{A}^*$ to be the adjoint operator of $\mathcal{A}$ with respect to the $L^2(\mathbb{R}^d)$ scalar product. Then, the Fokker-Planck equation associated to the process $X_t$ reads as follows
\begin{equation*}
  \begin{cases}
    \partial_t \rho(t,x)&=\mathcal{A}^*\rho(t,x),\qquad\mbox{in }\mathbb{R}_+\times\mathbb{R}^d\\
    \rho(0,x)&=\rho_0(x)\qquad\qquad\mbox{in }\mathbb{R}^d\,.
  \end{cases}
\end{equation*}
The solution to this equation when $\rho_0 = \delta$ is the conditional probability $p(X_t = x | X_0 = x)$.

Beside these two classical parabolic differential equations, there are also elliptic differential equations associated to the operator $\mathcal{A}$ which describes certain properties of the stochastic process, such as the mean exit time of the stochastic process from a domain $\Omega\subset\mathbb{R}^d$.

\noindent\textbf{Mean Exit Time equation}:\cite[Chapter 7.2]{Pavliotis} Since the mean exit time is time-indepen-dent, let us denote it by
\begin{equation*}
  u(x)=\mathbb{E}[\tau_\Omega^x]\,,
\end{equation*}
with
\begin{equation*}
  \tau_\Omega^x:=\inf\lbrace t>0 : X_t\not\in\Omega,X_0=x\rbrace\,.
\end{equation*}
This mean exit time satisfies the following elliptic equation
\begin{equation}\label{MET}
  \begin{cases}
    \mathcal{A}u(x)=-1, & \mbox{in } \Omega \\
    u(x)=0, & \mbox{otherwise}.
  \end{cases}
\end{equation}

For the sake of clarity, let us show a heuristic proof of the deduction of this equation.
\begin{proof}
Let $u:\mathbb{R}^d\rightarrow\mathbb{R}$ be the solution to \eqref{MET} and let us consider the stochastic process $u(X_t)$ with $X_0=x$. By Ito's Lemma, we have that
\begin{align*}
  \d\left(u(X_t)\right)=\mathcal{A}u(X_t)\dt+a(X_t)\nabla u(X_t)\d W_t\,.
\end{align*}
Integrating in $[0,t]$ we obtain
\begin{align*}
  u(X_t)-u(x)=\int_{0}^{t}\mathcal{A}u(X_s)\d s+\int_{0}^{t}a(X_s)\nabla u(X_s)\d W_s\,.
\end{align*}
Since above identity holds for every time $t>0$, let us choose $t$ to be equal to
 \begin{equation*}
   \tau_{{}_T}=\min\lbrace T,\tau_\Omega^x\rbrace\,,
 \end{equation*}
 in order to obtain
 \begin{equation}\label{Ito for Exit time}
   u(X_{\tau_{{}_T}})-u(x)=\int_{0}^{\tau_{{}_T}}\mathcal{A}u(X_s)\ds+\int_o^{\tau_{{}_T}}a(X_s)\nabla u(X_s)\d W_s\,.
 \end{equation}
 By construction of $\tau_{{}_T}$ we know that
 \begin{equation*}
   X_s\in\Omega\qquad\forall x\in\Omega\quad\forall 0<s<\tau_{{}_T}\,,
 \end{equation*}
and hence
\begin{equation*}
  \mathcal{A}u(X_s)=-1\qquad \forall 0<s<\tau_{{}_T}\,.
\end{equation*}
Using this property in \eqref{Ito for Exit time}, we obtain that
\begin{equation*}
  u(X_{\tau_{{}_T}})-u(x)=-\tau_{{}_T}+\int_{0}^{\tau_{{}_T}}a(X_s)\nabla u(X_s)\d W_s\,.
\end{equation*}
Next step is to take the expectation in both sides of the equality. For this purpose, note that under mild assumptions on $a$ we have
$$\mathbb{E}\left[\int_{0}^{\tau_{{}_T}}a(X_s)\nabla u(X_s)\d W_s\right]=0\,,$$
since $\mathbb{E}\left[\tau_{{}_T}\right]\le T<+\infty$. Thus, we obtain that
\begin{equation*}
  \mathbb{E}\left[u(X_{\tau_{{}_T}})\right]=u(x)-\mathbb{E}\left[\tau_{{}_T}\right]\,.
\end{equation*}
Since $\mathbb{E}\left[\tau_{{}_T}\right]\le\mathbb{E}\left[\tau_\Omega^x\right]<+\infty$ and the trajectories of $X_s$ are continuous, we can take the limit $T\rightarrow+\infty$ using Dominated Convergence Theorem,
\begin{equation*}
  \mathbb{E}\left[u(X_{\tau_\Omega^x})\right]=u(x)-\mathbb{E}\left[\tau_\Omega^x\right]\,.
\end{equation*}
Finally, note that by construction $X_{\tau_\Omega^x}\in\partial\Omega$ and hence $u(X_{\tau_\Omega^x})=0$. Then, we conclude
\begin{equation*}
  u(x)=\mathbb{E}\left[\tau^x_\Omega\right]\,. \qedhere
\end{equation*}
\end{proof}

\section{Kramers' Law}\label{sec Kramers Law}

The simplest example of the Mean Exit Time problem presented before is the one associated to the stochastic process
\begin{equation}\label{isotropic SDE}
  \d X_t=-\nabla L(X_t)\dt+\sqrt{2\varepsilon^2}\d W_t\,.
\end{equation}
The main advantage of this process is that the invariant probability measure is unique an explicit, namely, its density is given by
\begin{align*}
  \rho_\infty(x)= c\,e^{-L(x)/\varepsilon^2}\,,
\end{align*}
with  a normalizing constant $c>0$.
If the function $L$ has different local minima, for example $x_1,x_2\in\mathbb{R}^d$, a natural question to ask is what is the mean time that needs $X_t$ to reach a neighbourhood of $x_2$ if $X_0=x_1$. In 1940, Kramers addressed this problem from a physical point of view in \cite{Kramers}, where he developed what it is now known as the Kramers' Law to describe the mean transition time of an overdamped Brownian particle between local minima in a potential landscape. If we define
\begin{equation*}
  \tau_{x_2}^{x_1}=\min\lbrace t>0 : X_t\in B_R(x_2), X_0=x_1\rbrace\,,
\end{equation*}
Kramers Law in the one-dimensional case, d=1, reads
\begin{equation}\label{unidimensional Kramers Law}
  \mathbb{E}\left[\tau_{x_2}^{x_1}\right]\backsimeq\frac{2\pi}{\sqrt{L''(x_1)|L''(z)|}}e^{\left(L(z)-L(x_1)\right)/\varepsilon^2}\,.
\end{equation}
Here, $z\in\mathbb{R}^d$  is called the \emph{relevant saddle} point and it is the maximum point of $L$ among all the paths from $x_1$ to $x_2$, that is, $z$ is the point where the \emph{communication height}
\begin{equation*}
  H(x_1,x_2)=\inf\limits_{\varphi:x_1\rightarrow x_2}\left(\sup\limits_{y\in\varphi}L(y)\right)
\end{equation*}
is attained.
In the multidimensional case, a similar result holds assuming that the Hessian $D^2L(z)$ has a single negative eigenvalue. Thus, Kramers's Law for $d\ge 2$ reads
\begin{equation}\label{multidimensional Kramers Law}
  \mathbb{E}\left[\tau_{x_2}^{x_1}\right]\backsimeq\frac{2\pi}{|\lambda_1(z)|}\sqrt{\frac{\left|\det(D^2 L(z))\right|}{\det(D^2L(x_1))}} \; e^{\left(L(z)-L(x_1)\right)/\varepsilon^2}\,,
\end{equation}
with $\lambda_1(z)< 0$ being the unique negative eigenvalue of $D^2 L(z)$.
Despite Kramers' Law being a well-known fact in physics since the mid-20th century, the rigorous mathematical proof did not arrive until 2004, thanks to Berglund and Gentz in \cite{Belgrund2006}. They showed that \eqref{multidimensional Kramers Law} is an equality up to an error of order ${\varepsilon}|\log(\varepsilon^2)|^{3/2}$ by using analytical tools based on estimates on the Green function and the capacity, see \cite[Theorem 3.3]{Belgrund2013}.

However, another approach to this problem is the theory of large deviation, which gives a mathematically rigorous framework to the path-integral method used in physics. Considering the stochastic process in \eqref{isotropic SDE}, the large deviation principle states that for small $\varepsilon$, the probability of sample paths being close to a set $\Gamma$ of function $\varphi:[0,T]\rightarrow\mathbb{R}^d$ behaves like
\begin{equation*}
  \lim\limits_{\varepsilon\rightarrow 0} 2\varepsilon^2\log\mathbb{P}\lbrace \left(X_t\right)_{0\le t\le T}\in\Gamma\rbrace=-\inf\limits_{\varphi\in\Gamma}I(\varphi)\,,
\end{equation*}
with $I$ being the \emph{action function}
$$I(\varphi)=\frac{1}{2}\int_{0}^{T}\left|\frac{\d}{\dt}\varphi(t)+\nabla L(\varphi(t))\right|^2\dt\,.$$
This large deviation principle can be stated roughly as
\begin{equation*}
  \mathbb{P}\lbrace\left(X_t\right)_{0\le t\le T}\in\Gamma\rbrace\backsimeq e^{-\inf\limits_{\Gamma}I/2\varepsilon^2}\,.
\end{equation*}
Note that the action function presented above can be written as
\begin{align*}
  I(\varphi)&=\frac{1}{2}\int_{0}^{T}\left|\frac{\d}{\dt}\varphi(t)+\nabla L(\varphi(t))\right|^2\dt\\
  &=\frac{1}{2}\int_0^T\left|\frac{\d}{\dt}\varphi(t)-\nabla L(\varphi(t))\right|^2\dt+2\int_0^T\frac{\d}{\dt}\varphi(t)\cdot\nabla L(\varphi(t))\dt\\
  &=\frac{1}{2}\int_0^T\left|\frac{\d}{\dt}\varphi(t)-\nabla L(\varphi(t))\right|^2\dt+2\bigg[L(\varphi(t))-L(\varphi(0))\bigg]\,.
\end{align*}
Hence, if $\varphi$ satisfies the time-reversed system $\frac{\d}{\dt}\varphi=\nabla L(\varphi)$, the first term of the action function above vanishes. Connecting a local minimum $x_1$ to a point in the basin of attraction of $x_2$ with such a solution is possible if one allows for an arbitrarily long time. Therefore, the quasipotential is given by
$$\bar{L}=2\left[\inf\limits_{\partial\Omega} L-L(x_1)\right]\,,$$
with $\Omega\subset\mathbb{R}^d$ being a neighbourhood of $x_1$. In the case of double-well potential,  if $\Omega$ is chosen such that it is cointained in the basin of attraction of $x_1$ and its boundary is close to the relevant saddle point $z$, the large deviation principle reads
\begin{equation*}
  \lim\limits_{\varepsilon\rightarrow 0}\varepsilon^2\log\mathbb{E}\left[\tau_{B_R(x_2)}^{x_1}\right]=L(z)-L(x_1)\,.
\end{equation*}

It is possible to generalize this large deviation approach to stochastic processes of the form
\begin{equation}\label{anisotropic SDE}
\begin{cases}
  \d X_t&=-\nabla L(X_t)\dt+\sqrt{2\varepsilon}a(X_t)\d W_t\\
  X_0&=x\,,
\end{cases}
\end{equation}
with a  non-degenerate matrix $a(x)\in\mathbb{R}^{d\times d}$. In this case, the invariant probability measure is not explicit in general, nevertheless it is possible to write the action function as
\begin{equation*}
  I_a(\varphi)=\int_{0}^{T}\left|a(\varphi(t))^{-1}\left(\frac{\d}{\dt}\varphi(t)+\nabla L(\varphi(t))\right)\right|^2\dt\,,
\end{equation*}
if $\varphi$ is absolutely continuous, $\frac{\d}{\dt}\varphi$ is square integrable and $\varphi(0)=x$; in any other case $I_a(\varphi)=+\infty$. Under restrictive condition on $L$ and $a$, a large deviation principle for \eqref{anisotropic SDE} was proven in \cite{Dembo98,Var84}.

\begin{teo}\cite[Theorem 12.1]{Weinan-Book}
  Assume that $\nabla L$ and $a$ are bounded and Lipschitz continuous. If $a^T(x)a(x)$ is uniformly elliptic, then the following large deviation principle holds for every $t>0$:
  \begin{enumerate}[label=\roman*)]
    \item Upper bound. For any closed set $\Gamma_c\subset(C([0,T]))^d$,
    \begin{equation*}
      \lim\limits_{\varepsilon\rightarrow0}\varepsilon^2\log\mathbb{P}\lbrace(X_t)_{0\le t\le T}\in\Gamma_c\rbrace\le -\inf\limits_{\varphi\in\Gamma_c}I_a(\varphi)
    \end{equation*}

    \item Lower bound. For any open set $\Gamma_o\subset(C([0,T]))^d$,
    \begin{equation*}
      \lim\limits_{\varepsilon\rightarrow0}\varepsilon^2\log\mathbb{P}\lbrace(X_t)_{0\le t\le T}\in\Gamma_o\rbrace\ge -\inf\limits_{\varphi\in\Gamma_o}I_a(\varphi)
    \end{equation*}
  \end{enumerate}
\end{teo}

Hence, roughly speaking, for the stochastic process defined in \eqref{anisotropic SDE} we have that
\begin{equation*}
  \mathbb{P}\lbrace\left(X_t\right)_{0\le t\le T}\in\Gamma\rbrace\backsimeq e^{-\inf\limits_{\Gamma}I_a/2\varepsilon^2}\,.
\end{equation*}
This approach provides a method to compute likelihood of rare events in stochastic systems. For a practical explanation of the techniques used to derive sharp asymptotic estimates, refer to \cite{Vanden-Eijden2023}.

\end{appendices}

\newpage

\addcontentsline{toc}{section}{Declarations}

\noindent\textbf{Conflict of interest statement.} On behalf of all authors, the corresponding author states that there is no conflict of interest.

\vspace{2mm}

\noindent\textbf{Data availability statement.} All data generated or analysed during this study are included in the published article.


\addcontentsline{toc}{section}{References}

\begin{thebibliography}{1000}




\bibitem{Ambrosio-Brue-Semola}{L. Ambrosio, E. Brué and D. Semola.}{ \it Lectures on Optimal Transport.} Springer, Cham, 2021.

\bibitem{Ambrosio-Gigli} {L. Ambrosio and N. Gigli}. {\it A User’s Guide to Optimal Transport.} In: Modelling and Optimisation of Flows on Networks. Lecture Notes in Mathematics, vol 2062. Springer, Berlin, Heidelberg, 2013.

\bibitem{Andriushchenko2023} {M. Andriushchenko, A. Varre, L.  Pillaud-Vivien, and N. Flammarion} {\sl SGD with Large Step Sizes Learns Sparse Features}. Proceedings of the 40th International Conference on Machine Learning, PMLR 202, 2023.

\bibitem{Arnold2015}{F. Achleitner, A. Arnold and D. Stürzer.} {\sl Large-time behavior in non-symmetric Fokker-Planck equations.} Riv. Math. Univ. Parma (N.S.), 6(1):1–68, 2015.

\bibitem{Arnold2014}{ A. Arnold, J. Erb.} {\sl Sharp entropy decay for hypocoercive and non-symmetric Fokker-Planck equations with linear drift.} ASC Report 29/2014, 1-45, Institute of Analysis and Scientific Computing, TU Wien, 2014.

\bibitem{Arnold-Toscani2001}{A. Arnold, P. Markowich, G. Toscani and A. Unterreiter}  {\sl On convex Sobolev
inequalities and the rate of convergence to equilibrium for Fokker-Planck type equations.} Comm.
Partial Differential Equations 26 43–100, 2001.

\bibitem{AroraCohenHuLuo2019} {S. Arora, N. Cohen, W. Hu, and Y. Luo} {\sl Implicit Regularization in Deep Matrix Factorization}. 33rd Conference on Neural Information Processing Systems, 2019.

\bibitem{BachMoulines2011} {F. Bach, E. Moulines} {\sl Non-Asymptotic Analysis of Stochastic Approximation Algorithms for Machine Learning}. Advances in neural information processing systems 24, 2011.

\bibitem{Bakry2008} D. Bakry, F. Barthe and P. Cattiaux. {\sl A simple proof of the Poincaré inequality for a large class of probability measures including the log-concave case.} Electronic Communications in Probability, 13, 2008.

\bibitem{Bakry-Emery} {D. Bakry, M. Émery. }{\sl Diffusions hypercontractives.} Séminare de Probabilités XIX 1983/84, 177-206, 1985.

\bibitem{Bakry-Ledoux} {D. Bakry, I. Gentil, and M. Ledoux.} {\it Analysis and geometry of Markov diffusion operators.}
    Vol. 348. Grundlehren der Mathematischen Wissenschaften, Fundamental Principles of Mathematical Sciences. Springer, Cham, 2014.

\bibitem{BarrettDherin2021} {D. Barrett, B. Dherin} {\sl Implicit gradient regularization}. International Conference on Learning Representations, 2021.

\bibitem{Belgrund2013}
    {N. Belgrund.} {\sl Kramers’ law: validity, derivations and generalisations.} Markov Process. Relat. Fields 19, 459–490, 2013.

\bibitem{Belgrund2006}
    {N. Berglund and B. Gentz}. {\sl Noise-Induced Phenomena in Slow-fast Dynamical Systems.} Probability and its Applications (New York). Springer-Verlag London Ltd., London, 2006.

\bibitem{over-parametrized}
    {M. Belkin, C. Liu and L. Zhu.} {\sl Loss landscapes and optimization in over-parameterized non-linear systems and neural networks.} {Applied and Computational Harmonic Analysis}, vol 59, 85-116, 2022.

\bibitem{BenArousGheissariJagannath2022} {G. Ben Arous, R. Gheissari, A. Jagannath} {\sl High-dimensional limit theorems for SGD: Effective dynamics and critical scaling}. 36th Conference on Neural Information Processing Systems, 2022.
	
\bibitem{BiancaDogbe}
	{C. Bianca, C. Dogbe.}{\sl On the existence and uniqueness of invariant measure for multidimensional stochastic processes.} Nonlinear Studies - The International Journal, Cambridge Scientific Publishers, hal-02151779, 2017.

\bibitem{Blanc2020} {G. Blanc, N. Gupta, G. Valiant, and P. Valiant} {\sl Implicit regularization for deep neural networks driven by an Ornstein-Uhlenbeck like process}. Proceedings of Machine Learning Research vol 125:1–31, 2020.
	
\bibitem{BogaKrylov}
V.I. Bogachev, N.V. Krylov, M. Röckner, S.V. Shaposhnikov. {\it Fokker-Planck-Kolmogorov Equations.} Amer. Math. Soc., Providence, Rhode Island, 2015.
	
\bibitem{BCN18}
    {L. Bottou, F.E. Curtis and J. Nocedal.} {\sl Optimization Methods for Large-Scale Machine Learning.} SIAM Review 60, 2016.

\bibitem{ChaudhariSoatto2018} {P. Chaudhari, S. Soatto} {\sl Stochastic gradient descent performs variational inference, converges to limit cycles for deep networks}. In 2018 Information Theory and Applications Workshop (ITA), pp. 1–10.

\bibitem{ChengYinBartlettJordan2020} {X. Cheng, D. Yin, P. Bartlett, and Michael Jordan}. Stochastic gradient and Langevin processes. In International Conference on Machine Learning, pp. 1810–1819. PMLR 119, 2020.

\bibitem{CCFF2024} {L. Chizat, M. Colombo, X- Fern\'andez-Real, A. Figalli} {\sl Infinite-width limit of deep linear neural networks}. Communications on Pure and Applied Mathematics, 2024.
	
\bibitem{CrIsLi}
	{M.G. Crandall, H. Ishii, P.L. Lions. }{\sl User's guide to viscosity solutions of second order partial differential equations.} Bull. Amer. Math. Soc. 27, 1-67, 1992.

\bibitem{Dembo98}
    {A. Dembo and O. Zeitouni.} {\it Large deviations techniques and applications}, 2nd ed., Applications of Mathematics (New York), vol. 38, Springer-Verlag, New York, 1998.
    	
\bibitem{FengLiLiu}
	{Y. Feng, L. Li, J.-G. Liu.} {\sl Semi-groups of stochastic gradient descent and online principal component analysis: properties and diffusion approximations}. Commun. Math. Sci. 16 (3), 2018.

\bibitem{FernandezFigalli2022} {X. Fern\'andez-Real, A. Figalli} {\sl The Continuous Formulation of Shallow Neural Networks as Wasserstein-Type Gradient Flows}. In ``Analysis at large - Dedicated to the life and work of J. Bourgain'', A. Avila, M. Th. Rassias, Y. Sinai (Eds.), Springer 2022.

\bibitem{Figalli-Glaudo} {A. Figalli and F. Glaudo.} {\it An Invitation to Optimal Transport, Wasserstein
Distances, and Gradient Flows.}   EMS Press, Berlin, 2021.

\bibitem{Goyal2017} {P. Goyal, P. Dollár, R. Girshick, P. Noordhuis, L. Wesolowski, A. Kyrola, A. Tulloch, Y. Jia, K. He} {\sl Accurate, large minibatch sgd: Training imagenet in 1 hour}. arXiv preprint arXiv:1706.02677 (2017)

\bibitem{GoyalBengio2022} {P. Goyal, Y. Bengio} {\sl Inductive biases for deep learning of higher-level cognition}. Proc. R. Soc. A 478:20210068, 2022.

\bibitem{Vanden-Eijden2023}
    {T. Grafke, T. Schäfer and E. Vanden-Eijnden.} {\sl Sharp Asymptotic Estimates for Expectations, Probabilities, and Mean First Passage
Times in Stochastic Systems with Small Noise.} Commun. Pure Appl. Math. 2023.

\bibitem{Gunasekar2018} {S. Gunasekar, J. D. Lee, D. Soudry, and N. Srebro} {\sl Implicit Bias of Gradient Descent on Linear Convolutional Networks}. 32nd Conference on Neural Information Processing Systems, 2018,

\bibitem{Gurbuzbalaban2021} {M. Gurbuzbalaban, U. Simsekli, L. Zhu} {\sl The Heavy-Tail Phenomenon in SGD}. Proceedings of the 38th International Conference on Machine Learning, PMLR 139, 2021.

\bibitem{HaoChen2021} {J. Z. HaoChen, C. Wei, J. D. Lee, T. Ma} {\sl Shape Matters: Understanding the Implicit Bias of the Noise Covariance}. Proceedings of Machine Learning Research 134:1–43, 2021.
	
\bibitem{Hormander}
	{L. Hörmander. }{\sl Hypoelliptic second order differential equations}. Acta Math. 119, 147-171, 1967.

\bibitem{HuJunchiLiLiLiu2019} {W. Hu, C. Junchi Li, L. Li, and J.-G. Liu} {\sl On the diffusion approximation of nonconvex stochastic gradient descent}. Annals of Mathematical Sciences and Applications 4:3–32, 2019.

\bibitem{Jastrzebski2018} {S. Jastrzebski, Z. Kenton, D. Arpit, N. Ballas, A. Fischer, Y. Bengio, and A. Storkey} {\sl Three factors influencing minima in sgd}. In International Conference on Artificial Neural Newtorks 2018.

\bibitem{Jastrzebski2021} {S. Jastrzebski, M. Szymczak, S. Fort, D. Arpit, J. Tabor, K. Cho, and K. Geras} {\sl The break-even point on optimization trajectories of deep neural networks}. In International Conference on Learning Representations, 2021.
	
\bibitem{Kato1980} {T. Kato} {\sl Perturbation Theory for Linear Operators}. Springer, 1995.


\bibitem{KloedenPlaten}
	{Kloeden P E and Platen E}. {\it Numeric Solution of Stochastic Differential Equations}. Springer, Berlin, 395–401, 1992.
	
\bibitem{Kramers}
    {H. A. Kramers.} {\sl Brownian motion in a field of force and the diffusion model of chemical reactions.} Physica 7, 284–304, 1940.

\bibitem{KushnerYin2003} {H. J. Kushner, G. Yin} {\sl Stochastic approximation and recursive algorithms and applications}. Springer, 2003.

\bibitem{LeCun1998} {Y. LeCun, L. Bottou, G. B. Orr, and K.-R. M\"uller} {Efficient BackProp}. In G. Orr, K. M\"uller (eds) ``Neural networks: tricks of the trade". Springer 1998.

\bibitem{LiMalladiArora2021} {Z. Li, S. Malladi, and S. Arora} {\sl On the Validity of Modeling SGD with Stochastic Differential Equations (SDEs)}. Advances in Neural Information Processing Systems 34, 2021.

\bibitem{LiTaiWeinan2017}
	{Q. Li, C. Tai, and W. E.} {\sl Stochastic modified equations and adaptive stochastic
		gradient algorithms}. In Proceedings of the 34th International Conference on Machine Learning, 2101–2110, 2017.

\bibitem{LiTaiWeinan2019}
    {Q. Li, C. Tai, and W. E.} {\sl Stochastic modified equations and dynamics of stochastic gradient algorithms i: Mathematical foundations.} The Journal of Machine Learning Research, 20(1):1474–1520, 2019.

\bibitem{LiWangArora2022} {Z. Li, T. Wang, and S. Arora} {\sl What happens after SGD reaches zero loss? - A mathematical framework}. International Conference on Learning Representations, 2022.

\bibitem{LiWeiMa2019} {Y. Li, C. Wei, T. Ma} {\sl Towards Explaining the Regularization Effect of Initial Large Learning Rate in Training Neural Networks}. 33rd Conference on Neural Information Processing Systems, 2019.

\bibitem{Malladietal2023} Fine-Tuning Language Models with Just Forward Passes.

\bibitem{MallladiLyuPanigrahiArora2022} {S. Malladi, K. Lyu, A. Panigrahi, S. Arora} {\sl On the SDEs and Scaling Rules for Adaptive Gradient Algorithms}. 36th Conference on Neural Information Processing Systems, 2022.

\bibitem{MandtHoffmanBlei2016} {S. Mandt, M. Hoffman, and D. Blei} {\sl A variational analysis of stochastic gradient algorithms}. Proceedings of The 33rd International Conference on Machine Learning 48:354-363, 2016.

\bibitem{MandtHoffmanBlei2017} {Mandt, M. D. Hoffman, and D. M. Blei} {\sl Stochastic gradient descent as approximate Bayesian inference}. The Journal of Machine Learning Research 18:4873–4907, 2017.

\bibitem{MastersLuschi2018} {D. Masters, C. Luschi} {\sl Revisiting small batch training for deep neural networks}. Preprint at https://arxiv.org/abs/1804.07612 (2018)

\bibitem{Menz} {G. Menz, A. Schlichting.} {\sl Poincaré and logarithmic Sobolev inequalities by decomposition of the energy landscape.} Ann. Probab. 42 (5) 1809-1884,  2014.

\bibitem{Milstein} {G.N. Milstein.} {\it Numerical integration of stochastic differential equations,} volume 313. Springer Science \& Business
Media, 1995.

\bibitem{MoucerTaylorBach2023} {C. Moucer, A. Taylor, and F. Bach} {\sl A systematic approach to Lyapunov analyses of continuous-time models in convex optimization}. SIAM Journal on Optimization, 33:1558–1586, 2023.

\bibitem{Oksendal} {B. Øksendal.} {\it Stochastic differential equations.} Universitext. Springer-Verlag, Berlin, 2003.

\bibitem{Pavliotis} {G.A. Pavliotis.} {\it Stochastic Processes and Applications: Diffusion Processes, the Fokker-Planck and Langevin Equations.} Springer, New York, 2014.

\bibitem{Por} {A. Porretta. }{\sl Decay rates of convergence for Fokker-Planck equations with confining drift.} Advances in Mathematics, vol. 436, 109393, 2024.

\bibitem{RRT17}  M. Raginsky, A. Rakhlin, M. Telgarsky, \emph{Non-Convex Learning via Stochastic Gradient Langevin Dynamics: A Nonasymptotic Analysis}. Proceedings of Machine Learning Research vol 65:1–30, 2017.

\bibitem{Rotskoff2019} {G. Rotskoff, E. Vaden-Eijnden} {\sl Trainability and accuracy of neural networks: an interacting particle system approach}. Communications on Pure and Applied Mathematics LXXV:1889–1935, 2022.

\bibitem{Schuss} {Z. Schuss. } {\it Theory and Applications of Stochastic Differential Equations.} Applied Math. Science, vol. 170, 2010.

\bibitem{ShiSuJordan2023} {B. Shi, W. Su, M. Jordan} {\sl On Learning Rates and Schr\"odinger Operators}. The Journal of Machine Learning Research 24:1-53, 2023.

\bibitem{Simsekli2019} {U. Simsekli, L. Sagun, M. G\"urb\"uzbalaban} {\sl A Tail-Index Analysis of Stochastic Gradient Noise in Deep Neural Networks}. Proceedings of the 36th International Conference on Machine Learning, PMLR 97, 2019.

\bibitem{Soudry2018} {D. Soudry, E. Hoffer, M. Shpigel Nacson, S. Gunasekar, and N. Srebro} {\sl The Implicit Bias of Gradient Descent on Separable Data}. The Journal of Machine Learning Research 19:2822–2878, 2018.

\bibitem{Var84} {S. R. S. Varadhan.} {\it Large deviations and applications.} CBMS-NSF Regional Conference Series in Applied Mathematics, vol. 46, Society for Industrial and Applied Mathematics (SIAM), Philadelphia, 1984.

\bibitem{Vardi2023} {G. Vardi} {\sl On the Implicit Bias in Deep-Learning Algorithms}. Communications of the ACM 66:86-93 (2023).

\bibitem{Vaz18} J. L. V\'azquez, \emph{Asymptotic behaviour methods for the Heat Equation. Convergence to the Gaussian}. arXiv preprint https://arxiv.org/abs/1706.10034

\bibitem{Weinan-Book}{W. E, T. Li and E. Vanden-Eijnden. }{\it Applied Stochastic Analysis.} American Mathematical Society, Graduate Studies in Mathematics, 2021.

\bibitem{Wilson2017} {A. Wilson, R. Roelofs, M. Stern, N. Srebro, and B. Recht} {\sl The Marginal Value of Adaptive Gradient Methods in Machine Learning}. 31st Conference on Neural Information Processing System, 2017.
	
\bibitem{Woj}{S. Wojtowytsch.} {\sl Stochastic gradient descent with noise of machine learning type. Part II: Continuous time analysis}. Journal of Nonlinear Science, vol. 34, 2021.

\bibitem{XieSatoSugiyama2021} {Z. Xie, I. Sato, and M. Sugiyama} {\sl A diffusion theory for deep learning dynamics: Stochastic gradient descent exponentially favors flat minima}. International Conference on Learning Representations, 2021.

\bibitem{Yaida2019} {S. Yaida} {\sl Fluctuation - dissipation relations for stochastic gradient descent} International Conference on Learning Representations, 2019.

\bibitem{Zhang2017} {C. Zhang, S. Bengio, M. Hardt, B. Recht, and O. Vinyals} {\sl Understanding deep learning requires rethinking generalization}. International Conference on Learning Representations, 2017.

\bibitem{Zhou2020} {P. Zhou, J. Feng, C. Ma, C. Xiong, S. HOI, W. E} {\sl Towards Theoretically Understanding Why SGD Generalizes Better Than A DAM in Deep Learning}. 34th Conference on Neural Information Processing Systems, 2020.


	
\end{thebibliography}
\end{document}